\documentclass{article}

\usepackage{microtype}
\usepackage{graphicx}
\usepackage{subcaption}
\usepackage{booktabs} 

\usepackage[most]{tcolorbox}
\tcbset{
  myblock/.style={
    enhanced,
    colback=gray!10,
    colframe=black,
    boxsep=5pt,
    arc=2pt,
    left=6pt,right=6pt,top=6pt,bottom=6pt
  }
}

\usepackage{xcolor}

\usepackage{nicematrix, booktabs, xcolor}
\definecolor{myblue}{RGB}{232, 242, 255}
\definecolor{rowgray}{gray}{0.96}

\usepackage[table,dvipsnames]{xcolor} 
\definecolor{hlblue}{RGB}{220,235,255} 
\newcolumntype{C}[1]{@{}>{\centering\arraybackslash}p{#1}@{}}
\newcolumntype{L}[1]{@{}>{\arraybackslash}p{#1}@{}}

\usepackage{amsmath,amsfonts,bm}

\def\1{\bm{1}}

\DeclareMathAlphabet{\mathsfit}{\encodingdefault}{\sfdefault}{m}{sl}
\SetMathAlphabet{\mathsfit}{bold}{\encodingdefault}{\sfdefault}{bx}{n}

\usepackage{graphicx}  
\usepackage{amsmath,amsfonts,bm,amsthm,amssymb}

\usepackage{mathtools}
\makeatletter
\newcommand{\restore@Environment}[1]{%
    \AtBeginDocument{%
        \csletcs{#1*}{#1}%
        \csletcs{end#1*}{end#1}%
    }%
}
\forcsvlist\restore@Environment{alignat,equation,gather,multline,flalign,align}
\makeatother

\def\1{\bm{1}}

\DeclareMathAlphabet{\mathsfit}{\encodingdefault}{\sfdefault}{m}{sl}
\SetMathAlphabet{\mathsfit}{bold}{\encodingdefault}{\sfdefault}{bx}{n}

\usepackage{acro}
\DeclareAcronym{mbpo}{
    short = MBPO ,
    long = model-based policy optimization ,
}
\DeclareAcronym{sac}{
    short = SAC ,
    long = soft actor-critic
}

\DeclareAcronym{knn}{
    short = KNN ,
    long = K-Nearest Neighbors
}

\DeclareAcronym{mdp}{
    short = MDP ,
    long = Markov decision process
}

\DeclareAcronym{DQN}{
    short = DQN ,
    long = Deep Q-Network
}
\DeclareAcronym{ood}{
    short = OOD ,
    long = Out-of-distribution 
}

\DeclareAcronym{sota}{
    short = SOTA ,
    long = state-of-the-art
}

\DeclareAcronym{ann}{
    short = ANN ,
    long = Approximate Nearest Neighbour 
}

\DeclareAcronym{InterCascade}{
    short = InterCascade ,
    long = Interactive LLM Cascade 
}

\usepackage[normalem]{ulem}
\newcommand\redout{\bgroup\markoverwith{\textcolor{red}{\rule[.5ex]{2pt}{0.4pt}}}\ULon}

\DeclareFontFamily{U}{matha}{\hyphenchar\font45}
\DeclareFontShape{U}{matha}{m}{n}{
      <5> <6> <7> <8> <9> <10> gen * matha
      <10.95> matha10 <12> <14.4> <17.28> <20.74> <24.88> matha12
      }{}
\DeclareSymbolFont{matha}{U}{matha}{m}{n}
\DeclareFontSubstitution{U}{matha}{m}{n}

\DeclareFontFamily{U}{mathx}{\hyphenchar\font45}
\DeclareFontShape{U}{mathx}{m}{n}{
      <5> <6> <7> <8> <9> <10>
      <10.95> <12> <14.4> <17.28> <20.74> <24.88>
      mathx10
      }{}
\DeclareSymbolFont{mathx}{U}{mathx}{m}{n}
\DeclareFontSubstitution{U}{mathx}{m}{n}

\DeclareMathDelimiter{\vvvert}{0}{matha}{"7E}{mathx}{"17}

\newcommand{\Betainv}{\mathrm{Beta}^{-1}}
\newcommand{\Phiinv}{\Phi^{-1}}
\newcommand{\repo}{\operatorname{Repo}}

\theoremstyle{plain}
\newtheorem{theorem}{Theorem}[section]

\newtheorem{lemma}[theorem]{Lemma}

\theoremstyle{definition}

\theoremstyle{remark}
\newtheorem{remark}[theorem]{Remark}

\usepackage{graphicx}
\usepackage[T1]{fontenc}
\usepackage[utf8]{inputenc}
\usepackage[font=small]{caption}
\usepackage[font=scriptsize]{subcaption}
\usepackage[english]{babel}

\usepackage{booktabs, array}
\usepackage{multicol}
\usepackage{multirow}
\usepackage{morefloats}
\usepackage{pifont}
\usepackage{amssymb}
\usepackage{siunitx}
\usepackage{xcolor}
\definecolor{darkgreen}{rgb}{0,0.5,0}

\usepackage{soul} 

\usepackage{amsmath,amssymb,mathtools}
\usepackage{algorithm}
\usepackage[noend]{algpseudocode} %

\newcommand{\Repo}{\operatorname{Repo}}
\newcommand{\Match}{f}   
\newcommand{\DefW}{d_1}  
\newcommand{\GenW}{g_1}  
\newcommand{\DefS}{d_2}  
\newcommand{\GenS}{g_2}  
\newcommand{\StratGen}{h}

\usepackage[preprint]{icml2026}

\usepackage{amsmath}
\usepackage{amssymb}
\usepackage{mathtools}
\usepackage{amsthm}

\usepackage{xurl} %
\definecolor{OI-vermillion}{RGB}{213,94,0}
\usepackage[colorlinks=true,citecolor=OI-vermillion,urlcolor=OI-vermillion]{hyperref}
\usepackage[capitalize,noabbrev]{cleveref}

\icmltitlerunning{From Deferral to Learning: Online In-Context Knowledge Distillation for LLM Cascades}

\begin{document}

\twocolumn[
  \icmltitle{From Deferral to Learning:\\ Online In-Context Knowledge Distillation for LLM Cascades}



  \icmlsetsymbol{equal}{*}

  \begin{icmlauthorlist}
    \icmlauthor{Yu Wu}{equal,ru}
    \icmlauthor{Shuo Wu}{equal,uic}
    \icmlauthor{Ye Tao}{equal,ru}
    \icmlauthor{Yansong Li}{uic}
    \icmlauthor{Anand D. Sarwate}{ru}
  \end{icmlauthorlist}
  \icmlaffiliation{ru}{Department of Electrical and Computer Engineering, Rutgers University, NJ, USA}
  \icmlaffiliation{uic}{Electrical and Computer Engineering Department, University of Illinois Chicago, Illinois, USA}

  \icmlcorrespondingauthor{Yu Wu}{yu.wu@rutgers.edu}
  \icmlcorrespondingauthor{Anand D. Sarwate}{anand.sarwate@rutgers.edu}

  \icmlkeywords{Machine Learning, ICML}

  \vskip 0.3in
]



\printAffiliationsAndNotice{\icmlEqualContribution}

\begin{abstract}

    Standard \textit{LLM Cascades} improve efficiency by \textit{deferring} difficult queries from weak to strong models. However, these systems are typically static: when faced with repeated or semantically similar queries, they redundantly consult the expensive model, failing to adapt during inference. To address this, we propose \textit{Inter-Cascade}, an online, interactive framework that transforms the strong model from a temporary helper into a long-term teacher.
    In our approach, when the strong model resolves a deferred query, it generates a generalized, reusable problem-solving \textit{strategy}. These strategies are stored in a dynamic repository and retrieved via similarity matching to augment the weak model’s context for future queries. This enables the weak model to "learn" on the job without expensive parameter fine-tuning. We theoretically show that this mechanism improves the weak model's confidence calibration. Empirically, Inter-Cascade outperforms standard cascades on multiple benchmarks, improving weak model and overall system accuracy by up to 33.06\% and 6.35\%, while reducing strong model calls by up to 48.05\% and saving fee by up to 49.63\%. Inter-Cascade demonstrates the effective in-context knowledge transfer between LLMs, and provides a general, scalable framework applicable to both open-source and API-based LLMs.

\end{abstract}

\section{Introduction}\label{sec: intorduction}

Large Language Models (LLMs) demonstrate remarkable performance across a wide range of generation and reasoning tasks. Generally, performance scales with model size~\citep{kaplan2020scalinglawsneurallanguage}, creating a fundamental trade-off: larger models are more capable but significantly more expensive and slower. To address this, the \textit{LLM Cascade} paradigm has emerged as a standard solution, where weaker (cheaper) models handle routine queries and \textit{defer} only uncertain or complex cases to stronger (expensive) models~\citep{chen_frugalgpt_2024}.

However, current cascade systems suffer from a critical ``memoryless'' limitation. Standard \textit{deferral functions} are optimized offline and remain static during deployment~\citep{shen_learning_2024, jung2024trustescalatellmjudges}. When the strong model resolves a difficult query, its expensive reasoning process is discarded immediately after the answer is returned. Consequently, the system fails to learn from its own previous operations. As noted in recent position paper from NVIDIA~\citep{belcak2025small}, real-world query streams often exhibit a "similarity phenomenon," containing repeated or semantically similar tasks (e.g., slight variations of math problems in GSM-Plus~\citep{li_gsm-plus_2024}.) Faced with these recurring patterns, static LLM cascades redundantly consult the strong model for every instance, leading to a substantial waste of computation and tokens. While fine-tuning the weak model could theoretically solve this, it is often prohibitively expensive, slow, or impossible for API-based models.

To bridge this gap, we propose moving from static deferral to adaptive learning. We argue that the strong model should not merely serve as a temporary "backup" but as a long-term "teacher." To achieve this, we introduce \textit{Inter-Cascade}, an online framework that transforms the interaction between LLMs. Unlike simple caching which only memorizes specific answers, Inter-Cascade extracts generalized problem-solving \textit{strategies} from the strong model's reasoning. By retrieving and injecting these strategies into the weak model’s context via similarity matching, we realize online in-context knowledge distillation. This enables the weak model to "learn" on the job, dynamically improving its local success rate for future similar queries without parameter updates.

Our approach advances the concepts of In-Context Learning (ICL)~\citep{dong2024surveyincontextlearning}, few-shot prompting~\citep{parnami2022learningexamplessummaryapproaches} and Retrieval-Augmented Generation (RAG)~\citep{lewis_retrieval-augmented_2021}. While traditional ICL relies on fixed demonstrations and RAG typically queries static, human-curated databases, Inter-Cascade builds a self-evolving strategy repository autonomously. The "corpus" is generated by the strong model and curated by the system's own interaction history, requiring no human intervention. This creates a closed-loop system where the weak model continuously distills ``wisdom'' from the strong model to handle increasingly complex tasks locally.

\paragraph{Primary contributions.} Our contributions are as follows: (1) We propose Inter-Cascade, a general and modular framework for online interactive LLM Cascades. It allows the strong model to "teach" the weak model via a similarity-based strategy repository, effectively implementing widely applicable in-context knowledge transfer for both open-source and API-based models. Inter-Cascade is designed as a universal booster that also works for existing LLM Cascade methods. (2) We provide a theoretical framework proving that integrating strong-model strategies improves the weak model's confidence calibration. We show that this mechanism allows the weak model to more accurately assess its own competence, thereby reducing unnecessary deferrals while maintaining safety bounds. (3) We show empirically that compared to state-of-the-art cascades~\citep{jung2024trustescalatellmjudges}, Inter-Cascade improves the weak model's accuracy by up to 33.06\% and overall system accuracy by up to 6.35\%. Crucially, it reduces calls to the strong model by up to 48.05\%, translating to significant cost savings (up to 49.63\%) while strictly adhering to risk tolerance guarantees. 

\section{Improving the LLM Cascade}\label{sec: formulation}

\begin{figure*}[h]
\centering
\includegraphics[width=0.8\textwidth]{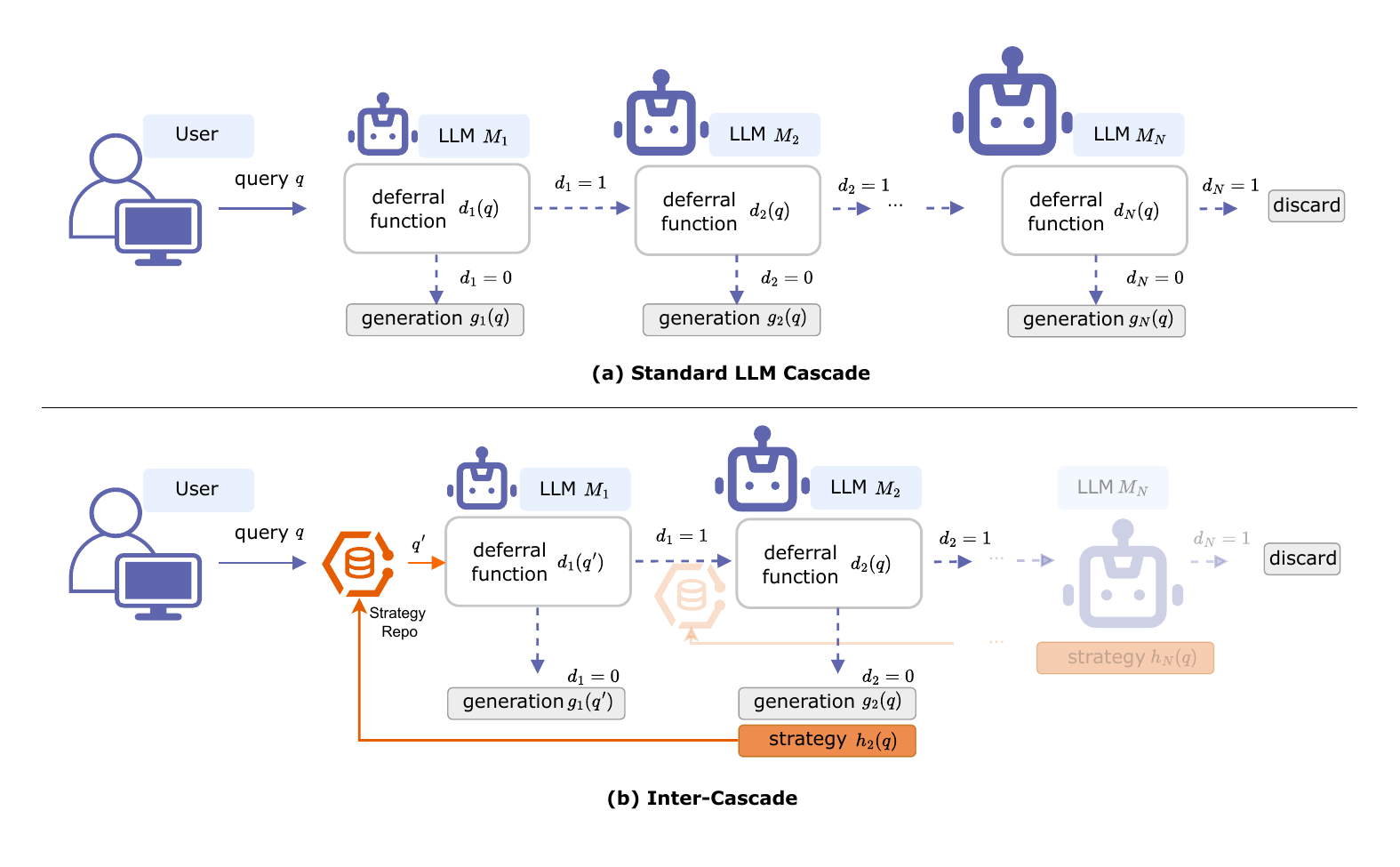}
\vspace{-1em}
\caption{(a) Pipeline of standard LLM Cascade systems. (b) Pipeline of Inter-Cascade. The unique components in Inter-Cascade are painted in orange. For the sake of clarity and readability, we only present the case of two LLMs Inter-Cascade system and the scalable parts beyond two LLMs are rendered in a lighter color.}
\label{fig:inter-cascade}
\end{figure*}

We first describe the standard LLM Cascade~\citep{chen_frugalgpt_2023} and revisit the accuracy bound and calibration method for the deferral threshold~\citet{jung2024trustescalatellmjudges}. We then introduce our proposed method Inter-Cascade and provide a theoretical framework to show when a weak model will be improved by a strong model's strategies. An extended discussion of related work is in Section \ref{sec: related work main} and Appendix \ref{sec: related work}.

\subsection{Standard LLM Cascade}
Figure~\ref{fig:inter-cascade}(a) shows the general $N$-LLM Cascade system~\citep{chen_frugalgpt_2023}. Each LLM \( M_i : i \in [N] \) contains two key components. One is the \emph{generation function} $g_i \colon \mathcal{Q} \rightarrow \mathcal{A}$, where $\mathcal{Q}$ is the space of queries and $\mathcal{A}$ is the space of answers. The other is \text{{\itshape{deferral function}}} $d_i \colon
\mathcal{Q} \rightarrow \{ 0, 1 \}$, which determines whether the $i$-th LLM
will answer the query by itself ($d_i(q) = 0$) or defer it to the $(i+1)$-th LLM ($d_i(q) = 1$).
Processing by the LLMs proceeds sequentially from
$M_1$ to $M_N$. We define a partial order $\preccurlyeq_{\text{wbc}}$ ( ``weaker but cheaper'') to compare models (see Appendix~\ref{appendix：Order of LLM}) and assume that in the cascade, $M_1 \preccurlyeq_{\text{wbc}} M_2\preccurlyeq_{\text{wbc}} ...\preccurlyeq_{\text{wbc}} M_N$.
%
%
%
For each query $q \in \mathcal{Q}$, the first LLM $M_1$ takes the query \( q \)
and gives a final answer $g_1 (q)$ if deferral function $d_1(q) = 0$, otherwise
$M_1$ defers this query to the next LLM $M_2$ if \ $d_1(q) = 1$. If $M_2$ takes
the query from $M_1$, it repeats the same process and so do the other LLMs
except the last model $M_N$. As $M_N$ doesn't have another LLM to offload the
query, $M_N$ discards this query if $d_N(q) = 1$. 
Recent studies propose different deferral functions \( d_i \) to meet the demands in different scenarios. 
We focus on the two-LLM case in the rest of this paper, as shown in Figure~\ref{fig:inter-cascade}(b). 
We call $M_1$ the \emph{Weak LLM} and $M_2$ the \emph{Strong LLM}. One common choice of deferral function is:
\begin{align}\label{eq: d}
    d_i (q) = \left\{ \begin{array}{cl}
     0, & \text{if } c(q) \geq \lambda,\\
     1, & \text{otherwise},
   \end{array} \right.
\end{align}
where $c : \mathcal{Q} \rightarrow [0, 1]$ is
a pre-defined or pre-trained ``confidence'' metric (usually defined in terms of the probability of output tokens) and $\lambda$ is a confidence threshold, which is a hyperparameter that controls the trade-off between the system performance and cost.

\textbf{Accuracy Guaranteed LLM Cascade.}
It is well known that LLMs suffer from systematic bias
~\citep{wang2023largelanguagemodelsfair,thakur2025judgingjudgesevaluatingalignment} and over-confidence~\citep{xiong_can_2024}. To address this, ~\citet{jung2024trustescalatellmjudges} propose a post-hoc calibration
algorithm, which provably guarantees that with the derived $\lambda$,
\begin{equation} 
P \left( g_i (q) = a_{\operatorname{true}} \mid c (q) \geq \lambda \right) \geq 1 - \alpha 
\label{eq:theoretical-promise-for-confidence-threshold}
\end{equation} 
with probability at least $1 - \delta$, as proved in Theorem~1 of their work, where $a_{\operatorname{true}}$ is the ground-truth answer to query $q$. The risk tolerance $\alpha$ and error level $\delta$ are hyperparameters corresponding to the applications and users' demands. To instantiate this guarantee, they first used fixed-sequence testing~\citep{bauer1991multiple} procedure to find the largest threshold $\lambda$ from a calibration set, 
such that $\mathbb{P}\left( g_i (q) = a_{\operatorname{true}} \mid c (q) \geq \lambda \right)$ is exactly and tightly bounded. The procedure is summarized in Algorithm~\ref{alg:calibration}. They also extend the single-model guarantees to the full cascade; see Section 2 and Appendix~A.2 in ~\citet{jung2024trustescalatellmjudges}'s paper for details.
%
\begin{algorithm}
\caption{Calibrating Deferral Threshold $\lambda$\citep{jung2024trustescalatellmjudges}}\label{alg:calibration}
\begin{algorithmic}[1]
\Require Calibration set $(q,a) \in D_{\text{cal}}$, confidence metric $c(\cdot)$, risk tolerance $\alpha$, error level $\delta$
\Ensure Threshold $\lambda$
\State Initialize $\Lambda=\{0.999,0.998,\ldots\}$ in decreasing order
\For{$\lambda \in \Lambda$}
  \State \label{ln:cal-n} $n(\lambda) \gets \sum_{(q,a)\in D_{\text{cal}}}\mathbf{1}\{c(q)\ge \lambda\}$
  \State \label{ln:cal-rhat} $\hat{R}(\lambda) \gets \frac{1}{n(\lambda)}\!\!\sum_{(q,a)\in D_{\text{cal}}}\!\!
               \mathbf{1}\{g_i(q)\neq a_{\operatorname{true}} \land c(q)\ge \lambda\}$
  \State \label{ln:cal-rhatplus} $\hat{R}^+(\lambda) \gets \sup\{R:\; \Pr[\mathrm{Bin}(n(\lambda),R)\le n(\lambda)\hat{R}(\lambda)]\ge \delta\}$
  \If{$\hat{R}^+(\lambda)\le \alpha$} {\Return $\lambda$}
  \EndIf
\EndFor
\end{algorithmic}
\end{algorithm}

 The general pipeline of LLM Cascade is shown in Figure~\ref{fig:inter-cascade}(a). By using this LLM cascade diagram, the deferral function can keep "confident" queries on Weak LLMs and only send "uncertain" queries to Strong LLMs, dramatically reducing at most 82.5\% usage of the strongest LLM as shown by ~\citet{jung2024trustescalatellmjudges} while ensuring the error rate is bounded by $\alpha$ with probability at least $1-\delta$.

\subsection{Interactive LLM Cascade}
LLM Cascade methods can be efficient and reliable although they still incur some  
waste in terms of tokens and latency as noted in Section~\ref{sec: intorduction}. In particular, for workloads in which the Weak LLM is fed a similar or repeated queries for which it chooses to defer, the Strong LLM is called repeatedly to generate the same tokens. To address this issue, we propose \emph{Inter-Cascade}. In Inter-Cascade, for both Weak LLM and Strong LLM, besides deferral function and generation function, we add the following components: \textit{strategy generator} and \emph{strategy repository}. In Strong LLM, we set up a \textit{strategy generator} $h\colon\mathcal{Q}\rightarrow\mathcal{S}$, where $\mathcal{S}$ is the space of strategies. The strategy $s \in \mathcal{S}$ is defined as a sequence of tokens that contains the query and the answer of Strong LLM, together with a generalized ideas or tips to solve logically similar problems.
To store those strategies, we construct a Strategy Repository called $\operatorname{Repo}$. The $\operatorname{Repo}$ is accompanied by a \textit{strategy matching function} $f \colon \mathcal{Q}\times\mathcal{Q}^{N}\rightarrow\mathcal{S}^{k}$, where $N$ is the size of current $\operatorname{Repo}$ and $k$ is a predefined hyperparameter that determines the number of strategies retrieved. 

\paragraph{Strategy Repository.}

The Strategy Repository $\operatorname{Repo}$ is formally
defined as a collection of query-strategy pairs: 
$\operatorname{Repo} = (q_j, s_j)_{j=1}^{N}$ where $q_j \in \mathcal{Q}$ are previously solved queries and $s_j \in \mathcal{S}$ are their corresponding strategies generated by Strong LLM. The strategy matching $f$
operates through multiple stages. 
The repository is initialized as an empty set and dynamically updated: when the Strong LLM generates a strategy $s = h(q)$ for a new query $q$, the pair
$(q,s)$ is added to $\operatorname{Repo}$, enabling future
reuse through the matching function $f$. 

For a query \(q \in \mathcal{Q}\) that is sent to the Weak LLM, let 
\(\operatorname{sim} \colon \mathcal{Q}\times\mathcal{Q}\to[0,1]\) 
be a ranking function.
Let the Top-\(k\) indices (sorted by decreasing similarity) be
\begin{align}
    \operatorname{TopIndex}(q) \triangleq (t_1, t_2, \dots, t_k),
\end{align}
where each \(t_i \in \{1,\dots, N\}\) indexes an item in \(\operatorname{Repo}\) and
\(\operatorname{sim}(q, q_{t_1}) \ge \cdots \ge \operatorname{sim}(q, q_{t_k})\ge\operatorname{sim}(q, q_{\text{else}})\).
After ranking, these strategies with Top-$k$ indexes are chosen to help the Weak LLM. Then the output of strategy matching function is 
$
f(q, \operatorname{Repo})
\triangleq
\{
s^{t_i}
\ \big|\ 
t_i \in \operatorname{TopIndex}(q)
\,\} $.
\begin{remark}
    Compared with finetuning or paying for Strong LLM, the cost of maintaining a $\repo$ and running similarity-based matching algorithms are negligible. According to the estimate formula suggested by~\citet{johnson2019billion}, conducting retrieval and Top-$2$ ranking on $1$ million query embeddings, which are $384$ dimensional vectors (the same size we used in experiments), only requires $0.2$–$0.8$ ms with $70$–$80$ MB GPU VRAM and $80$–$100$ MB RAM for long term storage. The demand can be easily fulfill on any PC or even phone, and imperceptible to human users.
\end{remark}

\begin{algorithm}[H]
\caption{Inter-Cascade Inference Pipeline}
\label{alg:intercascade}
\begin{algorithmic}[1]
\Require Test set $\mathcal{T}=\{q_1,\dots,q_I\}\subseteq\mathcal{Q}$; Weak LLM with deferral function $\DefW$, generation function $\GenW$, strategy repository $\Repo=\emptyset$; strategy matching function $\Match$; Strong LLM with deferral $\DefS$, generator $\GenS$, and strategy generator $\StratGen$.
\Statex \textbf{Deferral convention:} \\ \hspace{0.5cm} $0=\text{handle locally}$, $1=\text{defer/forward}$.
\For{$i \gets 1$ \textbf{to} $I$}
    \State $[s^{t_1}_i, s^{t_2}_i, ..., s^{t_k}_i] \gets \Match(q_i, \Repo)$ \Comment{ \textcolor{ForestGreen}{\textbf{Retrieval}}}
    \State $q^{\prime}_i \gets [q_i, s^{t_1}_i, s^{t_2}_i, ..., s^{t_k}_i] $ \Comment{\textcolor{Gray}{Concatenate strategies}}
    \If{$\DefW(q^{\prime}_i) = 0$} 
    \Comment{\textcolor{ForestGreen}{\textbf{Weak LLM decision}}}
        \State $a_i \gets \GenW(q^{\prime}_i)$ \Comment{\textcolor{Gray}{Answer locally}}
    \Else
        \If{$\DefS(q_i) = 0$}
        \Comment{\textcolor{ForestGreen}{\textbf{Strong LLM decision}}}
            
            \State $s_{\text{new}} \gets \StratGen(q_i)$ \Comment{\textcolor{ForestGreen}{\textbf{Strategy generation}}}
            \State $\Repo \gets \Repo \cup \{(q_i,s_{\text{new}})\}$ \Comment{\textcolor{Gray}{Send back strategy to Weak LLM and store}}
            \State generate answer $a_i \gets \GenS(x_i)$ \Comment{\textcolor{Gray}{Answer at Strong LLM}}
        \Else
            \State Discard current query $q_i$ \Comment{\textcolor{Gray}{None of LLMs are confident to answer the query}}
        \EndIf
    \EndIf
\EndFor
\end{algorithmic}
\end{algorithm}

\paragraph{Inter-Cascade Pipeline.}
The overall pipeline of Inter-Cascade is presented in Algorithm~\ref{alg:intercascade} and in Figure~\ref{fig:inter-cascade}(b).
For each query $q$, the Weak LLM first uses the strategy matching function $f(q, \operatorname{Repo})$ to find the most related strategies.
The query and these strategies are then sent to deferral function.
The augmented input is the prompt concatenation of query and strategies: $q'=[q,s^{t_1}, s^{t_2}, ..., s^{t_k}]$. If the Weak LLM's deferral function $d_1(q')=0$, then final answer $a$ for current query is $g_1(q')$. If $d_1(q')=1$, the query $q'$ is deferred to Strong LLM. Each time the query is sent to the Strong LLM, the deferral function in Strong LLM is called. If $d_2(q) = 0$, this query is discarded (since Strong LLM is the last model in two LLMs Cascade), otherwise $g_2(q)$ produces the answer and further, a new strategy is produced by $h(q)$. Then, the strategy will be stored into $\operatorname{Repo}$. Given $\alpha$ and $\delta$, we can derive the $\lambda$ from Algorithm~\ref{alg:calibration} and determine deferral function $d_1$ and $d_2$ as defined by~\cref{eq: d}.  Our algorithm can be extended to multi-LLM cases, the corresponding Algorithm~\ref{alg:general-intercascade} is shown in Appendix~\ref{app:general-inter}.

\paragraph{Strategies Provide Improved Calibration.}\label{sec: theory}

The $\repo$ we build during the usage of the combination of LLMs collects the strategies of the Strong LLM and provides strategies to help the Weak LLM answer queries. With the help of strategies, the Weak LLM is able to solve the more challenging problems that appear frequently and be more aware of its correctness of answering the queries, leading better confidence. However, it is not clear that how this increment in the accuracy and the quality of confidence could be preserved in the queries after the filtration. After all, all the queries, even to which the Weak LLM answers correctly would be deferred if the Weak LLM's confidence can not pass the threshold. Therefore, we present the following theories to estimate such an increment that would remain in the filtered queries.

To be specific, we first assume that, after adding strategies,
under the same confidence threshold $\lambda$, the number of queries that pass the confidence threshold increases from $n(\lambda)$ to $bn \coloneqq n'(\lambda)$, $b\in[1,\infty)$, where $n(\lambda)$ is first defined in Algorithm~\ref{alg:calibration}. The number of wrongly answered queries before and after the help of strategies are denoted by $x$ and $\epsilon x$, respectively, where $\epsilon\in(0,1)$. We want to understand the potential benefit in terms of the reduction in risk $\alpha$ under the same error level $\delta$. We do not change the threshold $\lambda$, which is the case when the strategy repository is enlarged during the running process of the Inter-Cascade. Theorem~\ref{thm: decrease} states our main result. For the convenience of the statement, we define $\alpha(\epsilon,b)$ as the value of risk tolerance $\alpha$ when total number of queries that pass threshold is $bn$ and incorrectly answered queries is$\epsilon x$.
\begin{theorem}\label{thm: decrease}
Suppose that $\widehat R^{+}(\lambda)$ is a monotonic decreasing function of $\lambda$. Fix $\delta\in(0,1)$ and an integer $n\ge 1$. For $x\in\{0,1,\dots,n\}$, $\epsilon\in(0,1]$, and $b\in[1,\infty)$. 
Suppose that $\min\{\epsilon x+1,\,n-\epsilon x\}$ is moderately large and $1-\delta$ is not an extreme tail,  then:

\smallskip
\noindent\textbf{(a) Decrease in value.}
$\alpha(\epsilon,b)\leq \alpha(1,1)$ when $\epsilon\in(0,1]$ and $b\in[1,\infty)$.

\smallskip
\noindent\textbf{(b) Normal approximation for the amount of decrease.}
Let $z:=\Phiinv(1-\delta)$, where $\Phi$ is the Normal cumulative distribution function, when $n$ is large enough, the decrease of the risk under same level of tolerance is given by,

\begin{equation}
\begin{aligned}
\alpha(1,1) - \alpha(\epsilon,b)
&\approx
\left(\frac{x+1}{n+1}-\frac{\epsilon x+1}{bn+1}\right)
\notag\\[4pt]
&\quad
+ z\Biggl[
\sqrt{\frac{(x+1)(n-x)}{(n+1)^2(n+2)}}
\notag\\[4pt]
&\qquad\qquad
-
\sqrt{\frac{(\epsilon x+1)(bn-\epsilon x)}
           {(bn+1)^2(bn+2)}}
\Biggr].
\end{aligned}
\end{equation}

\end{theorem}

The proof of this theorem is in Appendix~\ref{proof: change of confident queries}. Theorem~\ref{thm: decrease} states that, when the $\delta$ and confidence threshold $\lambda$ do not change, if more queries can pass the threshold, after combining with strategies and under certain conditions, we can ensure a smaller risk tolerance $\alpha
$ in the guarantee of this inequality~(\ref{eq:theoretical-promise-for-confidence-threshold}). That is, Inter-Cascade yields a higher success rate for Weak LLM.

Other than the case that $\lambda$ remains unchanged, which is analyzed above, another case may be that when the users want the same number of queries to be covered by the Weak LLM during two rounds of queries (before and after adding strategies). This case considers the influence of a better Weak LLM on our pipeline. In this case, we instead assume that $n(\lambda) = n(\lambda')$, which ensures the same coverage of Weak LLM. We also show that we can ensure a smaller risk tolerance $\alpha$ when threshold becomes $\lambda'$ while $\delta$ and number of queries that pass threshold remain unchanged. And the reduction in tolerance level $\alpha(1,1)-\alpha(\epsilon, 1)$ is approximately linear to $1-\epsilon$. The full statement of Theorem~\ref{thm: decrease-c} and the proof are shown in Appendix~\ref{proof: change of threshold}.


\section{Experiments}\label{sec: experiments}
\subsection{Benchmarks}

We conduct a comprehensive evaluation on a suite of eight diverse benchmarks. To provide a focused analysis in the main text, we select four representative datasets spanning two primary categories: reasoning-intensive tasks (\textit{GSM-Symbolic}~\citep{mirzadeh_gsm-symbolic_2025}, \textit{GSM-Plus}~\citep{li_gsm-plus_2024}, \textit{MetaMath}~\citep{yu_metamath_2024}) and factual knowledge tasks (\textit{NASA-History-MCQ}~\citep{patrickfleith_NASA_History_MCQ_2025}). While the reasoning datasets evaluate Inter-Cascade's ability to handle structural variations, NASA-History-MCQ is featured specifically for its lack of explicit sample variants. This benchmark serves as a robustness test, allowing us to evaluate whether Inter-Cascade can still enhance efficiency and calibration in general scenarios where the ``similarity phenomenon'' is less pronounced.
These selections highlight our method's adaptability across different difficulty levels. Full results for the remaining four benchmarks, including standard baselines (\textit{GSM8K}~\citep{cobbe_training_2021}, \textit{BigBench Hard}~\citep{Suzgun2022ChallengingBT}) and domain-specific tasks like legal benchmark \textit{BarExamQA}~\citep{zhong_bar_exam_qa_2024} and medicine benchmark \textit{MedMCQA}~\citep{pal_medmcqa_2022}, are detailed in Appendix~\ref{app:extensive}, further demonstrating the framework's generalizability across broader scenarios without explicit query variants. The detailed descriptions of selected benchmarks are in Appendix~\ref{appendix:benchmarks}. The prompt template and an example problem for each benchmark are provided in Appendix~\ref{appendix:prompt_example}.


\subsection{Experimental Settings}
\textbf{Inter-Cascade.} On all benchmarks, \textit{Gemini-2.0-flash} consistently outperforms \textit{GPT-3.5-turbo} (see Table~\ref{tab:single_llm_benchmarks}), and is therefore designated as the Strong LLM in our two-LLM Inter-Cascade, with \textit{GPT-3.5-turbo} as the Weak LLM. We extract the normalized token probability from the LLM's output as confidence score $c(q)$ in following experiments. In preparation phase, with given risk tolerance $\alpha$ and error level $\delta$, we derive desired confidence threshold $\lambda$ from calibration set by following Algo.~\ref{alg:calibration}. Then deploy corresponding deferral functions $d_i$ according to equation~(\ref{eq: d}).

Our similarity-based strategy matching process on $\repo$ works as follows. Given a new query, it is encoded into a vector and used to retrieve the top-$k$ semantically similar queries from $\repo$. We employ the \textit{all-MiniLM-L6-v2} transformer~\citep{reimers-2019-sentence-bert} to produce $384$-dimensional sentence embeddings and use the FAISS library~\citep{douze2025faisslibrary} for efficient approximate nearest-neighbor search. FAISS returns the top-$k$ vectors that minimize cosine distance, providing the Inter-Cascade with prior Strong LLM responses, including queries, answers and strategies, which can inform the Weak LLM's responses.


\begin{table}[ht]
\footnotesize
\caption{Accuracies of the base LLMs on four benchmarks.}
\label{tab:single_llm_benchmarks}
\centering

\setlength{\tabcolsep}{6pt}
\renewcommand{\arraystretch}{1.05}

\begin{tabular}{l l r}
\hline
\textbf{Benchmark} & \textbf{LLM} & \textbf{Accuracy} \\
\hline

\multirow{2}{*}{GSM-Symbolic}
& gpt-3.5-turbo     & 13.36\% \\
& gemini-2.0-flash  & 69.36\% \\
\hline

\multirow{2}{*}{GSM-Plus}
& gpt-3.5-turbo     & 23.00\% \\
& gemini-2.0-flash  & 73.57\% \\
\hline

\multirow{2}{*}{MetaMath}
& gpt-3.5-turbo     & 37.30\% \\
& gemini-2.0-flash  & 79.70\% \\
\hline

\multirow{2}{*}{NASA-History}
& gpt-3.5-turbo     & 65.30\% \\
& gemini-2.0-flash  & 78.80\% \\
\hline

\end{tabular}
\end{table}

\textbf{Inter-Cascade with No Strategies.} 
To isolate the impact of strategy on the result in our pipeline, we only integrate the most similar questions and answers without the problem strategies to query.

\textbf{Inter-Cascade with Random Strategies.} 
To evaluate the impact of similarity-based retrieval on $\repo$, we randomly select the same number of strategies for each query, instead of choosing the top-$k$ most similar queries.

\textbf{Jung Proposed LLM Cascade.} 
To evaluate the performance and effectiveness of the Inter-Cascade, we choose~\citet{jung2024trustescalatellmjudges}'s \textit{Cascaded Selective Evaluation} as the baseline model. Its method for deriving confidence scores and thresholds provides a provable lower bound on the error risk and achieves state-of-the-art performance compared with other confidence-based LLM cascades.

\subsection{Evaluation Metrics}
\label{subsection: evaluation metrics}
We first define the notations used in our evaluation. Let $T$ and $U$ denote the total number of queries and the number of uncovered queries in a benchmark, respectively. Let $N_w$ and $N_s$ be the number of times the Weak and Strong LLMs are invoked, and let $C_w$ and $C_s$ denote the number of queries correctly answered by these models that also pass the confidence threshold. $C_w^\mathrm{total}$ denotes the total number of queries answered correctly by the Weak LLM. Let $\operatorname{Tok}_J$ and $\operatorname{Tok}_O$ be the tokens consumed by Jung's method and our proposed Inter-Cascade pipeline, and $\operatorname{Cost}_J$ and $\operatorname{Cost}_O$ denote their corresponding costs. The evaluation metrics are summarized in Table~\ref{tab:metrics}.



\begin{table}[h]
\footnotesize
\caption{Evaluation Metrics}
\centering
\begingroup

\setlength{\tabcolsep}{5pt}
\renewcommand{\arraystretch}{1.1}
\setlength{\aboverulesep}{0.25ex}
\setlength{\belowrulesep}{0.25ex}

\begin{tabular}{@{} l l @{}}
\toprule
\textbf{Metric} & \textbf{Formula} \\
\midrule

Pipeline Accuracy 
& $(C_w + C_s)/(T - U)$ \\

Strong LLM Call Rate 
& $N_s/T$ \\

Weak LLM Accuracy 
& ${C_w^{\mathrm{total}}}/({T - U})$ \\

Weak Correct Accepted 
& ${C_w}/({T - U})$ \\

Coverage Rate 
& $(T - U)/T$ \\

Token Reduction 
& $(\mathrm{Tok}_J - \mathrm{Tok}_O)/{\mathrm{Tok}_J}$ \\

Cost Reduction 
& $(\mathrm{Cost}_J - \mathrm{Cost}_O)/{\mathrm{Cost}_J}$ \\

\bottomrule
\end{tabular}

\label{tab:metrics}
\endgroup
\end{table}

\subsection{Performance and Cost Analysis}
\label{sec: pipeline result}
\textbf{Inter-Cascade vs. Jung's LLM Cascade.} We evaluate our \textit{Inter-Cascade} pipeline and Jung's method, as shown in Table~\ref{tab:overall-per-dataset-tight}. Our method outperforms Jung's, with a $4.33\% - 6.35\%$ increase in Pipeline Accuracy and $29.92\%-51.93\%$ reduction in Strong LLM Call Rate on reasoning benchmarks. Crucially, on the NASA-History benchmark, which lacks structural variants, Inter-Cascade maintains high accuracy (+0.76\%) while still successfully reducing strong model calls by 15.5\% (relative). These results indicate that \textit{Inter-Cascade} pipeline is beneficial across different categories of tasks and particularly effective for reasoning-intensive tasks. Experiment results on extensive and diverse benchmarks are attached in Appendix~\ref{app:extensive}.

\begin{table}[t]
\centering
\caption{Results across datasets using different pipelines. ``Jung'' denotes Jung's LLM-Cascade and ``Our (Retrieval)'' denotes the Inter-Cascade with similarity-based retrieval. The number of strategies is fixed at $k=2$ for both Inter-Cascade settings. Metrics reported are Pipeline Accuracy (Pipeline Acc.), Strong LLM Call Rate (Strong Call), and Coverage Rate (Cov.). (a) GSM-Symbolic: For the Strong LLM, $\alpha_s = 0.2, \delta_s = 0.8, \lambda_s = 0.47$. For the Weak LLM, $\alpha_w = 0.6, \delta_w = 0.6, \lambda_w = 0.45$. (b) GSM-Plus: For the Strong LLM, $\alpha_s = 0.2, \delta_s = 0.8, \lambda_s = 0.51$. For the Weak LLM, $\alpha_w = 0.6, \delta_w = 0.6, \lambda_w = 0.48$. (c) MetaMath: No threshold is applied for the Strong LLM. For the Weak LLM, $\alpha_w = 0.4, \delta_w = 0.6, \lambda_w = 0.61$. (d) NASA-History: No threshold is applied for the Strong LLM. For the Weak LLM, $\alpha_w = 0.2, \delta_w = 0.7, \lambda_w = 0.87$.}
\label{tab:overall-per-dataset-tight}


\begingroup
\footnotesize  
\setlength{\tabcolsep}{0pt}
\renewcommand{\arraystretch}{1.0} 
\setlength{\aboverulesep}{0.3ex}
\setlength{\belowrulesep}{0.3ex}
\setlength{\cmidrulekern}{0.3em}

\begin{tabular}{@{} l C{5em} C{5em} C{5em} C{5em} @{}}
\toprule

\multirow{2}{*}{\textbf{Benchmark}} &
\multirow{2}{*}{\textbf{Pipeline}} &
\textbf{Pipeline} &
\textbf{Strong} &
\multirow{2}{*}{\textbf{Cov. (\%)}} \\
& &
\textbf{Acc. (\%) $\uparrow$} &
\textbf{Call (\%) $\downarrow$} &
\\
\midrule
\multirow{2}{*}{GSM-Symb.}
& Jung            & 66.04 & 59.37 & 86.31 \\
& Our & \cellcolor{hlblue} \textbf{70.37} & \cellcolor{hlblue} \textbf{30.84} & \cellcolor{hlblue} \textbf{90.35} \\
\midrule
\multirow{2}{*}{GSM-Plus}
& Jung            & 52.78 & 46.29 & 93.57 \\
& Our & \cellcolor{hlblue} \textbf{58.31} 
    & \cellcolor{hlblue} \textbf{32.44} 
    & \cellcolor{hlblue} \textbf{94.79} \\
\midrule
MetaMath
& Jung            & 65.21 & 49.26 & 100.00 \\
(20K) & Our & \cellcolor{hlblue} \textbf{71.56} & \cellcolor{hlblue} \textbf{23.68} & 100.00 \\
\midrule
\multirow{2}{*}{NASA-Hist.}
& Jung            & 71.88 & 26.68 & 100.00 \\
& Our & \cellcolor{hlblue} \textbf{72.64} & \cellcolor{hlblue} \textbf{22.54} & 100.00 \\
\bottomrule
\end{tabular}










\endgroup
\end{table}

\textbf{Impact of Inter-Cascade on Weak LLM.} Having examined the overall pipeline improvements, including Pipeline Accuracy and Strong LLM Call Rate reduction, we now investigate how our proposed \textit{Inter-Cascade} affects the Weak LLM. As shown in Table~\ref{tab:weakllm-per-dataset-tight}, our Weak LLM outperforms the Weak LLM in the other pipeline across all benchmarks. The improvements are particularly pronounced on reasoning benchmarks, with gains of $23.21\%$, $16.2\%$, and $33.06\%$ on MetaMath, GSM-Plus, and GSM-Symbolic, respectively. 
On NASA-History, while the absolute accuracy gain is modest (+0.48\%), the Weak Correct Accepted Rate increases by 3.03\% (from 55.37\% to 58.40\%).
Importantly, improvements in the Weak LLM's accuracy contribute to the pipeline's performance only when the correctly answered queries exceed the confidence threshold. This is captured by the \emph{Weak Correct Accepted} metric in Table~\ref{tab:weakllm-per-dataset-tight}, which represents the proportion of correctly answered queries that surpass the Weak LLM's threshold. 
The observed increase in Weak Correct Accepted shows Strong LLM's strategies helped the weak model better calibrate its confidence, validating our theoretical claim that strategies help the Weak LLM identify correct answers it would otherwise have deferred unnecessarily. This is a crucial factor in converting local improvements into overall pipeline gains. 



\begin{table}[t]
\centering
\caption{Results on Weak LLM across datasets. Reported metrics are Weak LLM Accuracy (Weak Acc.) and Weak Correct Accepted (Weak Corr. Accpt.). Parameter settings are the same as in Table~\ref{tab:overall-per-dataset-tight}.}
\label{tab:weakllm-per-dataset-tight}

\begingroup

\footnotesize  
\setlength{\tabcolsep}{0pt} 
\renewcommand{\arraystretch}{1.0} 
\setlength{\aboverulesep}{0.3ex}
\setlength{\belowrulesep}{0.3ex}
\setlength{\cmidrulekern}{0.3em}

\begin{tabular}{@{} l C{6em} C{6em} C{6em} @{}}
\toprule
\multirow{2}{*}{\textbf{Benchmark}} &
\multirow{2}{*}{\textbf{Pipeline}} &
\textbf{Weak} &
\textbf{Weak Corr.} \\
& & \textbf{Acc. (\%) $\uparrow$} & \textbf{Accpt. (\%) $\uparrow$} \\
\midrule
\multirow{2}{*}{GSM-Symb.}
& Jung            & 15.04 & 12.34 \\
& Our& \cellcolor{hlblue} \textbf{48.10} & \cellcolor{hlblue} \textbf{46.09} \\
\midrule
\multirow{2}{*}{GSM-Plus}
& Jung            & 22.46 & 19.13 \\
& Our   & \cellcolor{hlblue} \textbf{38.66} & \cellcolor{hlblue} \textbf{35.73} \\
\midrule
\multirow{2}{*}{MetaMath(20K)}
& Jung            & 34.95 & 28.54 \\
& Our   & \cellcolor{hlblue} \textbf{58.16} & \cellcolor{hlblue} \textbf{54.07} \\
\midrule
\multirow{2}{*}{NASA-Hist.}
& Jung            & 66.22 & 55.37 \\
& Our   & \cellcolor{hlblue} \textbf{66.70} & \cellcolor{hlblue} \textbf{58.40} \\
\bottomrule
\end{tabular}
\endgroup
\end{table}

\begin{table}[htbp]
\centering
\caption{Token and API cost changes across datasets for Inter-Cascade compared with Jung's pipeline. More detailed analysis with input/output tokens is in Appendix~\ref{app:full cost}: Table~\ref{tab:full-cost}.}
\label{tab:cost}

\begingroup
\scriptsize
\setlength{\tabcolsep}{4pt}
\renewcommand{\arraystretch}{0.9}
\setlength{\aboverulesep}{0.2ex}
\setlength{\belowrulesep}{0.2ex}
\setlength{\cmidrulekern}{0.2em}

\begin{tabular}{lccc}
\toprule

\multirow{2}{*}{\textbf{Benchmark}} &
\multicolumn{1}{c}{ \textbf{Weak LLM Tokens}} &
\multicolumn{1}{c}{\textbf{Strong LLM Tokens}} &
\multirow{2}{*}{\textbf{Token} \textbf{Price}} \\

\cmidrule(lr){2-2} \cmidrule(lr){3-3}

 & Total & Total & \\

\midrule

GSM-Symb.   & +147.66\% & -47.80\% & -49.63\% \\
\midrule

GSM-Plus    & +145.96\% & -29.95\% & -30.41\% \\
\midrule

Meta(20K)  & +127.90\% & -52.18\% & -52.15\% \\
\midrule

NASA-Hist.  & +132.58\% & -15.47\% & -15.75\% \\

\bottomrule
\end{tabular}
\endgroup
\end{table}

\begin{table}[htbp]
\centering
\caption{Processing Latency and Strategy Repository Size across different datasets. Retrieval refers to the time spent on strategies matching and ranking. Generation refers to time spent on generating answer via API.}
\label{tab:efficiency_analysis_moved}

\begingroup
\scriptsize  
\setlength{\tabcolsep}{3.5pt} 
\renewcommand{\arraystretch}{0.9} 
\setlength{\aboverulesep}{0.2ex} 
\setlength{\belowrulesep}{0.2ex}
\setlength{\cmidrulekern}{0.2em}

\begin{tabular}{lcccccc}
\toprule
\multirow{2}{*}{\textbf{Benchmark}} & 
\multirow{2}{*}{\shortstack{\textbf{Tested}\\\textbf{Samples}}} & 
\multicolumn{3}{c}{\textbf{Our}} & 
\multicolumn{1}{c}{\textbf{Jung}} & 
\multirow{2}{*}{\shortstack{\textbf{Repository}\\\textbf{Size}}} \\

\cmidrule(lr){3-5} \cmidrule(lr){6-6}

 & & Total & Retrieval & Generation & Total & \\
\midrule


GSM-Symb. & 11250 & 2.19s & 0.10s & 2.09s & 1.83s & 15.4 MB \\
\midrule

GSM-Plus & 9504 & 1.72s & 0.06s & 1.66s & 1.66s & 12.9 MB \\
\midrule

Meta(20K) & 20000 & 1.60s & 0.06s & 1.54s & 1.54s & 19.6 MB \\
\midrule

NASA-Hist. & 6469 & 1.28s & 0.07s & 1.21s & 1.30s & 8.8 MB \\

\bottomrule
\end{tabular}
\endgroup
\end{table}

\begin{figure}[h]
  \centering
  \begin{subfigure}[t]{0.45\linewidth}
    \centering
    \includegraphics[width=0.9\linewidth,keepaspectratio]{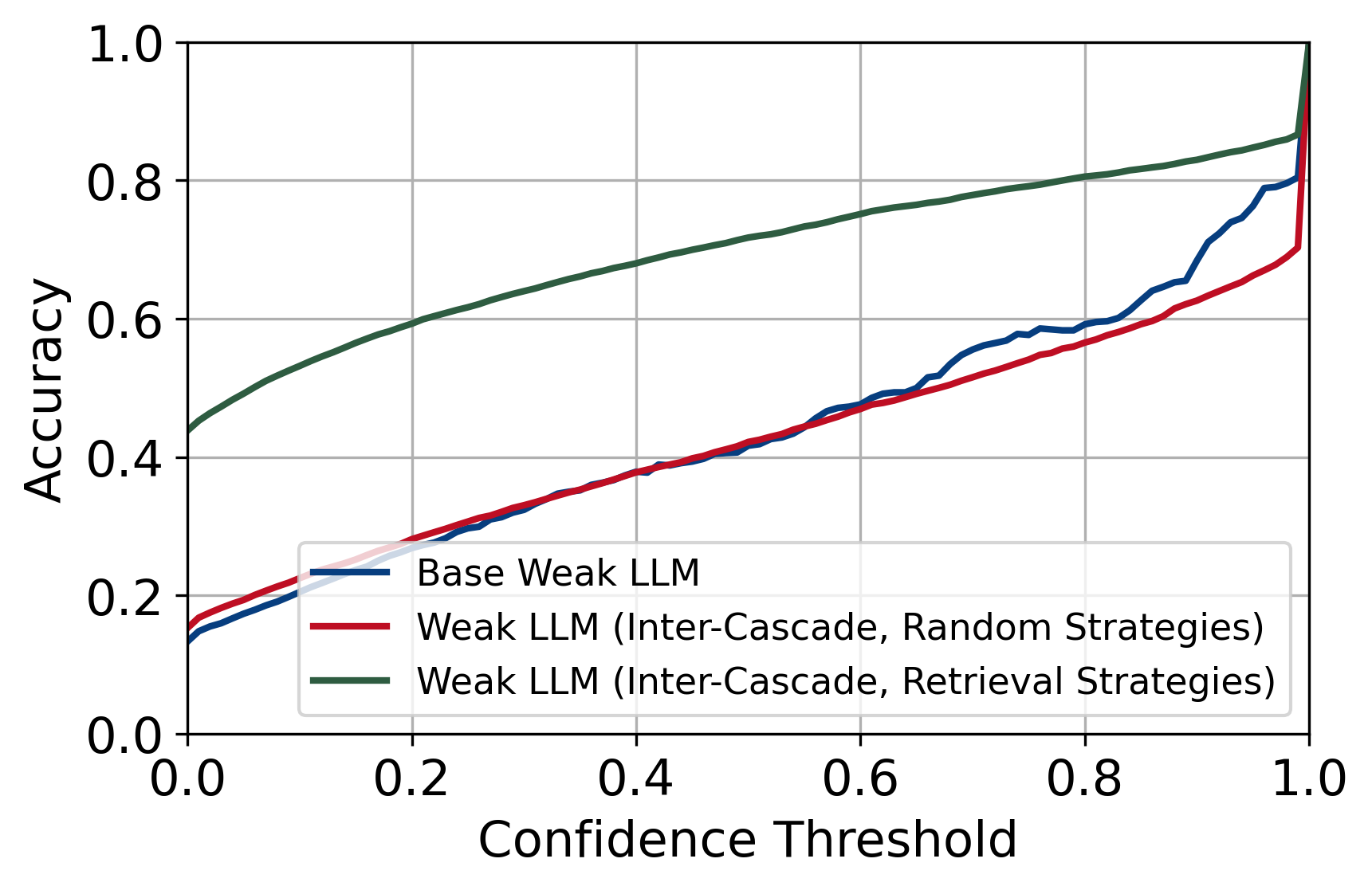}
    \caption{Accuracy vs. Confidence Threshold}
    \label{fig:gsm_symbolic_accuracy_vs_confidence}
  \end{subfigure}
  \begin{subfigure}[t]{0.45\linewidth}
    \centering
    \includegraphics[width=0.9\linewidth,keepaspectratio]{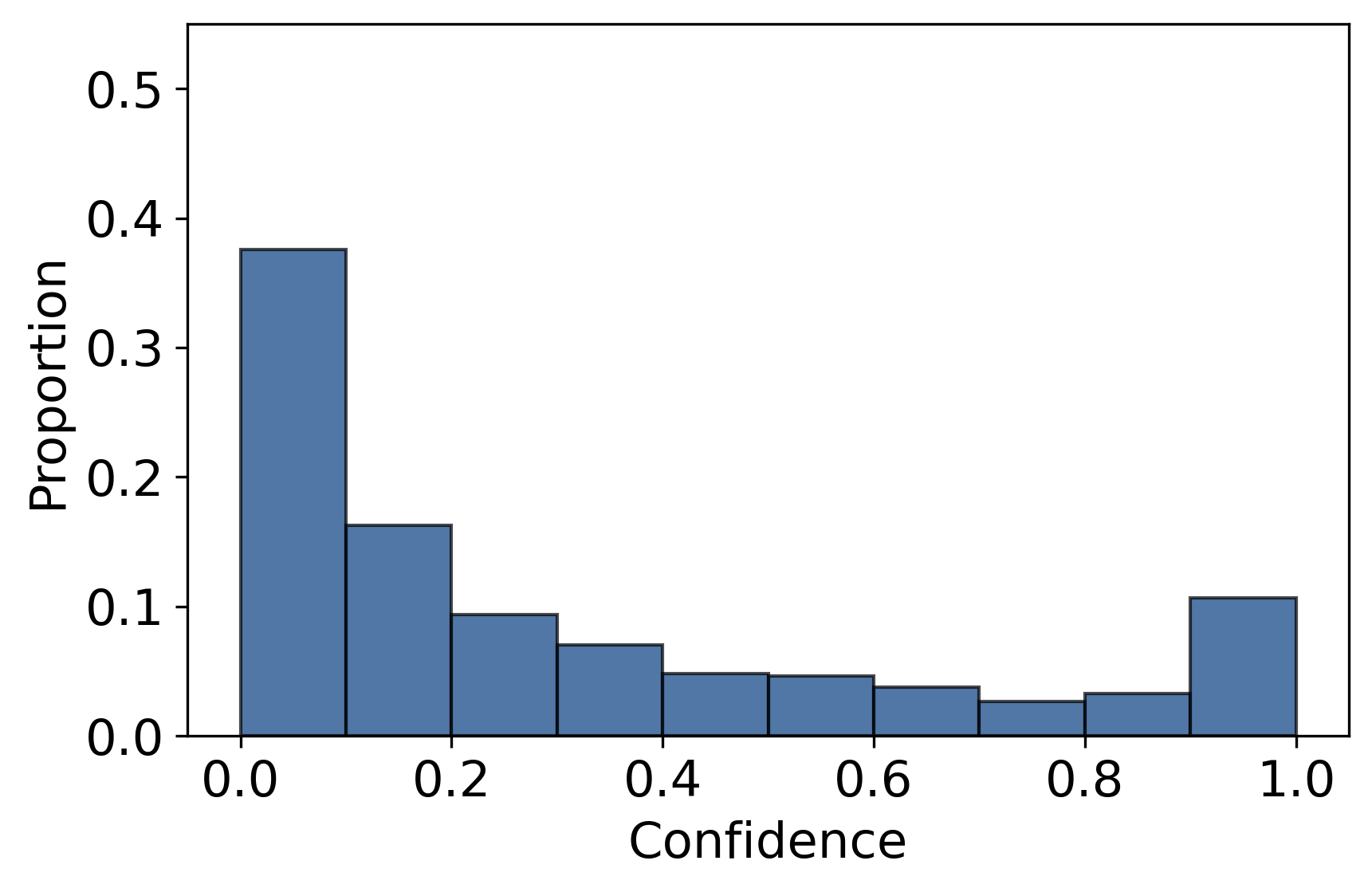}
    \caption{Base Weak LLM}
    \label{fig:gsm_symbolic_conf_base}
  \end{subfigure}


  \begin{subfigure}[t]{0.45\linewidth}
    \centering
    \includegraphics[width=0.9\linewidth,keepaspectratio]{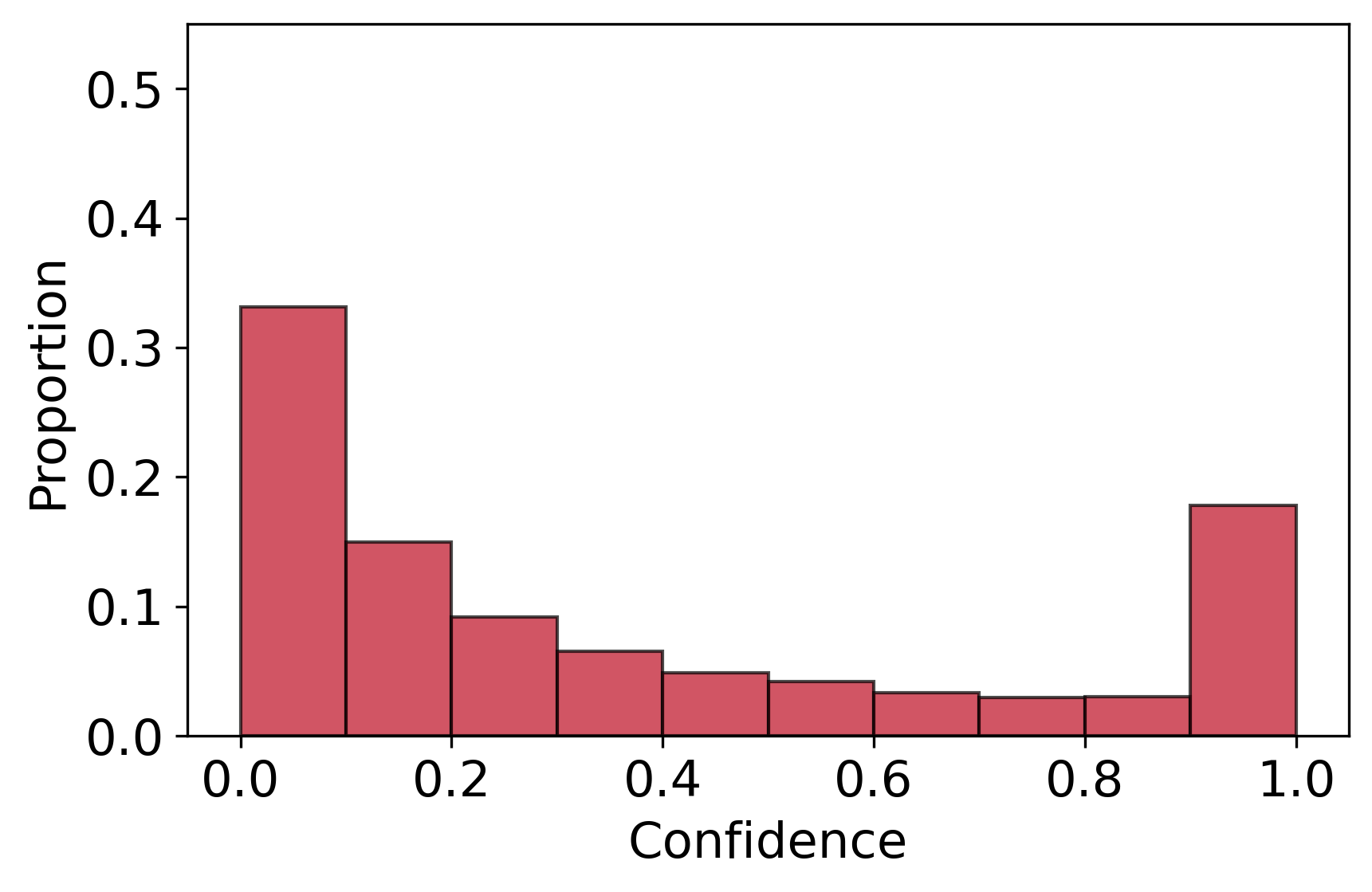}
    \caption{Inter-Cascade Random Strategies}
    \label{fig:gsm_symbolic_conf_random}
  \end{subfigure}
  \begin{subfigure}[t]{0.45\linewidth}
    \centering
    \includegraphics[width=0.9\linewidth,keepaspectratio]{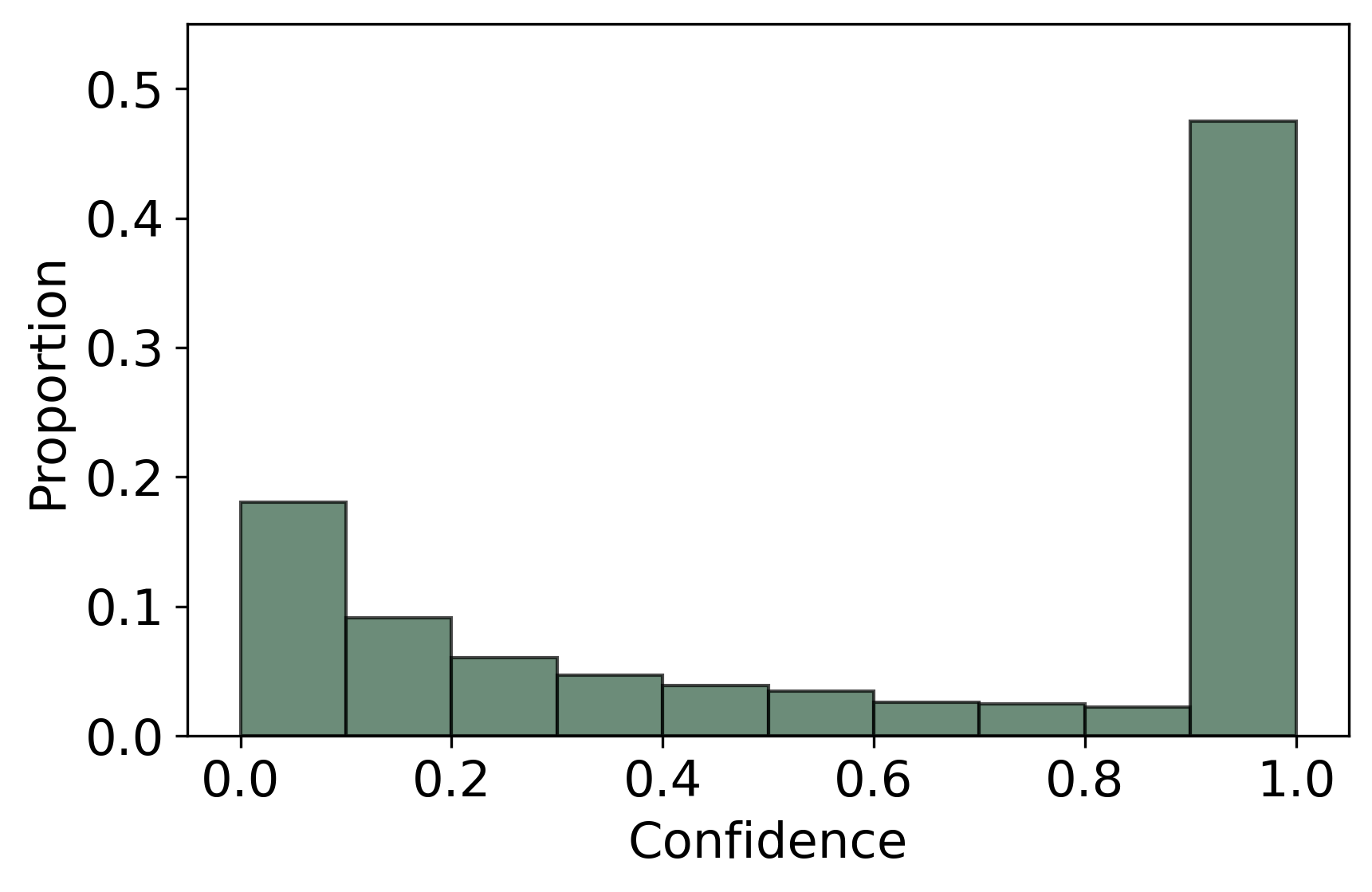}
    \caption{Inter-Cascade Retrieval Strategies}
    \label{fig:gsm_symbolic_conf_retrieval}
  \end{subfigure}
  \caption{GSM-Symbolic dataset: (a) Accuracy as a function of the confidence threshold for the base Weak LLM, Inter-Cascade with random strategies, and Inter-Cascade with retrieval strategies, and (b) - (d) their corresponding confidence histograms. Our Inter-Cascade (Retrieval) consistently concentrates probability mass near high confidence ($0.9$–$1.0$), while the weak and random variants place more mass at low confidence, which explains the accuracy gains observed in (a).}
  \label{fig:gsm_symbolic_2x2}
\end{figure}


\begin{table}[t]
\centering
\caption{Pipeline Accuracy and Strong LLM Call Rate in the ablation study on strategy selection : Our ``No strategy'' (Our NS) vs. Our ``Random'' Rand vs. Our ``Retrieval'' (Ret). Parameter settings are the same as Table \ref{tab:overall-per-dataset-tight}.}
\label{tab:ablation_strategy}
\begingroup


\footnotesize  %
\setlength{\tabcolsep}{0pt} 
\renewcommand{\arraystretch}{0.85} 
\setlength{\aboverulesep}{0.2ex} 
\setlength{\belowrulesep}{0.2ex}
\setlength{\cmidrulekern}{0.2em}

\begin{tabular}{@{} L{7em} L{5em} C{5em} C{5em} C{4em} @{}}
\toprule
\textbf{Bench-} &
\multirow{2}{*}{\textbf{Pipeline}} &
\textbf{Pipeline} &
\textbf{Strong} &
\multirow{2}{*}{\textbf{Cov. (\%)}} \\
\textbf{mark}& &
\textbf{Acc. (\%) $\uparrow$} &
\textbf{Call (\%) $\downarrow$} &
\\
\midrule
\multirow{3}{*}{\shortstack[c]{GSM Symb.}}
& Our (NS) & 67.55 & 65.15 & 83.14\\
& Our (Rand)    & 63.61 & 54.20 & 87.90 \\
& Our (Ret) & \cellcolor{hlblue} \textbf{70.37} & \cellcolor{hlblue} \cellcolor{hlblue} \textbf{30.84} & \cellcolor{hlblue}\textbf{90.35} \\
\midrule
\multirow{3}{*}{\shortstack[c]{GSM Plus}}
& Our (NS) & 58.12& 54.81 & 93.83 \\
& Our (Rand)    & 53.63 & 43.64 & 94.10 \\
& Our (Ret) & \cellcolor{hlblue} \textbf{58.31} & \cellcolor{hlblue} \textbf{32.44} & \cellcolor{hlblue} \textbf{94.79} \\
\midrule
\multirow{3}{*}{\shortstack[c]{MetaMath (20K)}}
& Our (NS) & \cellcolor{hlblue} \textbf{74.48} & 57.32 & 100.00 \\
& Our (Rand)    & 67.85 & 45.99 & 100.00 \\
& Our (Ret) & 71.56 & \cellcolor{hlblue}\textbf{23.68} & 100.00 \\
\midrule
\multirow{3}{*}{\shortstack[c]{NASA-\\Hist.}}
& Our (NS) & \cellcolor{hlblue}\textbf{74.64} & 65.12 & 100.00\\
& Our (Rand)    & 71.32 & 25.09 & 100.00 \\
& Our (Ret) & 72.64 & \cellcolor{hlblue} \textbf{22.54} & 100.00 \\
\bottomrule
\end{tabular}
\endgroup
\end{table}

\begin{table}[t]
\centering
\caption{Weak LLM performance in the ablation study on strategy selection: Our ``No strategy'' (Our NS) vs. Our ``Random'' Rand vs. Our ``Retrieval'' (Ret). Parameter settings are the same as Table \ref{tab:overall-per-dataset-tight}.}
\label{tab:weakllm-ablation}
\begingroup

\footnotesize  %
\setlength{\tabcolsep}{0pt}
\renewcommand{\arraystretch}{0.85} 
\setlength{\aboverulesep}{0.2ex} 
\setlength{\belowrulesep}{0.2ex}
\setlength{\cmidrulekern}{0.2em}

\begin{tabular}{@{} L{8em} L{5em} C{6em} C{6em} C{5em} @{}}
\toprule
\multirow{2}{*}{\textbf{Benchmark}} &
\multirow{2}{*}{\textbf{Pipeline}} &
\textbf{Weak} &
\textbf{Weak Corr.} \\
& & \textbf{Acc. (\%) $\uparrow$} & \textbf{Accpt. (\%) $\uparrow$} \\
\midrule
\multirow{3}{*}{GSM-Symb.}
& Our (NS) & 10.23 & 17.08 \\
& Our (Rand)      & 17.40 & 15.27 \\
& Our (Ret)  & \cellcolor{hlblue} \textbf{48.10} & \cellcolor{hlblue} \textbf{46.09} \\
\midrule
\multirow{3}{*}{GSM-Plus}
& Our (NS) & 20.20 & 17.08 \\
& Our (Rand)      & 25.51 & 22.38 \\
& Our (Ret)   & \cellcolor{hlblue} \textbf{38.66} & \cellcolor{hlblue} \textbf{35.73} \\
\midrule
\multirow{3}{*}{MetaMath(20K)}
& Our (NS) & 33.40 & 28.38 \\
& Our (Rand)      & 38.64 & 32.66 \\
& Our (Ret)   & \cellcolor{hlblue} \textbf{58.16} & \cellcolor{hlblue} \textbf{54.07} \\
\midrule
\multirow{3}{*}{NASA-Hist.}
& Our (NS) & 28.21 & 22.88 \\
& Our (Rand)      & 65.22 & 55.56 \\
& Our (Ret)   & \cellcolor{hlblue} \textbf{66.70} & \cellcolor{hlblue} \textbf{58.40} \\
\bottomrule
\end{tabular}
\endgroup
\end{table}

\textbf{Effect of Strategies on Accuracy and Confidence Calibration.}
As mentioned earlier, one notable observation from our experiments is that providing strategies enhances the Weak LLM's ability to assess its own accuracy. To further investigate this observation, we present 
Figure~\ref{fig:gsm_symbolic_2x2} for the GSM-Symbolic dataset. Analyses for the other three datasets, which exhibit similar patterns, are provided in Appendix~\ref{appendix:sec_confidence_dist}. Figure~\ref{fig:gsm_symbolic_accuracy_vs_confidence} depicts the accuracy of the Weak LLM as a function of the confidence threshold. For each threshold, only queries with confidence equal to or above the threshold are considered, and accuracy is calculated as the proportion of correct predictions. The figure further demonstrates that our pipeline consistently improves the accuracy of queries that pass the threshold. Figures~\ref{fig:gsm_symbolic_conf_base},~\ref{fig:gsm_symbolic_conf_random}, and~\ref{fig:gsm_symbolic_conf_retrieval} illustrate the distribution of query confidence. The histogram offers insight into prediction coverage across different confidence thresholds and shows that our method outperforms the baselines in terms of coverage. Together, these figures indicate that our method not only helps the Weak LLM produce correct answers, but also enables it to better calibrate its confidence by being more confident when the answer is correct and less confident when it is incorrect.

\textbf{Token and API Cost Savings.} Our pipeline not only improves accuracy but also reduces the number of Strong LLM calls, resulting in substantially lower token consumption on Strong LLM. Table~\ref{tab:cost} shows the percentage changes in token usage and corresponding API costs compared with Jung's pipeline. Table~\ref{tab:efficiency_analysis_moved} shows the average processing time per query (including the call of Strong LLM) and the final size of strategies repository across datasets. The results imply that the time difference is between -0.02s and +0.36s, which won't impact the user experience. The size of repository is at level of 10MB+ when the number of queries is at 10K+ level, which can be easily maintained in resource limited settings like mobile or edge device.
More promisingly, accumulated queries and responses can serve as training data for periodic offline fine-tuning the Weak LLM (for example as part of a software update), enabling a self-improving pipeline that dynamically adapts to new data.

\textbf{Ablation Study on Strategy Selection} In order to evaluate the impact of each part when we add strategies to the input of Weak LLM, we conduct ablation experiments for different settings: only adding similar questions and answers (No strategy), adding randomly selected strategies (Random), and our standard Inter-Cascade pipeline (Retrieval). The results in Table~\ref{tab:ablation_strategy} and Table~\ref{tab:weakllm-ablation}, show that the performance of Random Strategy method is between our standard pipeline and Jung's method, while No Strategy is not an acceptable option. Although in benchmarks like NASA-History, the overall accuracy is 2.00\% higher than our standard pipeline, the cost is significant: the Strong Call Rate increase by 42.58\%, which means only adding similar question and answer to the input of Weak LLM would use 2.89x of the Strong LLM. Moreover, the Weak LLM's accuracy would be dramatically undermined by adding non-strategy information to the input of Weak LLM compared to the accuracy for single Weak LLM in Table~\ref{tab:single_llm_benchmarks}. Only adding retrieved question and answers without instructive and generalized problem solving strategy to Weak LLM input is harmful: not only lower the accuracy of Weak LLM, but also call more Strong LLM, which is more expensive. Extensive Ablation studies on \textbf{cold start} of the strategy repository, \textbf{effect of the size} of strategies and different \textbf{selection of LLM pairs} are attached in Appendix~\ref{app:ablation}.

\textbf{Inter-Cascade Robustness under Automatic Strategies.} 
All strategies and their corresponding answers are generated by the Strong LLM in a streaming manner, and any strategy whose confidence exceeds the threshold $\lambda_s$ is automatically accepted. This differentiates \textit{Inter-Cascade} from other LLM augmentation methods such as manually selected in-context learning, few-shot prompting, or static retrieval-augmented generation. Consequently, the strategy repository may contain incorrect strategies. Nonetheless, the results in Table~\ref{tab:overall-per-dataset-tight} and Table~\ref{tab:weakllm-per-dataset-tight} demonstrate the effectiveness of $\lambda_s$ and the robustness of the \textit{Inter-Cascade} pipeline.

\section{Related Work} \label{sec: related work main}

\textbf{LLM Cascades and learning to defer.}
LLM cascades route queries across models of different cost and capability using confidence-based deferral policies, aiming to balance quality and compute~\citep{chen_frugalgpt_2023}.
Recent work explores token-level deferral and post-hoc routing functions~\citep{shen_learning_2024,rayan_learning_2025}, learned routers that decide \emph{before} invoking a stronger model~\citep{ong_routellm_2024}, and cost-aware extensions such as early discarding or rational tuning~\citep{zellinger_cost-saving_2025,zellinger_rational_2025}.
These lines build on learning-with-reject frameworks~\citep{chow_optimum_1957,chow1970optimum,madras_predict_2018,mozannar_consistent_2021,wu_learning_2025}.
However, most deployed cascades remain largely static after training: similar hard queries can repeatedly trigger strong-model calls without transferring knowledge to the weak model.

\textbf{Distillation and retrieval-augmented generation.}
Knowledge distillation transfers capabilities from a strong teacher to a weaker student, typically via (re)training with soft targets or intermediate supervision~\citep{hinton_distilling_2015,romero_fitnets_2015}.
RAG methods instead augment generation with non-parametric memory, usually retrieving from a fixed external corpus~\citep{lewis_retrieval-augmented_2021} or from human-chatbot interaction histories for personalization~\citep{zhang2025dh,mo2025conversational}.
Inter-Cascade connects these directions: when the weak model defers, the strong model produces reusable strategies that are stored and later retrieved to guide future weak-model attempts, yielding an \emph{online, in-context distillation} mechanism at inference time that is complementary to classical distillation and RAG, without parameter updates and human involvement.

\textbf{Other related topics} There are also a weak model and strong model in \textit{Speculative decoding}~\citep{leviathan_fast_2023,narasimhan_faster_2024}, where the weak model works as a answer draft while the strong model works as a verifier to speed up the generation compared to only using strong model. However, in Inter-Cascade, Strong LLM is called only when the Weak LLM is unable to handle current query. \textit{CombLM}~\citep{ormazabal_comblm_2023} and \textit{LLM Debate}~\citep{irving_ai_2018,du_improving_2023,estornell_multi-llm_2024,khan_debating_2024,zhou2025debatereflectdistillmultiagent} are other branches of works that also involve interaction between LLMs. CombLM integrates the logit distribution of two LLMs while LLM Debate requires different LLMs to argue and refine their initial answers and eventually reach consensus through multiple rounds of interaction. The key difference between Inter-Cascade and them is that Inter-Cascade let the Strong LLM and Weak LLM work in a sequential order can conduct early stop to save tokens.

Extensive discussion on related works is in Appendix~\ref{sec: related work}.
\section{Conclusion}
We propose \textit{Inter-Cascade}, an online interactive framework that enables Weak LLMs to learn from Strong LLMs' prior reasoning without fine-tuning. By transforming the strong model into a teacher, Inter-Cascade significantly improves both the weak model’s accuracy and overall system performance while reducing computational costs and reliance on expensive models compared to standard static cascades.

As a general and scalable framework, Inter-Cascade opens several avenues for future research. Immediate improvements could focus on refining strategy generation, optimizing similarity retrieval algorithms, and mitigating context mismatch. Furthermore, the framework is naturally suited for distributed systems, allowing local models to tailor their capabilities by selectively querying Strong LLM. Finally, Inter-Cascade bridges the gap between online and offline learning. The dynamically generated strategy repository not only augments inference in real-time but can also serve as a high-quality dataset for periodic fine-tuning, permanently internalizing the strong model's capabilities. We hope this work inspires further exploration into interactive, teacher-student dynamics within multi-LLM systems.

\section*{Impact Statement}
This paper presents work whose goal is to advance the field of Machine Learning. There are many potential societal consequences of our work, none of which we feel must be specifically highlighted here.

\bibliography{yw}

\begin{thebibliography}{89}
\providecommand{\natexlab}[1]{#1}
\providecommand{\url}[1]{\texttt{#1}}
\expandafter\ifx\csname urlstyle\endcsname\relax
  \providecommand{\doi}[1]{doi: #1}\else
  \providecommand{\doi}{doi: \begingroup \urlstyle{rm}\Url}\fi

\bibitem[Bai et~al.(2024)Bai, Miao, Chen, Wang, Li, Ren, Xie, Yang, and Cai]{bai2024pistis}
Bai, Y., Miao, Y., Chen, L., Wang, D., Li, D., Ren, Y., Xie, H., Yang, C., and Cai, X.
\newblock Pistis-rag: Enhancing retrieval-augmented generation with human feedback.
\newblock \emph{arXiv preprint arXiv:2407.00072}, 2024.

\bibitem[Bauer(1991)]{bauer1991multiple}
Bauer, P.
\newblock Multiple testing in clinical trials.
\newblock \emph{Statistics in medicine}, 10\penalty0 (6):\penalty0 871--890, 1991.

\bibitem[Belcak et~al.(2025)Belcak, Heinrich, Diao, Fu, Dong, Muralidharan, Lin, and Molchanov]{belcak2025small}
Belcak, P., Heinrich, G., Diao, S., Fu, Y., Dong, X., Muralidharan, S., Lin, Y.~C., and Molchanov, P.
\newblock Small language models are the future of agentic ai.
\newblock \emph{arXiv preprint arXiv:2506.02153}, 2025.

\bibitem[Chen et~al.(2024{\natexlab{a}})Chen, Zaharia, and Zou]{chen_frugalgpt_2023}
Chen, L., Zaharia, M., and Zou, J.
\newblock Frugal{GPT}: How to use large language models while reducing cost and improving performance.
\newblock \emph{Transactions on Machine Learning Research}, 2024{\natexlab{a}}.
\newblock ISSN 2835-8856.
\newblock URL \url{https://openreview.net/forum?id=cSimKw5p6R}.

\bibitem[Chen et~al.(2024{\natexlab{b}})Chen, Zaharia, and Zou]{chen_frugalgpt_2024}
Chen, L., Zaharia, M., and Zou, J.
\newblock {{FrugalGPT}}: {{How}} to {{Use Large Language Models While Reducing Cost}} and {{Improving Performance}}.
\newblock \emph{Transactions on Machine Learning Research}, July 2024{\natexlab{b}}.
\newblock ISSN 2835-8856.

\bibitem[Chen et~al.(2025{\natexlab{a}})Chen, Tao, Zhu, Xi, Guo, Wang, Wang, and Lan]{chen2025comrag}
Chen, Q., Tao, W., Zhu, Z., Xi, M., Guo, L., Wang, Y., Wang, W., and Lan, Y.
\newblock Comrag: Retrieval-augmented generation with dynamic vector stores for real-time community question answering in industry.
\newblock \emph{arXiv preprint arXiv:2506.21098}, 2025{\natexlab{a}}.

\bibitem[Chen et~al.(2025{\natexlab{b}})Chen, Li, Chen, Li, Sun, Luo, Mao, Li, Xiao, Yang, et~al.]{chen2025harnessing}
Chen, Z., Li, J., Chen, P., Li, Z., Sun, K., Luo, Y., Mao, Q., Li, M., Xiao, L., Yang, D., et~al.
\newblock Harnessing multiple large language models: A survey on llm ensemble.
\newblock \emph{arXiv preprint arXiv:2502.18036}, 2025{\natexlab{b}}.

\bibitem[Chow(1970)]{chow1970optimum}
Chow, C.
\newblock On optimum recognition error and reject tradeoff.
\newblock \emph{IEEE Transactions on information theory}, 16\penalty0 (1):\penalty0 41--46, 1970.

\bibitem[Chow(1957)]{chow_optimum_1957}
Chow, C.~K.
\newblock An optimum character recognition system using decision functions.
\newblock \emph{IRE Transactions on Electronic Computers}, EC-6\penalty0 (4):\penalty0 247--254, December 1957.
\newblock ISSN 0367-9950.
\newblock URL \url{https://doi.org/10.1109/TEC.1957.5222035}.

\bibitem[Chuang et~al.(2025)Chuang, Zhou, Sarma, Gopalan, Boccio, Bolouki, and Hu]{chuang_learning_2025}
Chuang, Y.-N., Zhou, H., Sarma, P.~K., Gopalan, P., Boccio, J., Bolouki, S., and Hu, X.
\newblock Learning to {{Route LLMs}} with {{Confidence Tokens}}.
\newblock In \emph{Proceedings of the Forty-Second International Conference on Machine Learning}. PMLR, 2025.

\bibitem[Cobbe et~al.(2021{\natexlab{a}})Cobbe, Kosaraju, Bavarian, Chen, Jun, Kaiser, Plappert, Tworek, Hilton, Nakano, Hesse, and Schulman]{cobbe_training_2021}
Cobbe, K., Kosaraju, V., Bavarian, M., Chen, M., Jun, H., Kaiser, L., Plappert, M., Tworek, J., Hilton, J., Nakano, R., Hesse, C., and Schulman, J.
\newblock Training {{Verifiers}} to {{Solve Math Word Problems}}.
\newblock Technical Report 2110.14168, arXiv, November 2021{\natexlab{a}}.
\newblock URL \url{https://doi.org/10.48550/arXiv.2110.14168}.

\bibitem[Cobbe et~al.(2021{\natexlab{b}})Cobbe, Kosaraju, Bavarian, Chen, Jun, Kaiser, Plappert, Tworek, Hilton, Nakano, Hesse, and Schulman]{cobbe_training_2021-1}
Cobbe, K., Kosaraju, V., Bavarian, M., Chen, M., Jun, H., Kaiser, L., Plappert, M., Tworek, J., Hilton, J., Nakano, R., Hesse, C., and Schulman, J.
\newblock Training {{Verifiers}} to {{Solve Math Word Problems}}.
\newblock Technical Report 2110.14168, arXiv, November 2021{\natexlab{b}}.
\newblock URL \url{https://doi.org/10.48550/arXiv.2110.14168}.

\bibitem[Cortes et~al.(2016)Cortes, DeSalvo, and Mohri]{ortner_learning_2016}
Cortes, C., DeSalvo, G., and Mohri, M.
\newblock Learning with {{Rejection}}.
\newblock In Ortner, R., Simon, H.~U., and Zilles, S. (eds.), \emph{Algorithmic {{Learning Theory}}}, volume 9925, pp.\  67--82. Springer International Publishing, Cham, 2016.
\newblock ISBN 978-3-319-46378-0 978-3-319-46379-7.
\newblock URL \url{https://doi.org/10.1007/978-3-319-46379-7_5}.

\bibitem[Dong et~al.(2024)Dong, Li, Dai, Zheng, Ma, Li, Xia, Xu, Wu, Chang, Sun, Li, and Sui]{dong2024surveyincontextlearning}
Dong, Q., Li, L., Dai, D., Zheng, C., Ma, J., Li, R., Xia, H., Xu, J., Wu, Z., Chang, B., Sun, X., Li, L., and Sui, Z.
\newblock A survey on in-context learning.
\newblock In Al-Onaizan, Y., Bansal, M., and Chen, Y.-N. (eds.), \emph{Proceedings of the 2024 Conference on Empirical Methods in Natural Language Processing}, pp.\  1107--1128, Miami, Florida, USA, November 2024. Association for Computational Linguistics.
\newblock \doi{10.18653/v1/2024.emnlp-main.64}.
\newblock URL \url{https://aclanthology.org/2024.emnlp-main.64/}.

\bibitem[Douze et~al.(2025)Douze, Guzhva, Deng, Johnson, Szilvasy, Mazar{\'e}, Lomeli, Hosseini, and J{\'e}gou]{douze2025faisslibrary}
Douze, M., Guzhva, A., Deng, C., Johnson, J., Szilvasy, G., Mazar{\'e}, P.-E., Lomeli, M., Hosseini, L., and J{\'e}gou, H.
\newblock The {{Faiss}} library.
\newblock Technical Report 2401.08281, arXiv, 2025.
\newblock URL \url{https://arxiv.org/abs/2401.08281}.

\bibitem[Du et~al.(2023)Du, Li, Torralba, Tenenbaum, and Mordatch]{du_improving_2023}
Du, Y., Li, S., Torralba, A., Tenenbaum, J.~B., and Mordatch, I.
\newblock Improving factuality and reasoning in language models through multiagent debate.
\newblock In \emph{Forty-first International Conference on Machine Learning}, 2023.

\bibitem[Edge et~al.(2025)Edge, Trinh, Cheng, Bradley, Chao, Mody, Truitt, Metropolitansky, Ness, and Larson]{edge_local_2025}
Edge, D., Trinh, H., Cheng, N., Bradley, J., Chao, A., Mody, A., Truitt, S., Metropolitansky, D., Ness, R.~O., and Larson, J.
\newblock From {{Local}} to {{Global}}: {{A Graph RAG Approach}} to {{Query-Focused Summarization}}.
\newblock Technical report, arXiv, February 2025.
\newblock URL \url{https://doi.org/10.48550/arXiv.2404.16130}.

\bibitem[Estornell \& Liu(2024)Estornell and Liu]{estornell_multi-llm_2024}
Estornell, A. and Liu, Y.
\newblock Multi-{{LLM Debate}}: {{Framework}}, {{Principals}}, and {{Interventions}}.
\newblock In \emph{The {{Thirty-eighth Annual Conference}} on {{Neural Information Processing Systems}}}, November 2024.

\bibitem[Fleith(2025)]{patrickfleith_NASA_History_MCQ_2025}
Fleith, P.
\newblock {NASA-history-MCQ}.
\newblock Dataset, Hugging Face, 2025.
\newblock URL \url{https://huggingface.co/datasets/patrickfleith/NASA-History-MCQ}.

\bibitem[Guti{\'e}rrez et~al.(2025)Guti{\'e}rrez, Shu, Qi, Zhou, and Su]{gutierrez2025rag}
Guti{\'e}rrez, B.~J., Shu, Y., Qi, W., Zhou, S., and Su, Y.
\newblock From rag to memory: Non-parametric continual learning for large language models.
\newblock \emph{arXiv preprint arXiv:2502.14802}, 2025.

\bibitem[Han et~al.(2025)Han, Xia, Zhang, Sun, Li, Zhu, and Yao]{han2025mdocagent}
Han, S., Xia, P., Zhang, R., Sun, T., Li, Y., Zhu, H., and Yao, H.
\newblock Mdocagent: A multi-modal multi-agent framework for document understanding.
\newblock \emph{arXiv preprint arXiv:2503.13964}, 2025.

\bibitem[Hendrycks et~al.(2021)Hendrycks, Burns, Kadavath, Arora, Basart, Tang, Song, and Steinhardt]{hendrycks2021measuringmathematicalproblemsolving}
Hendrycks, D., Burns, C., Kadavath, S., Arora, A., Basart, S., Tang, E., Song, D., and Steinhardt, J.
\newblock Measuring mathematical problem solving with the math dataset.
\newblock \emph{NeurIPS}, 2021.

\bibitem[Herbei \& Wegkamp(2006)Herbei and Wegkamp]{herbei_classification_2006}
Herbei, R. and Wegkamp, M.~H.
\newblock Classification with {{Reject Option}}.
\newblock \emph{The Canadian Journal of Statistics / La Revue Canadienne de Statistique}, 34\penalty0 (4):\penalty0 709--721, 2006.
\newblock ISSN 0319-5724.

\bibitem[Hinton et~al.(2015)Hinton, Vinyals, and Dean]{hinton_distilling_2015}
Hinton, G., Vinyals, O., and Dean, J.
\newblock Distilling the {{Knowledge}} in a {{Neural Network}}.
\newblock Technical Report 1503.02531, arXiv, March 2015.
\newblock URL \url{https://doi.org/10.48550/arXiv.1503.02531}.

\bibitem[Irving et~al.(2018)Irving, Christiano, and Amodei]{irving_ai_2018}
Irving, G., Christiano, P., and Amodei, D.
\newblock {{AI}} safety via debate.
\newblock Technical Report 1805.00899, arXiv, October 2018.
\newblock URL \url{https://doi.org/10.48550/arXiv.1805.00899}.

\bibitem[Jiang et~al.(2024)Jiang, Shi, Yu, Liu, Zhang, Li, and Kwok]{jiang_forward-backward_2024}
Jiang, W., Shi, H., Yu, L., Liu, Z., Zhang, Y., Li, Z., and Kwok, J.
\newblock Forward-backward reasoning in large language models for mathematical verification.
\newblock In Ku, L.-W., Martins, A., and Srikumar, V. (eds.), \emph{Findings of the Association for Computational Linguistics: ACL 2024}, pp.\  6647--6661, Bangkok, Thailand, August 2024. Association for Computational Linguistics.
\newblock \doi{10.18653/v1/2024.findings-acl.397}.
\newblock URL \url{https://aclanthology.org/2024.findings-acl.397/}.

\bibitem[Jitkrittum et~al.(2023)Jitkrittum, Gupta, Menon, Narasimhan, Rawat, and Kumar]{jitkrittum_when_2023}
Jitkrittum, W., Gupta, N., Menon, A.~K., Narasimhan, H., Rawat, A., and Kumar, S.
\newblock When {{Does Confidence-Based Cascade Deferral Suffice}}?
\newblock \emph{Advances in Neural Information Processing Systems}, 36:\penalty0 9891--9906, December 2023.

\bibitem[Johnson et~al.(2021)Johnson, Douze, and Jégou]{johnson2019billion}
Johnson, J., Douze, M., and Jégou, H.
\newblock Billion-scale similarity search with gpus.
\newblock \emph{IEEE Transactions on Big Data}, 7\penalty0 (3):\penalty0 535--547, 2021.
\newblock URL \url{https://doi.org/10.1109/TBDATA.2019.2921572}.

\bibitem[Joshi et~al.(2024)Joshi, Liu, Xun, Lin, and Foo]{joshi_representation_2024}
Joshi, C.~K., Liu, F., Xun, X., Lin, J., and Foo, C.-S.
\newblock On {{Representation Knowledge Distillation}} for {{Graph Neural Networks}}.
\newblock \emph{IEEE Transactions on Neural Networks and Learning Systems}, 35\penalty0 (4):\penalty0 4656--4667, April 2024.
\newblock ISSN 2162-237X, 2162-2388.
\newblock URL \url{https://doi.org/10.1109/TNNLS.2022.3223018}.

\bibitem[Jung et~al.(2025)Jung, Brahman, and Choi]{jung2024trustescalatellmjudges}
Jung, J., Brahman, F., and Choi, Y.
\newblock Trust or escalate: Llm judges with provable guarantees for human agreement.
\newblock In Yue, Y., Garg, A., Peng, N., Sha, F., and Yu, R. (eds.), \emph{International Conference on Representation Learning}, volume 2025, pp.\  3101--3125, 2025.
\newblock URL \url{https://proceedings.iclr.cc/paper_files/paper/2025/file/08dabd5345b37fffcbe335bd578b15a0-Paper-Conference.pdf}.

\bibitem[Kaplan et~al.(2020)Kaplan, McCandlish, Henighan, Brown, Chess, Child, Gray, Radford, Wu, and Amodei]{kaplan2020scalinglawsneurallanguage}
Kaplan, J., McCandlish, S., Henighan, T., Brown, T.~B., Chess, B., Child, R., Gray, S., Radford, A., Wu, J., and Amodei, D.
\newblock Scaling laws for neural language models.
\newblock Technical Report 2001.08361, arXiv, 2020.
\newblock URL \url{https://arxiv.org/abs/2001.08361}.

\bibitem[Khan et~al.(2024)Khan, Hughes, Valentine, Ruis, Sachan, Radhakrishnan, Grefenstette, Bowman, Rockt\"{a}schel, and Perez]{khan_debating_2024}
Khan, A., Hughes, J., Valentine, D., Ruis, L., Sachan, K., Radhakrishnan, A., Grefenstette, E., Bowman, S.~R., Rockt\"{a}schel, T., and Perez, E.
\newblock Debating with more persuasive llms leads to more truthful answers.
\newblock In \emph{Proceedings of the 41st International Conference on Machine Learning}, ICML'24. JMLR.org, 2024.

\bibitem[Lee et~al.(2023)Lee, Park, Seo, and Kang]{lee_self-knowledge_2022}
Lee, H., Park, Y., Seo, H., and Kang, M.
\newblock Self-knowledge distillation via dropout.
\newblock \emph{Comput. Vis. Image Underst.}, 233\penalty0 (C), August 2023.
\newblock ISSN 1077-3142.
\newblock \doi{10.1016/j.cviu.2023.103720}.
\newblock URL \url{https://doi.org/10.1016/j.cviu.2023.103720}.

\bibitem[Leviathan et~al.(2023)Leviathan, Kalman, and Matias]{leviathan_fast_2023}
Leviathan, Y., Kalman, M., and Matias, Y.
\newblock Fast inference from transformers via speculative decoding.
\newblock In \emph{Proceedings of the 40th International Conference on Machine Learning}, ICML'23. JMLR.org, 2023.

\bibitem[Lewis et~al.(2020)Lewis, Perez, Piktus, Petroni, Karpukhin, Goyal, K{\"u}ttler, Lewis, Yih, Rockt{\"a}schel, et~al.]{lewis_retrieval-augmented_2021}
Lewis, P., Perez, E., Piktus, A., Petroni, F., Karpukhin, V., Goyal, N., K{\"u}ttler, H., Lewis, M., Yih, W.-t., Rockt{\"a}schel, T., et~al.
\newblock Retrieval-augmented generation for knowledge-intensive nlp tasks.
\newblock \emph{Advances in neural information processing systems}, 33:\penalty0 9459--9474, 2020.

\bibitem[Li et~al.(2024)Li, Cui, Zhao, Kong, and Bi]{li_gsm-plus_2024}
Li, Q., Cui, L., Zhao, X., Kong, L., and Bi, W.
\newblock {GSM}-plus: A comprehensive benchmark for evaluating the robustness of {LLM}s as mathematical problem solvers.
\newblock In Ku, L.-W., Martins, A., and Srikumar, V. (eds.), \emph{Proceedings of the 62nd Annual Meeting of the Association for Computational Linguistics (Volume 1: Long Papers)}, pp.\  2961--2984, Bangkok, Thailand, August 2024. Association for Computational Linguistics.
\newblock \doi{10.18653/v1/2024.acl-long.163}.
\newblock URL \url{https://aclanthology.org/2024.acl-long.163/}.

\bibitem[Li et~al.(2025)Li, Zhang, Yang, Huang, Wu, Luo, Bei, Zou, Luo, Zhao, et~al.]{li2025towards}
Li, Y., Zhang, W., Yang, Y., Huang, W.-C., Wu, Y., Luo, J., Bei, Y., Zou, H.~P., Luo, X., Zhao, Y., et~al.
\newblock Towards agentic rag with deep reasoning: A survey of rag-reasoning systems in llms.
\newblock \emph{arXiv preprint arXiv:2507.09477}, 2025.

\bibitem[Liu et~al.(2024)Liu, Zhao, Kuang, Kang, Jiang, Sun, and Wu]{liu2024evolvingknowledgedistillationlarge}
Liu, C., Zhao, F., Kuang, K., Kang, Y., Jiang, Z., Sun, C., and Wu, F.
\newblock Evolving knowledge distillation with large language models and active learning.
\newblock In Calzolari, N., Kan, M.-Y., Hoste, V., Lenci, A., Sakti, S., and Xue, N. (eds.), \emph{Proceedings of the 2024 Joint International Conference on Computational Linguistics, Language Resources and Evaluation (LREC-COLING 2024)}, pp.\  6717--6731, Torino, Italia, May 2024. ELRA and ICCL.
\newblock URL \url{https://aclanthology.org/2024.lrec-main.593/}.

\bibitem[Liu et~al.(2025)Liu, Liu, Yao, Liu, Meng, Wang, and Ma]{liu2025hm}
Liu, P., Liu, X., Yao, R., Liu, J., Meng, S., Wang, D., and Ma, J.
\newblock Hm-rag: Hierarchical multi-agent multimodal retrieval augmented generation.
\newblock In \emph{Proceedings of the 33rd ACM International Conference on Multimedia}, pp.\  2781--2790, 2025.

\bibitem[Low et~al.(2025)Low, Wang, Zhang, Zeng, Zhuo, Mazomenos, and Jin]{low2025surgraw}
Low, C.~H., Wang, Z., Zhang, T., Zeng, Z., Zhuo, Z., Mazomenos, E.~B., and Jin, Y.
\newblock Surgraw: Multi-agent workflow with chain-of-thought reasoning for surgical intelligence.
\newblock \emph{arXiv preprint arXiv:2503.10265}, 2025.

\bibitem[Madras et~al.(2018)Madras, Pitassi, and Zemel]{madras_predict_2018}
Madras, D., Pitassi, T., and Zemel, R.
\newblock Predict responsibly: improving fairness and accuracy by learning to defer.
\newblock In \emph{Proceedings of the 32nd International Conference on Neural Information Processing Systems}, NIPS'18, pp.\  6150–6160, Red Hook, NY, USA, 2018. Curran Associates Inc.

\bibitem[Mao et~al.(2024{\natexlab{a}})Mao, Mohri, and Zhong]{barneva_principled_2024}
Mao, A., Mohri, M., and Zhong, Y.
\newblock Principled {{Approaches}} for {{Learning}} to {{Defer}} with {{Multiple Experts}}.
\newblock In Barneva, R.~P., Brimkov, V.~E., Gentile, C., and Pacchiano, A. (eds.), \emph{Artificial {{Intelligence}} and {{Image Analysis}}}, volume 14494, pp.\  107--135. Springer Nature Switzerland, Cham, 2024{\natexlab{a}}.
\newblock ISBN 978-3-031-63734-6 978-3-031-63735-3.
\newblock URL \url{https://doi.org/10.1007/978-3-031-63735-3_7}.

\bibitem[Mao et~al.(2024{\natexlab{b}})Mao, Mohri, and Zhong]{mao_theoretically_2024}
Mao, A., Mohri, M., and Zhong, Y.
\newblock Theoretically {{Grounded Loss Functions}} and {{Algorithms}} for {{Score-Based Multi-Class Abstention}}.
\newblock In \emph{Proceedings of {{The}} 27th {{International Conference}} on {{Artificial Intelligence}} and {{Statistics}}}, pp.\  4753--4761. PMLR, April 2024{\natexlab{b}}.

\bibitem[Margatina et~al.(2023)Margatina, Schick, Aletras, and Dwivedi-Yu]{margatina_active_2023-1}
Margatina, K., Schick, T., Aletras, N., and Dwivedi-Yu, J.
\newblock Active learning principles for in-context learning with large language models.
\newblock In \emph{Findings of the Association for Computational Linguistics: EMNLP 2023}, pp.\  5011--5034, December 2023.
\newblock URL \url{https://aclanthology.org/2023.findings-emnlp.334/}.

\bibitem[Mirzadeh et~al.(2025)Mirzadeh, Alizadeh, Shahrokhi, Tuzel, Bengio, and Farajtabar]{mirzadeh_gsm-symbolic_2025}
Mirzadeh, S.~I., Alizadeh, K., Shahrokhi, H., Tuzel, O., Bengio, S., and Farajtabar, M.
\newblock {GSM}-symbolic: Understanding the limitations of mathematical reasoning in large language models.
\newblock In \emph{The Thirteenth International Conference on Learning Representations}, 2025.
\newblock URL \url{https://openreview.net/forum?id=AjXkRZIvjB}.

\bibitem[Mo et~al.(2025)Mo, Meng, Aliannejadi, and Nie]{mo2025conversational}
Mo, F., Meng, C., Aliannejadi, M., and Nie, J.-Y.
\newblock Conversational search: From fundamentals to frontiers in the llm era.
\newblock In \emph{Proceedings of the 48th International ACM SIGIR Conference on Research and Development in Information Retrieval}, pp.\  4094--4097, 2025.

\bibitem[Mozannar \& Sontag(2020)Mozannar and Sontag]{mozannar_consistent_2021}
Mozannar, H. and Sontag, D.
\newblock Consistent estimators for learning to defer to an expert.
\newblock In III, H.~D. and Singh, A. (eds.), \emph{Proceedings of the 37th International Conference on Machine Learning}, volume 119 of \emph{Proceedings of Machine Learning Research}, pp.\  7076--7087. PMLR, 13--18 Jul 2020.
\newblock URL \url{https://proceedings.mlr.press/v119/mozannar20b.html}.

\bibitem[Narasimhan et~al.(2025)Narasimhan, Jitkrittum, Rawat, Kim, Gupta, Menon, and Kumar]{narasimhan_faster_2024}
Narasimhan, H., Jitkrittum, W., Rawat, A.~S., Kim, S., Gupta, N., Menon, A.~K., and Kumar, S.
\newblock Faster cascades via speculative decoding.
\newblock In \emph{The Thirteenth International Conference on Learning Representations}, 2025.
\newblock URL \url{https://openreview.net/forum?id=vo9t20wsmd}.

\bibitem[Nguyen et~al.(2025{\natexlab{a}})Nguyen, Do, and Carneiro]{nguyen_probabilistic_2025}
Nguyen, C.~C., Do, T.-T., and Carneiro, G.
\newblock Probabilistic learning to defer: Handling missing expert annotations and controlling workload distribution.
\newblock In \emph{The Thirteenth International Conference on Learning Representations}, 2025{\natexlab{a}}.
\newblock URL \url{https://openreview.net/forum?id=zl0HLZOJC9}.

\bibitem[Nguyen et~al.(2025{\natexlab{b}})Nguyen, Chin, and Tai]{nguyen2025ma}
Nguyen, T., Chin, P., and Tai, Y.-W.
\newblock Ma-rag: Multi-agent retrieval-augmented generation via collaborative chain-of-thought reasoning.
\newblock \emph{arXiv preprint arXiv:2505.20096}, 2025{\natexlab{b}}.

\bibitem[Nie et~al.(2024)Nie, Ding, Hu, Jermaine, and Chaudhuri]{nie_online_2024}
Nie, L., Ding, Z., Hu, E., Jermaine, C., and Chaudhuri, S.
\newblock Online cascade learning for efficient inference over streams.
\newblock In \emph{Proceedings of the 41st International Conference on Machine Learning}, ICML'24. JMLR.org, 2024.

\bibitem[Ong et~al.(2025)Ong, Almahairi, Wu, Chiang, Wu, Gonzalez, Kadous, and Stoica]{ong_routellm_2024}
Ong, I., Almahairi, A., Wu, V., Chiang, W.-L., Wu, T., Gonzalez, J.~E., Kadous, M.~W., and Stoica, I.
\newblock Route{LLM}: Learning to route {LLM}s from preference data.
\newblock In \emph{The Thirteenth International Conference on Learning Representations}, 2025.
\newblock URL \url{https://openreview.net/forum?id=8sSqNntaMr}.

\bibitem[Ormazabal et~al.(2023)Ormazabal, Artetxe, and Agirre]{ormazabal_comblm_2023}
Ormazabal, A., Artetxe, M., and Agirre, E.
\newblock {C}omb{LM}: Adapting black-box language models through small fine-tuned models.
\newblock In Bouamor, H., Pino, J., and Bali, K. (eds.), \emph{Proceedings of the 2023 Conference on Empirical Methods in Natural Language Processing}, pp.\  2961--2974, Singapore, December 2023. Association for Computational Linguistics.
\newblock \doi{10.18653/v1/2023.emnlp-main.180}.
\newblock URL \url{https://aclanthology.org/2023.emnlp-main.180/}.

\bibitem[Pal et~al.(2022)Pal, Umapathi, and Sankarasubbu]{pal_medmcqa_2022}
Pal, A., Umapathi, L.~K., and Sankarasubbu, M.
\newblock Medmcqa: A large-scale multi-subject multi-choice dataset for medical domain question answering.
\newblock In Flores, G., Chen, G.~H., Pollard, T., Ho, J.~C., and Naumann, T. (eds.), \emph{Proceedings of the Conference on Health, Inference, and Learning}, volume 174 of \emph{Proceedings of Machine Learning Research}, pp.\  248--260. PMLR, 07--08 Apr 2022.
\newblock URL \url{https://proceedings.mlr.press/v174/pal22a.html}.

\bibitem[Parnami \& Lee(2022)Parnami and Lee]{parnami2022learningexamplessummaryapproaches}
Parnami, A. and Lee, M.
\newblock Learning from few examples: A summary of approaches to few-shot learning, 2022.
\newblock URL \url{https://arxiv.org/abs/2203.04291}.

\bibitem[Pham et~al.(2024)Pham, Nguyen, Le, Phung, Carneiro, and Do]{pham_frequency_2024}
Pham, C., Nguyen, V.-A., Le, T., Phung, D., Carneiro, G., and Do, T.-T.
\newblock Frequency attention for knowledge distillation.
\newblock In \emph{2024 IEEE/CVF Winter Conference on Applications of Computer Vision (WACV)}, pp.\  2266--2275, 2024.

\bibitem[Rayan \& Tewari(2025)Rayan and Tewari]{rayan_learning_2025}
Rayan, S. and Tewari, A.
\newblock Learning to {{Partially Defer}} for {{Sequences}}.
\newblock Technical Report 2502.01459, arXiv, February 2025.
\newblock URL \url{https://doi.org/10.48550/arXiv.2502.01459}.

\bibitem[Reimers \& Gurevych(2019)Reimers and Gurevych]{reimers-2019-sentence-bert}
Reimers, N. and Gurevych, I.
\newblock Sentence-{BERT}: Sentence embeddings using {S}iamese {BERT}-networks.
\newblock In Inui, K., Jiang, J., Ng, V., and Wan, X. (eds.), \emph{Proceedings of the 2019 Conference on Empirical Methods in Natural Language Processing and the 9th International Joint Conference on Natural Language Processing (EMNLP-IJCNLP)}, pp.\  3982--3992, November 2019.
\newblock URL \url{https://aclanthology.org/D19-1410/}.

\bibitem[Romero et~al.(2015)Romero, Ballas, Kahou, Chassang, Gatta, and Bengio]{romero_fitnets_2015}
Romero, A., Ballas, N., Kahou, S.~E., Chassang, A., Gatta, C., and Bengio, Y.
\newblock {{FitNets}}: {{Hints}} for {{Thin Deep Nets}}.
\newblock Technical Report 1412.6550, arXiv, March 2015.
\newblock URL \url{https://doi.org/10.48550/arXiv.1412.6550}.

\bibitem[Rubin et~al.(2022)Rubin, Herzig, and Berant]{rubin_learning_2022-2}
Rubin, O., Herzig, J., and Berant, J.
\newblock Learning to retrieve prompts for in-context learning.
\newblock In \emph{Proceedings of the 2022 Conference of the North American Chapter of the Association for Computational Linguistics: Human Language Technologies}, pp.\  2655--2671, July 2022.
\newblock URL \url{https://aclanthology.org/2022.naacl-main.191/}.

\bibitem[Shen et~al.(2024)Shen, Lang, Wang, Kim, and Sontag]{shen_learning_2024}
Shen, Z., Lang, H., Wang, B., Kim, Y., and Sontag, D.
\newblock Learning to decode collaboratively with multiple language models.
\newblock In \emph{Proceedings of the 62nd Annual Meeting of the Association for Computational Linguistics (Volume 1: Long Papers)}, pp.\  12974--12990, August 2024.
\newblock URL \url{https://aclanthology.org/2024.acl-long.701/}.

\bibitem[Shi et~al.(2024)Shi, Zi, Shi, Zhang, Wu, and Xu]{shi2024eragent}
Shi, Y., Zi, X., Shi, Z., Zhang, H., Wu, Q., and Xu, M.
\newblock Eragent: Enhancing retrieval-augmented language models with improved accuracy, efficiency, and personalization.
\newblock \emph{arXiv preprint arXiv:2405.06683}, 2024.

\bibitem[Shrestha et~al.(2024)Shrestha, Zou, Chen, Li, Xie, and Deng]{shrestha2024fairrag}
Shrestha, R., Zou, Y., Chen, Q., Li, Z., Xie, Y., and Deng, S.
\newblock Fairrag: Fair human generation via fair retrieval augmentation.
\newblock In \emph{Proceedings of the IEEE/CVF Conference on Computer Vision and Pattern Recognition}, pp.\  11996--12005, 2024.

\bibitem[Srivastava et~al.(2022)Srivastava, Rastogi, Rao, Shoeb, Abid, Fisch, Brown, Santoro, Gupta, Garriga-Alonso, Kluska, Lewkowycz, Agarwal, Power, Ray, Warstadt, Kocurek, Safaya, Tazarv, Xiang, Parrish, Nie, Hussain, Askell, Dsouza, Slone, Rahane, Iyer, Andreassen, Madotto, Santilli, Stuhlm{\"u}ller, Dai, La, Lampinen, Zou, Jiang, Chen, Vuong, Gupta, Gottardi, Norelli, Venkatesh, Gholamidavoodi, Tabassum, Menezes, Kirubarajan, Mullokandov, Sabharwal, Herrick, Efrat, Erdem, Karaka{\c{s}}, Roberts, Loe, Zoph, Bojanowski, {\"O}zyurt, Hedayatnia, Neyshabur, Inden, Stein, Ekmekci, Lin, Howald, Orinion, Diao, Dour, Stinson, Argueta, Ferri, Singh, Rathkopf, Meng, Baral, Wu, Callison-Burch, Waites, Voigt, Manning, Potts, Ramirez, Rivera, Siro, Raffel, Ashcraft, Garbacea, Sileo, Garrette, Hendrycks, Kilman, Roth, Freeman, Khashabi, Levy, Gonz{\'a}lez, Perszyk, Hernandez, Chen, Ippolito, Gilboa, Dohan, Drakard, Jurgens, Datta, Ganguli, Emelin, Kleyko, Yuret, Chen, Tam, Hupkes, Misra, Buzan, Mollo, Yang, Lee,
  Schrader, Shutova, Cubuk, Segal, Hagerman, Barnes, Donoway, Pavlick, Rodol{\`a}, Lam, Chu, Tang, Erdem, Chang, Chi, Dyer, Jerzak, Kim, Manyasi, Zheltonozhskii, Xia, Siar, Mart{\'\i}nez-Plumed, Happ{\'e}, Chollet, Rong, Mishra, Winata, de~Melo, Kruszewski, Parascandolo, Mariani, Wang, Jaimovitch-Lopez, Betz, Gur-Ari, Galijasevic, Kim, Rashkin, Hajishirzi, Mehta, Bogar, Shevlin, Schuetze, Yakura, Zhang, Wong, Ng, Noble, Jumelet, Geissinger, Kernion, Hilton, Lee, Fisac, Simon, Koppel, Zheng, Zou, Kocon, Thompson, Wingfield, Kaplan, Radom, Sohl-Dickstein, Phang, Wei, Yosinski, Novikova, Bosscher, Marsh, Kim, Taal, Engel, Alabi, Xu, Song, Tang, Waweru, Burden, Miller, Balis, Batchelder, Berant, Frohberg, Rozen, Hernandez-Orallo, Boudeman, Guerr, Jones, Tenenbaum, Rule, Chua, Kanclerz, Livescu, Krauth, Gopalakrishnan, Ignatyeva, Markert, Dhole, Gimpel, Omondi, Mathewson, Chiafullo, Shkaruta, Shridhar, McDonell, Richardson, Reynolds, Gao, Zhang, Dugan, Qin, Contreras-Ochando, Morency, Moschella, Lam, Noble,
  Schmidt, He, Oliveros-Col{\'o}n, Metz, Senel, Bosma, Sap, Hoeve, Farooqi, Faruqui, Mazeika, Baturan, Marelli, Maru, Ramirez-Quintana, Tolkiehn, Giulianelli, Lewis, Potthast, Leavitt, Hagen, Schubert, Baitemirova, Arnaud, McElrath, Yee, Cohen, Gu, Ivanitskiy, Starritt, Strube, Sw{\k{e}}drowski, Bevilacqua, Yasunaga, Kale, Cain, Xu, Suzgun, Walker, Tiwari, Bansal, Aminnaseri, Geva, Gheini, T, Peng, Chi, Lee, Krakover, Cameron, Roberts, Doiron, Martinez, Nangia, Deckers, Muennighoff, Keskar, Iyer, Constant, Fiedel, Wen, Zhang, Agha, Elbaghdadi, Levy, Evans, Casares, Doshi, Fung, Liang, Vicol, Alipoormolabashi, Liao, Liang, Chang, Eckersley, Htut, Hwang, Mi{\l}kowski, Patil, Pezeshkpour, Oli, Mei, Lyu, Chen, Banjade, Rudolph, Gabriel, Habacker, Risco, Milli{\`e}re, Garg, Barnes, Saurous, Arakawa, Raymaekers, Frank, Sikand, Novak, Sitelew, Bras, Liu, Jacobs, Zhang, Salakhutdinov, Chi, Lee, Stovall, Teehan, Yang, Singh, Mohammad, Anand, Dillavou, Shleifer, Wiseman, Gruetter, Bowman, Schoenholz, Han, Kwatra, Rous,
  Ghazarian, Ghosh, Casey, Bischoff, Gehrmann, Schuster, Sadeghi, Hamdan, Zhou, Srivastava, Shi, Singh, Asaadi, Gu, Pachchigar, Toshniwal, Upadhyay, Debnath, Shakeri, Thormeyer, Melzi, Reddy, Makini, Lee, Torene, Hatwar, Dehaene, Divic, Ermon, Biderman, Lin, Prasad, Piantadosi, Shieber, Misherghi, Kiritchenko, Mishra, Linzen, Schuster, Li, Yu, Ali, Hashimoto, Wu, Desbordes, Rothschild, Phan, Wang, Nkinyili, Schick, Kornev, Tunduny, Gerstenberg, Chang, Neeraj, Khot, Shultz, Shaham, Misra, Demberg, Nyamai, Raunak, Ramasesh, vinay~uday prabhu, Padmakumar, Srikumar, Fedus, Saunders, Zhang, Vossen, Ren, Tong, Zhao, Wu, Shen, Yaghoobzadeh, Lakretz, Song, Bahri, Choi, Yang, Hao, Chen, Belinkov, Hou, Hou, Bai, Seid, Zhao, Wang, Wang, Wang, and Wu]{srivastava_bigbench_2022}
Srivastava, A., Rastogi, A., Rao, A., Shoeb, A. A.~M., Abid, A., Fisch, A., Brown, A.~R., Santoro, A., Gupta, A., Garriga-Alonso, A., Kluska, A., Lewkowycz, A., Agarwal, A., Power, A., Ray, A., Warstadt, A., Kocurek, A.~W., Safaya, A., Tazarv, A., Xiang, A., Parrish, A., Nie, A., Hussain, A., Askell, A., Dsouza, A., Slone, A., Rahane, A., Iyer, A.~S., Andreassen, A.~J., Madotto, A., Santilli, A., Stuhlm{\"u}ller, A., Dai, A.~M., La, A., Lampinen, A.~K., Zou, A., Jiang, A., Chen, A., Vuong, A., Gupta, A., Gottardi, A., Norelli, A., Venkatesh, A., Gholamidavoodi, A., Tabassum, A., Menezes, A., Kirubarajan, A., Mullokandov, A., Sabharwal, A., Herrick, A., Efrat, A., Erdem, A., Karaka{\c{s}}, A., Roberts, B.~R., Loe, B.~S., Zoph, B., Bojanowski, B., {\"O}zyurt, B., Hedayatnia, B., Neyshabur, B., Inden, B., Stein, B., Ekmekci, B., Lin, B.~Y., Howald, B., Orinion, B., Diao, C., Dour, C., Stinson, C., Argueta, C., Ferri, C., Singh, C., Rathkopf, C., Meng, C., Baral, C., Wu, C., Callison-Burch, C., Waites, C., Voigt,
  C., Manning, C.~D., Potts, C., Ramirez, C., Rivera, C.~E., Siro, C., Raffel, C., Ashcraft, C., Garbacea, C., Sileo, D., Garrette, D., Hendrycks, D., Kilman, D., Roth, D., Freeman, C.~D., Khashabi, D., Levy, D., Gonz{\'a}lez, D.~M., Perszyk, D., Hernandez, D., Chen, D., Ippolito, D., Gilboa, D., Dohan, D., Drakard, D., Jurgens, D., Datta, D., Ganguli, D., Emelin, D., Kleyko, D., Yuret, D., Chen, D., Tam, D., Hupkes, D., Misra, D., Buzan, D., Mollo, D.~C., Yang, D., Lee, D.-H., Schrader, D., Shutova, E., Cubuk, E.~D., Segal, E., Hagerman, E., Barnes, E., Donoway, E., Pavlick, E., Rodol{\`a}, E., Lam, E., Chu, E., Tang, E., Erdem, E., Chang, E., Chi, E.~A., Dyer, E., Jerzak, E., Kim, E., Manyasi, E.~E., Zheltonozhskii, E., Xia, F., Siar, F., Mart{\'\i}nez-Plumed, F., Happ{\'e}, F., Chollet, F., Rong, F., Mishra, G., Winata, G.~I., de~Melo, G., Kruszewski, G., Parascandolo, G., Mariani, G., Wang, G.~X., Jaimovitch-Lopez, G., Betz, G., Gur-Ari, G., Galijasevic, H., Kim, H., Rashkin, H., Hajishirzi, H., Mehta,
  H., Bogar, H., Shevlin, H. F.~A., Schuetze, H., Yakura, H., Zhang, H., Wong, H.~M., Ng, I., Noble, I., Jumelet, J., Geissinger, J., Kernion, J., Hilton, J., Lee, J., Fisac, J.~F., Simon, J.~B., Koppel, J., Zheng, J., Zou, J., Kocon, J., Thompson, J., Wingfield, J., Kaplan, J., Radom, J., Sohl-Dickstein, J., Phang, J., Wei, J., Yosinski, J., Novikova, J., Bosscher, J., Marsh, J., Kim, J., Taal, J., Engel, J., Alabi, J., Xu, J., Song, J., Tang, J., Waweru, J., Burden, J., Miller, J., Balis, J.~U., Batchelder, J., Berant, J., Frohberg, J., Rozen, J., Hernandez-Orallo, J., Boudeman, J., Guerr, J., Jones, J., Tenenbaum, J.~B., Rule, J.~S., Chua, J., Kanclerz, K., Livescu, K., Krauth, K., Gopalakrishnan, K., Ignatyeva, K., Markert, K., Dhole, K., Gimpel, K., Omondi, K., Mathewson, K.~W., Chiafullo, K., Shkaruta, K., Shridhar, K., McDonell, K., Richardson, K., Reynolds, L., Gao, L., Zhang, L., Dugan, L., Qin, L., Contreras-Ochando, L., Morency, L.-P., Moschella, L., Lam, L., Noble, L., Schmidt, L., He, L.,
  Oliveros-Col{\'o}n, L., Metz, L., Senel, L.~K., Bosma, M., Sap, M., Hoeve, M.~T., Farooqi, M., Faruqui, M., Mazeika, M., Baturan, M., Marelli, M., Maru, M., Ramirez-Quintana, M.~J., Tolkiehn, M., Giulianelli, M., Lewis, M., Potthast, M., Leavitt, M.~L., Hagen, M., Schubert, M., Baitemirova, M.~O., Arnaud, M., McElrath, M., Yee, M.~A., Cohen, M., Gu, M., Ivanitskiy, M., Starritt, M., Strube, M., Sw{\k{e}}drowski, M., Bevilacqua, M., Yasunaga, M., Kale, M., Cain, M., Xu, M., Suzgun, M., Walker, M., Tiwari, M., Bansal, M., Aminnaseri, M., Geva, M., Gheini, M., T, M.~V., Peng, N., Chi, N.~A., Lee, N., Krakover, N. G.-A., Cameron, N., Roberts, N., Doiron, N., Martinez, N., Nangia, N., Deckers, N., Muennighoff, N., Keskar, N.~S., Iyer, N.~S., Constant, N., Fiedel, N., Wen, N., Zhang, O., Agha, O., Elbaghdadi, O., Levy, O., Evans, O., Casares, P. A.~M., Doshi, P., Fung, P., Liang, P.~P., Vicol, P., Alipoormolabashi, P., Liao, P., Liang, P., Chang, P.~W., Eckersley, P., Htut, P.~M., Hwang, P., Mi{\l}kowski, P.,
  Patil, P., Pezeshkpour, P., Oli, P., Mei, Q., Lyu, Q., Chen, Q., Banjade, R., Rudolph, R.~E., Gabriel, R., Habacker, R., Risco, R., Milli{\`e}re, R., Garg, R., Barnes, R., Saurous, R.~A., Arakawa, R., Raymaekers, R., Frank, R., Sikand, R., Novak, R., Sitelew, R., Bras, R.~L., Liu, R., Jacobs, R., Zhang, R., Salakhutdinov, R., Chi, R.~A., Lee, S.~R., Stovall, R., Teehan, R., Yang, R., Singh, S., Mohammad, S.~M., Anand, S., Dillavou, S., Shleifer, S., Wiseman, S., Gruetter, S., Bowman, S.~R., Schoenholz, S.~S., Han, S., Kwatra, S., Rous, S.~A., Ghazarian, S., Ghosh, S., Casey, S., Bischoff, S., Gehrmann, S., Schuster, S., Sadeghi, S., Hamdan, S., Zhou, S., Srivastava, S., Shi, S., Singh, S., Asaadi, S., Gu, S.~S., Pachchigar, S., Toshniwal, S., Upadhyay, S., Debnath, S.~S., Shakeri, S., Thormeyer, S., Melzi, S., Reddy, S., Makini, S.~P., Lee, S.-H., Torene, S., Hatwar, S., Dehaene, S., Divic, S., Ermon, S., Biderman, S., Lin, S., Prasad, S., Piantadosi, S., Shieber, S., Misherghi, S., Kiritchenko, S., Mishra,
  S., Linzen, T., Schuster, T., Li, T., Yu, T., Ali, T., Hashimoto, T., Wu, T.-L., Desbordes, T., Rothschild, T., Phan, T., Wang, T., Nkinyili, T., Schick, T., Kornev, T., Tunduny, T., Gerstenberg, T., Chang, T., Neeraj, T., Khot, T., Shultz, T., Shaham, U., Misra, V., Demberg, V., Nyamai, V., Raunak, V., Ramasesh, V.~V., vinay~uday prabhu, Padmakumar, V., Srikumar, V., Fedus, W., Saunders, W., Zhang, W., Vossen, W., Ren, X., Tong, X., Zhao, X., Wu, X., Shen, X., Yaghoobzadeh, Y., Lakretz, Y., Song, Y., Bahri, Y., Choi, Y., Yang, Y., Hao, S., Chen, Y., Belinkov, Y., Hou, Y., Hou, Y., Bai, Y., Seid, Z., Zhao, Z., Wang, Z., Wang, Z.~J., Wang, Z., and Wu, Z.
\newblock Beyond the imitation game: Quantifying and extrapolating the capabilities of language models.
\newblock \emph{Transactions on Machine Learning Research}, 2022.
\newblock ISSN 2835-8856.
\newblock URL \url{https://openreview.net/forum?id=uyTL5Bvosj}.
\newblock Featured Certification.

\bibitem[Strong et~al.(2025{\natexlab{a}})Strong, Men, and Noble]{strong_trustworthy_2025}
Strong, J., Men, Q., and Noble, J.~A.
\newblock Trustworthy and practical ai for healthcare: a guided deferral system with large language models.
\newblock In \emph{Proceedings of the Thirty-Ninth AAAI Conference on Artificial Intelligence and Thirty-Seventh Conference on Innovative Applications of Artificial Intelligence and Fifteenth Symposium on Educational Advances in Artificial Intelligence}, AAAI'25/IAAI'25/EAAI'25. AAAI Press, 2025{\natexlab{a}}.
\newblock ISBN 978-1-57735-897-8.
\newblock \doi{10.1609/aaai.v39i27.35063}.
\newblock URL \url{https://doi.org/10.1609/aaai.v39i27.35063}.

\bibitem[Strong et~al.(2025{\natexlab{b}})Strong, Saha, Ibrahim, Ouyang, and Noble]{strong_expert-agnostic_2025}
Strong, J., Saha, P., Ibrahim, Y., Ouyang, C., and Noble, A.
\newblock Expert-agnostic learning to defer, 2025{\natexlab{b}}.
\newblock URL \url{https://arxiv.org/abs/2502.10533}.

\bibitem[Suzgun et~al.(2022)Suzgun, Scales, Scharli, Gehrmann, Tay, Chung, Chowdhery, Le, Chi, Zhou, and Wei]{Suzgun2022ChallengingBT}
Suzgun, M., Scales, N., Scharli, N., Gehrmann, S., Tay, Y., Chung, H.~W., Chowdhery, A., Le, Q.~V., Chi, E.~H., Zhou, D., and Wei, J.
\newblock Challenging big-bench tasks and whether chain-of-thought can solve them.
\newblock In \emph{Annual Meeting of the Association for Computational Linguistics}, 2022.
\newblock URL \url{https://api.semanticscholar.org/CorpusID:252917648}.

\bibitem[Tailor et~al.(2024)Tailor, Patra, Verma, Manggala, and Nalisnick]{tailor_learning_2024}
Tailor, D., Patra, A., Verma, R., Manggala, P., and Nalisnick, E.
\newblock Learning to {{Defer}} to a {{Population}}: {{A Meta-Learning Approach}}.
\newblock In \emph{Proceedings of {{The}} 27th {{International Conference}} on {{Artificial Intelligence}} and {{Statistics}}}, pp.\  3475--3483. PMLR, April 2024.

\bibitem[Teerapittayanon et~al.(2016)Teerapittayanon, McDanel, and Kung]{teerapittayanon2016branchynet}
Teerapittayanon, S., McDanel, B., and Kung, H.-T.
\newblock Branchynet: {{Fast}} inference via early exiting from deep neural networks.
\newblock In \emph{2016 23rd International Conference on Pattern Recognition ({{ICPR}})}, pp.\  2464--2469. IEEE, 2016.

\bibitem[Thakur et~al.(2025)Thakur, Choudhary, Ramayapally, Vaidyanathan, and Hupkes]{thakur2025judgingjudgesevaluatingalignment}
Thakur, A.~S., Choudhary, K., Ramayapally, V.~S., Vaidyanathan, S., and Hupkes, D.
\newblock Judging the judges: Evaluating alignment and vulnerabilities in {LLM}s-as-judges.
\newblock In \emph{Proceedings of the Fourth Workshop on Generation, Evaluation and Metrics (GEM{\texttwosuperior})}, pp.\  404--430, July 2025.
\newblock URL \url{https://aclanthology.org/2025.gem-1.33/}.

\bibitem[Verma \& Nalisnick(2022)Verma and Nalisnick]{verma_calibrated_2022}
Verma, R. and Nalisnick, E.
\newblock Calibrated {{Learning}} to {{Defer}} with {{One-vs-All Classifiers}}.
\newblock In \emph{Proceedings of the 39th {{International Conference}} on {{Machine Learning}}}, pp.\  22184--22202. PMLR, June 2022.

\bibitem[Verma et~al.(2023)Verma, Barrejon, and Nalisnick]{verma_learning_2023}
Verma, R., Barrejon, D., and Nalisnick, E.
\newblock Learning to {{Defer}} to {{Multiple Experts}}: {{Consistent Surrogate Losses}}, {{Confidence Calibration}}, and {{Conformal Ensembles}}.
\newblock In \emph{Proceedings of {{The}} 26th {{International Conference}} on {{Artificial Intelligence}} and {{Statistics}}}, pp.\  11415--11434. PMLR, April 2023.

\bibitem[Wang et~al.(2025{\natexlab{a}})Wang, Yan, Zhang, and Lin]{wang2025elicitllmaugmentationexternal}
Wang, F., Yan, J., Zhang, Y., and Lin, T.
\newblock {ELICIT}: {LLM} augmentation via external in-context capability.
\newblock In \emph{The Thirteenth International Conference on Learning Representations}, 2025{\natexlab{a}}.
\newblock URL \url{https://openreview.net/forum?id=CI4sCBMXjP}.

\bibitem[Wang et~al.(2024{\natexlab{a}})Wang, Zhang, Li, Kong, Zhuang, Chen, and Zhang]{wang2024tpdenhancingstudentlanguage}
Wang, H., Zhang, R., Li, Y., Kong, L., Zhuang, Y., Chen, X., and Zhang, C.
\newblock {TPD}: Enhancing student language model reasoning via principle discovery and guidance.
\newblock In \emph{First Conference on Language Modeling}, 2024{\natexlab{a}}.
\newblock URL \url{https://openreview.net/forum?id=sJvhwDtFhQ}.

\bibitem[Wang et~al.(2024{\natexlab{b}})Wang, Li, Chen, Cai, Zhu, Lin, Cao, Kong, Liu, Liu, and Sui]{wang2023largelanguagemodelsfair}
Wang, P., Li, L., Chen, L., Cai, Z., Zhu, D., Lin, B., Cao, Y., Kong, L., Liu, Q., Liu, T., and Sui, Z.
\newblock Large language models are not fair evaluators.
\newblock In \emph{Proceedings of the 62nd Annual Meeting of the Association for Computational Linguistics}, pp.\  9440--9450, August 2024{\natexlab{b}}.
\newblock URL \url{https://aclanthology.org/2024.acl-long.511/}.

\bibitem[Wang et~al.(2025{\natexlab{b}})Wang, Zhou, Qiu, Chang, Bragg, and Zhang]{wang2025social}
Wang, R., Zhou, X., Qiu, L., Chang, J.~C., Bragg, J., and Zhang, A.~X.
\newblock Social-rag: Retrieving from group interactions to socially ground ai generation.
\newblock In \emph{Proceedings of the 2025 CHI Conference on Human Factors in Computing Systems}, pp.\  1--25, 2025{\natexlab{b}}.

\bibitem[Wang et~al.(2024{\natexlab{c}})Wang, Teo, Ouyang, Xu, and Shi]{wang2024m}
Wang, Z., Teo, S., Ouyang, J., Xu, Y., and Shi, W.
\newblock M-rag: Reinforcing large language model performance through retrieval-augmented generation with multiple partitions.
\newblock In \emph{Proceedings of the 62nd Annual Meeting of the Association for Computational Linguistics (Volume 1: Long Papers)}, pp.\  1966--1978, 2024{\natexlab{c}}.

\bibitem[Wu \& Sarwate(2024)Wu and Sarwate]{wu2024learninghelptrainingmodels}
Wu, Y. and Sarwate, A.
\newblock Learning to help: {{Training}} models to assist legacy devices.
\newblock Technical Report 2409.16253, arXiv, 2024.

\bibitem[Wu et~al.(2025)Wu, Li, Dong, Sathyavageeswaran, and Sarwate]{wu_learning_2025}
Wu, Y., Li, Y., Dong, Z., Sathyavageeswaran, N., and Sarwate, A.~D.
\newblock Learning to help in multi-class settings.
\newblock In \emph{The Thirteenth International Conference on Learning Representations}, 2025.
\newblock URL \url{https://openreview.net/forum?id=NCgTbt2j1F}.

\bibitem[Xia et~al.(2024)Xia, Kong, Yu, Guo, Rossi, Kim, and Li]{xia_which_2024}
Xia, Y., Kong, F., Yu, T., Guo, L., Rossi, R.~A., Kim, S., and Li, S.
\newblock Which llm to play? convergence-aware online model selection with time-increasing bandits.
\newblock In \emph{Proceedings of the ACM Web Conference 2024}, WWW '24, pp.\  4059–4070, 2024.
\newblock URL \url{https://doi.org/10.1145/3589334.3645420}.

\bibitem[Xiong et~al.(2024)Xiong, Hu, Lu, LI, Fu, He, and Hooi]{xiong_can_2024}
Xiong, M., Hu, Z., Lu, X., LI, Y., Fu, J., He, J., and Hooi, B.
\newblock Can {LLM}s express their uncertainty? an empirical evaluation of confidence elicitation in {LLM}s.
\newblock In \emph{The Twelfth International Conference on Learning Representations}, 2024.
\newblock URL \url{https://openreview.net/forum?id=gjeQKFxFpZ}.

\bibitem[Xu et~al.(2025)Xu, Wang, Wang, Ye, Du, Ma, and Tian]{xu2025recon}
Xu, Z., Wang, M., Wang, Y., Ye, W., Du, Y., Ma, Y., and Tian, Y.
\newblock Recon: Reasoning with condensation for efficient retrieval-augmented generation.
\newblock \emph{arXiv preprint arXiv:2510.10448}, 2025.

\bibitem[Yang et~al.(2024)Yang, Rao, Chen, Guo, Zhang, Yang, and Zhang]{yang2024rag}
Yang, D., Rao, J., Chen, K., Guo, X., Zhang, Y., Yang, J., and Zhang, Y.
\newblock Im-rag: Multi-round retrieval-augmented generation through learning inner monologues.
\newblock In \emph{Proceedings of the 47th International ACM SIGIR Conference on Research and Development in Information Retrieval}, pp.\  730--740, 2024.

\bibitem[Yu et~al.(2024)Yu, Jiang, Shi, YU, Liu, Zhang, Kwok, Li, Weller, and Liu]{yu_metamath_2024}
Yu, L., Jiang, W., Shi, H., YU, J., Liu, Z., Zhang, Y., Kwok, J., Li, Z., Weller, A., and Liu, W.
\newblock Metamath: Bootstrap your own mathematical questions for large language models.
\newblock In \emph{The Twelfth International Conference on Learning Representations}, 2024.
\newblock URL \url{https://openreview.net/forum?id=N8N0hgNDRt}.

\bibitem[Zellinger \& Thomson(2025)Zellinger and Thomson]{zellinger_rational_2025}
Zellinger, M.~J. and Thomson, M.
\newblock Rational tuning of {LLM} cascades via probabilistic modeling.
\newblock \emph{Transactions on Machine Learning Research}, 2025.
\newblock URL \url{https://openreview.net/forum?id=YCBVcGSZeR}.

\bibitem[Zellinger et~al.(2025)Zellinger, Liu, and Thomson]{zellinger_cost-saving_2025}
Zellinger, M.~J., Liu, R., and Thomson, M.
\newblock Cost-{{Saving LLM Cascades}} with {{Early Abstention}}.
\newblock Technical Report 2502.09054, arXiv, February 2025.
\newblock URL \url{https://doi.org/10.48550/arXiv.2502.09054}.

\bibitem[Zhang et~al.(2025)Zhang, Zhu, Ming, Jin, Chai, Yang, Tian, Fan, and Chen]{zhang2025dh}
Zhang, F., Zhu, D., Ming, J., Jin, Y., Chai, D., Yang, L., Tian, H., Fan, Z., and Chen, K.
\newblock Dh-rag: A dynamic historical context-powered retrieval-augmented generation method for multi-turn dialogue.
\newblock \emph{arXiv preprint arXiv:2502.13847}, 2025.

\bibitem[Zheng et~al.(2025)Zheng, Guha, Arifov, Zhang, Skreta, Manning, Henderson, and Ho]{zhong_bar_exam_qa_2024}
Zheng, L., Guha, N., Arifov, J., Zhang, S., Skreta, M., Manning, C.~D., Henderson, P., and Ho, D.~E.
\newblock A reasoning-focused legal retrieval benchmark.
\newblock In \emph{Proceedings of the 2025 Symposium on Computer Science and Law}, CSLAW '25, pp.\  169–193, New York, NY, USA, 2025. Association for Computing Machinery.
\newblock ISBN 9798400714214.
\newblock \doi{10.1145/3709025.3712219}.
\newblock URL \url{https://doi.org/10.1145/3709025.3712219}.

\bibitem[Zhou et~al.(2025)Zhou, Huang, and Liao]{zhou2025debatereflectdistillmultiagent}
Zhou, X., Huang, H., and Liao, L.
\newblock Debate, reflect, and distill: Multi-agent feedback with tree-structured preference optimization for efficient language model enhancement.
\newblock In \emph{Findings of the Association for Computational Linguistics: ACL 2025}, pp.\  9122--9137, July 2025.
\newblock URL \url{https://aclanthology.org/2025.findings-acl.475/}.

\end{thebibliography}
\bibliographystyle{icml2026}

\newpage
\appendix
\onecolumn

\section{Extended Related Work}\label{sec: related work}
\textbf{LLM Cascade} There are many LLM paradigms that contain collaboration between multiples LLMs in a system~\citep{chen2025harnessing}: a)Ensemble before inference, where router choose one LLM from candidates for inference; b)Ensemble during inference, where LLMs work in parallel; c) Ensemble after inference, where LLMs work in sequence and LLM Cascade belongs to this filed. LLM Cascade is firstly proposed by~\citet{chen_frugalgpt_2023} to balance the LLM performance and cost by allocating queries to a weak model or a strong model according to the confidence estimate of the queried question. ~\citet{shen_learning_2024} propose a latent variable model to let the weak model learn the deferral function at the token-level. ~\citet{rayan_learning_2025} also extend the Learning to Defer~\citep{madras_predict_2018} setting to LLM by training a post-hoc deferral function for each token of the sequence. ~\citet{ong_routellm_2024} train a separate router such that deferral decision can be made before sending the query to weak LLM, saving more tokens. ~\citet{zellinger_cost-saving_2025} provide extra option to early discard the unsolvable queries in weak model. ~\citet{xia_which_2024,nie_online_2024} formulate LLM Cascade as online problem to dynamically adjust its deferral policy over time.~\citet{zellinger_rational_2025} propose a rational tuning pipeline for LLM Cascade via probabilistic modeling. Since the deferral result relies on the confidence score of weak model, there are are literatures focusing on boosting the the measure of confidence of weak model's output~\citep{jitkrittum_when_2023, chuang_learning_2025}. Together with experimental verification, ~\citet{jung2024trustescalatellmjudges} conduct fixed sequence testing to provably guarantee the lower bound of accuracy. Therefore, we choose Cascaded Selective Evaluation by~\citet{jung2024trustescalatellmjudges} as the baseline of our work. Beside deferring to strong model,
Beyond standard LLM Cascade, ~\citet{strong_trustworthy_2025} propose a deferral system that weak model also sends its generated intelligent guidance to strong model once deferred, boosting the performance of next level model. However, current LLM Cascades cannot adapt to the query streaming once trained and deployed. And the weak model cannot learn from the previous deferrals and corresponding strategies generated by the strong model, causing the waste of computation, tokens, money and sometimes communication.

\textbf{Learning With Reject Option} The general framework that allows a machine learning model to abstain from making decision was originally propose by~\citet{chow_optimum_1957,chow1970optimum} in the 1950s. After decades, the Learning with reject option was continuously explored in different periods by~\citet{herbei_classification_2006} and~\citet{ortner_learning_2016}. The more recent works extend the framework to a multi models system where the local model can learn to defer its task to one expert (human or existing model) ~\citep{madras_predict_2018,mozannar_consistent_2021,verma_calibrated_2022,mao_theoretically_2024}, multiple experts~\citep{verma_learning_2023,barneva_principled_2024} or unknown experts~\citep{nguyen_probabilistic_2025,strong_expert-agnostic_2025,tailor_learning_2024}. There are literature that also explore the case when expert can learn to adaptively help the local model~\citep{wu2024learninghelptrainingmodels,wu_learning_2025}. Adding reject option at the network layer level is another branch of works called early exiting~\citep{teerapittayanon2016branchynet}. However, most of the learning with reject option works focus on classical prediction tasks, few of them address the NLP tasks that rely on generative-based model while this work focus on the collaboration between LLMs. 

\textbf{Knowledge Distillation} Knowledge distillation (KD) is a machine learning technique for training smaller "student" models by transferring "knowledge" from larger, more powerful "teacher" models. Classical knowledge distillation use soft-labels~\citep{hinton_distilling_2015} to let the student model learn the distribution of teacher model. The concept of KD is expanded to more levels: besides mimicking the output of teacher model, the student model can also learn from intermediate features~\citep{romero_fitnets_2015,pham_frequency_2024}, relationships~\citep{joshi_representation_2024}, actively chosen sample~\citep{liu2024evolvingknowledgedistillationlarge}, principle discovery~\citep{wang2024tpdenhancingstudentlanguage} and itself ~\citep{lee_self-knowledge_2022}. Our Inter-Cascade also helps the knowledge transfer from the Strong LLM to Weak LLM. However, current knowledge distillation relies on the training or finetuning of the student model and can not continue learning process during inference phase while our method doesn't require the updating of the LLM parameters and continually improves during the inference phase via dynamically matching stored Strong LLM's strategy.

\textbf{Retrieval-Augmented Generation(RAG)}RAG~\citep{lewis_retrieval-augmented_2021} is an approach that combines pre-trained parametric and non-parametric memory for language generation. Given the focus of our work, we group RAG-style approaches into three categories: static RAG, history-aware RAG, and agentic RAG.

\textit{Static RAG.} Classical RAG assumes a fixed, pre-constructed external corpus and focuses on how to retrieve, re-rank, and fuse evidence to support generation. Works in this line focus on design dense retrieval and re-ranking pipelines over a static collection~\citep{lewis_retrieval-augmented_2021,edge_local_2025, wang2025elicitllmaugmentationexternal,rubin_learning_2022-2,margatina_active_2023-1}.In all these methods, the source of knowledge is an offline, human-curated dataset, and the system’s adaptivity lies purely in how it accesses this corpus, not what the corpus contains.
By contrast, Inter-Cascade does not assume any pre-existing database: the “corpus” is constructed online as the strong LLM generates strategies and reasoning traces that are stored for future reuse by the weak LLM. Thus, our system is closer to an online, LLM-driven knowledge construction mechanism than to classical static RAG.

\textit{History-Aware RAG.} A second line of work augments RAG with dialogue history and user feedback, dynamically updating a memory store based on past interactions. Conversational RAG frameworks like DH-RAG~\citep{zhang2025dh}, CHIQ~\citep{mo2025conversational} maintain short-term and long-term memories of successful dialogue turns, using them to improve future retrieval and personalization. Other methods such as ComRAG~\citep{chen2025comrag}, ERAGent~\citep{shi2024eragent}, Pistis-RAG~\citep{bai2024pistis}, and Social-RAG~\citep{wang2025social} update user profiles or QA memories when users provide explicit positive feedback or when high-quality answers are validated by the social community. Despite their dynamism, these systems either take history information for self usage or treat the human user (or user community) as the source of new content. The resulting models are primarily personalized assistants. In Inter-Cascade, the update loop is fundamentally different: the weak LLM decides when to update, and the strong LLM decides what to write, without any human in the loop. The stored content is not user utterances or QA pairs, but LLM-generated strategies and reasoning structures distilled from a stronger model. Rather than personalizing to a single user, Inter-Cascade uses interaction between two models to build a reusable strategic knowledge base for many users and tasks.

\textit{Agentic RAG} A third, increasingly prominent direction combines RAG with multi-agent or agentic architectures~\citep{li2025towards}. In these systems, different agents are assigned distinct roles, e.g., planner, retriever, answer generator, or verifier.
Those agents collaborate via tool calls and message passing.
For centralized systems like MA-RAG~\citep{nguyen2025ma}, HM-RAG~\citep{liu2025hm}, and SurgRaw~\citep{low2025surgraw}, the focus is on managing the workflow, such as deciding when to use the retriever to access the existing database.
Decentralized methods like M-RAG~\citep{wang2024m} and MDocAgent~\citep{han2025mdocagent} consider retrieval from partitioned databases. There are also works like RECOND (end-to-end generation)~\citep{xu2025recon}Hippo (knowledge-graph)~\citep{gutierrez2025rag}, IM-RAG (multi step refinement)~\citep{yang2024rag} and FAIR-RAG (fair retrieval)~\citep{shrestha2024fairrag} propose algorithms to refine answers from RAG database. 
However, in all such designs, the RAG component itself remains an external, fixed resource: agents coordinate how to use RAG, but no agent is responsible for constructing a new corpus of knowledge for others. Inter-Cascade differs from these agentic RAG systems in two key aspects. First, there are only two “agents”: a weak LLM and a strong LLM, but their interaction is explicitly teacher–student and online knowledge distillation, rather than mere division of labor. Second, the strong LLM actively produces the knowledge store that the weak LLM later retrieves, making the RAG-like database a product of model interaction rather than a static tool.

Across all three categories, existing RAG approaches either (i) operate over a fixed, human-curated external corpus, (ii) update a memory store using human dialogue and feedback, or (iii) update a memory using self history for personalization without knowledge transfer. To our knowledge, Inter-Cascade is the first framework where a weak LLM and a strong LLM jointly and autonomously build a RAG-like corpus under the framework of LLM Cascade, with the weak model deciding when to consult and update it, and the strong model providing the organized knowledge. This yields a new form of online, interaction-driven distillation, particularly suitable for small models without access to large external knowledge bases or the Internet.

\textbf{Other related topics} There are also a weak model and strong model in \textit{Speculative decoding}~\citep{leviathan_fast_2023,narasimhan_faster_2024}. In speculative decoding, the weak model works as a answer draft while the strong model work as a verifier to speed up the generation compared to only using strong model. However, in Inter-Cascade, Strong LLM is called only when the Weak LLM is unable to handle current query. \textit{CombLM}~\citep{ormazabal_comblm_2023} and \textit{LLM Debate}~\citep{irving_ai_2018,du_improving_2023,estornell_multi-llm_2024,khan_debating_2024,zhou2025debatereflectdistillmultiagent} are other branches of works that also involve interaction between LLMs. CombLM integrates the logit distribution of two LLMs while LLM Debate requires different LLMs to argue and refine their initial answers and eventually reach consensus through multiple rounds of interaction. The key difference between Inter-Cascade and them is that Inter-Cascade let the Strong LLM and Weak LLM work in a sequential order can conduct early stop to save tokens.

\begin{algorithm}[H]
\caption{Inter-Cascade Inference Pipeline}
\label{alg:general-intercascade}
\begin{algorithmic}[1]
\Require Test set $\mathcal{T}=\{q_1,\dots,q_I\}\subseteq\mathcal{Q}$; LLM $M_n$ with deferral function $d_n$, generation function $g_n$, strategy repository $\Repo_n$ and strategy generator $h_n$.
\Statex \textbf{Deferral convention:} $0=\text{handle locally}$, $1=\text{defer/forward}$.
\State $\Repo=\emptyset$
\For{$n \gets 1$ \textbf{to} $N$}
    \For{{$i \gets 1$ \textbf{to} $I$}}
        \If{$n<N$}
            \State \textbf{(Strategy matching)}
            \State $[s^{t_1}_i, s^{t_2}_i, ..., s^{t_k}_i] \gets \Match_n(q_i, \Repo_n)$ \Comment{Find most relevant top-$k$ strategies to $q_i$}
            \State $q^{\prime}_i \gets [q_i, s^{t_1}_i, s^{t_2}_i, ..., s^{t_k}_i] $ \Comment{concatenate query and strategies}
        \Else
            \State $q'_i=q_i$
            \Comment{Last LLM doesn't maintain $\Repo$}
        \EndIf

        \Statex
        \State \textbf{(Deferral Decision)}
        \If{$d_n(q^{\prime}_i) = 0$}
            \State generate answer $a_i \gets \GenW(q^{\prime}_i)$ \Comment{Answer locally at Weaker LLM}
            \State $s_{\text{new}} \gets \StratGen(q_i)$
            \State $\Repo_{<n} \gets \Repo_{<n} \cup \{s_{\text{new}}\}$
            \Comment{Add strategy to all the weaker LLMs}
        \Else
            \If {$n<N$}
            \State Pass
            \Comment{Defer to next level}
            \Else
            \State Discard current query $q_i$ 
            \Comment{None of LLMs are confident to answer the query}
            \EndIf
        \EndIf
    \EndFor
\EndFor
\end{algorithmic}
\end{algorithm}

\section{Order of LLMs}\label{appendix：Order of LLM}
To distinguish two LLMs into strong model \( M_s \) and weak model \( M_w \), we make following definitions.  
For a task distribution $\mathcal{D}$, we denote the performance of a model $M$ by $\operatorname{Perf}(M)$, which can be instantiated by measures such as the expected accuracy or negative loss on $\mathcal{D}$. 
Similarly, we let $\operatorname{Cost}(M)$ represent the expected cost of using $M$ on $\mathcal{D}$, such as the price, latency, or required computation resource. Note that \( \mathrm{Cost} \) also depends on the task distribution \( \mathcal{D} \), for simplicity, we only use the notation \( \mathrm{Cost}(M) \).
We say that $M_w$ is weaker than $M_s$ if $\operatorname{Perf}(m_w) \leq \operatorname{Perf}(m_s)$, and that it is cheaper if $\operatorname{Cost}(m_w) \leq \operatorname{Cost}(m_s)$. To simplify notation, we introduce the shorthand relation
\begin{align}
    M_w \preccurlyeq_{\text{wbc}} M_s
\end{align}

if and only if
\begin{align}
    \operatorname{Perf}(M_w) \le \operatorname{Perf}(M_s)
\quad \text{and} \quad
\operatorname{Cost}(M_w) \leq \operatorname{Cost}(M_s),
\end{align}

where the term ``wbc'' represents “weaker but cheaper”.
Consider a multi-LLM inference/generation system, which contains $N$ LLM
models, $\mathcal{M}= \{ M_1, M_2, ..., M_N \}$, with different capacities and
use costs to a query. WLOG, we assume that $M_1 \preccurlyeq_{\text{wbc}} M_2\preccurlyeq_{\text{wbc}} ...\preccurlyeq_{\text{wbc}} M_N$.

\section{Algorithm for General Inter-Cascade}\label{app:general-inter}
Since Inter-Cascade is scalable to any number of layers for LLM, the general Inter-Cascade pipeline for $N$-LLM cascade system is shown in Algo.~\ref{alg:general-intercascade}.


\section{Proof: Clopper-Pearson Upper bound as a Beta quantile}\label{proof: CP upper bound}

In the lemma below, we apply the Clopper-Pearson upper bound to rewrite $R^{+}(\lambda)$, yielding a clearer form that facilitates computation. This helps the proof of Theorem~\ref{thm: decrease} and Theorem~\ref{thm: decrease-c}.

\begin{lemma}[Clopper--Pearson upper bound as a Beta quantile]\label{lemma: upper bound quantile}
Let $n(\lambda)\in\mathbb{N}$ be the number of evaluated items at threshold $\lambda$, 
let $R(\lambda)\in[0,1]$ denote the unknown risk, and suppose
\[
X \sim \mathrm{Bin}\big(n(\lambda),\,R(\lambda)\big),
\]
and $x \in \{0,1,\dots,n(\lambda)\}$ is the number of error observed. Write $\widehat R(\lambda)=x/n(\lambda)$. For a fixed $\delta\in(0,1)$, define the one-sided $(1-\delta)$
upper confidence limit by
\[
\widehat R^{+}(\lambda)
:=\sup\Big\{\,p\in[0,1]:\ \Pr_{p}\!\big(\mathrm{Bin}(n(\lambda),p)\le x\big)\ge \delta\,\Big\}.
\]
Then
\[
\ \widehat R^{+}(\lambda)
= \Betainv\!\big(1-\delta;\ x+1,\ n(\lambda)-x\big)
\]
with the usual edge conventions
$\Betainv(1-\delta;1,n)=1-\delta^{1/n}$ when $x=0$ and $\widehat R^{+}(\lambda)=1$ when $x=n(\lambda)$.
\end{lemma}
\begin{proof}
For fixed $x<n(\lambda)$ the map $p\mapsto F(p):=\Pr\big(\mathrm{Bin}(n(\lambda),p)\le x\big)$
is strictly decreasing in $p$, so the set in the definition of $\widehat R^{+}(\lambda)$
is an interval $[0,p^\star]$ and the supremum $p^\star$ uniquely solves
\begin{equation}\label{eq:crit}
F(p^\star)=P\big(\mathrm{Bin}(n(\lambda),p^\star)\le x\big)=\delta.
\end{equation}
Using the standard identity linking the binomial tail to the regularized incomplete beta function,
for integers $0\le x\le n(\lambda)-1$,
\[
P(X\le x)
= \sum_{k=0}^{x}\binom{n(\lambda)}{k}p^{k}(1-p)^{n(\lambda)-k}
= 1 - I_{p}\!\big(x+1,\ n(\lambda)-x\big),
\]
where $I_{p}(a,b)$ is the CDF of $\mathrm{Beta}(a,b)$ at $p$.
Plugging this into \eqref{eq:crit} gives
\[
I_{p^\star}\!\big(x+1,\ n(\lambda)-x\big)=1-\delta,
\]
so $p^\star$ is the $(1-\delta)$ quantile of the $\mathrm{Beta}\big(x+1,\ n(\lambda)-x\big)$ distribution:
\[
p^\star=\Betainv\!\big(1-\delta;\ x+1,\ n(\lambda)-x\big).
\]
This equals $\widehat R^{+}(\lambda)$ by definition. The stated edge cases follow from
$F(p)=(1-p)^{n(\lambda)}$ when $x=0$ and from monotonicity when $x=n(\lambda)$.
\end{proof}

\section{Proof: Unchanged Threshold}\label{proof: change of confident queries}

\begin{theorem}
Suppose that $\widehat R^{+}(\lambda)$ is a monotonic decreasing function of $\lambda$. Fix $\delta\in(0,1)$ and an integer $n\ge 1$. For $x\in\{0,1,\dots,n\}$, $\epsilon\in(0,1]$, and $b\in[1,\infty)$. 
Suppose that $\min\{\epsilon x+1,\,n-\epsilon x\}$ is moderately large and $1-\delta$ is not an extreme tail,  then:

\smallskip
\noindent\textbf{(a) Decrease in value.}
$\alpha(\epsilon,b)\leq \alpha(1,1)$ when $\epsilon\in(0,1]$ and $b\in[1,\infty)$.

\smallskip
\noindent\textbf{(b) Normal approximation for the amount of decrease.}
Let $z:=\Phiinv(1-\delta)$, where $\Phi$ is the Normal cumulative distribution function, when $n$ is large enough, the decrease of the risk under same level of tolerance is given by,
\begin{align*}
\alpha(1,1) - \alpha(\epsilon,b)
\approx
\left(\frac{x+1}{n+1}-\frac{\epsilon x+1}{b n+1}\right)
+ z\!\left[
\sqrt{\frac{(x+1)(n-x)}{(n+1)^2(n+2)}}
-
\sqrt{\frac{(\epsilon x+1)(b n-\epsilon x)}{(b n+1)^2(b n+2)}}
\right].
\end{align*}

\end{theorem}

\begin{proof}
We use a Beta function to represent the variable $\widehat R^{+}(\lambda)$, which is equivalent to the risk $\alpha$, when $\widehat R^{+}(\lambda)$ is a monotonic decreasing function of $\lambda$. We then use the approximation to Beta function to evaluate the decrease of $\alpha$ by definition. For the convenience of statement of our theories, we define that $\alpha(\epsilon,b)$ as the the value of risk bound $\alpha$ when the obtained $\lambda$ satisfies $n(\lambda)=bn$ and incorrectly answered queries among $n(\lambda)$ is $x(\lambda) = \epsilon x$, given the $\delta$ fixed.
(a) Notice that we assume that $\widehat R^{+}(\lambda)$ is a monotonic decreasing function of $\lambda$. Let us suppose that $\lambda_0$ satisfies that $n(\lambda_0)=bn$ and $x(\lambda_0)=\epsilon x$. By Algorithm~\ref{alg:calibration}, this shows that $\widehat R^{+}(\lambda_0) = \alpha(\epsilon,b)$. 

From Lemma~\ref{lemma: upper bound quantile}, we know that 
\[
\alpha(\epsilon,b)\;:=\;\Betainv\!\big(1-\delta;\,\epsilon x+1,\,bn-\epsilon x\big).
\]
Let \(p_1=\Betainv(1-\delta;\,x+1,\,n-x)\).
Then, by the property of Beta distribution,
\(P\!\big(\mathrm{Bin}(n,p_1)\le x\big)=\delta\).
It follows that,
\[
P\!\big(\mathrm{Bin}(bn,p_1)\le \epsilon x\big)
\;\le\;
P\!\big(\mathrm{Bin}(n,p_1)\le x\big)
=\delta,
\]
because lowering the threshold (\(\epsilon x \leq bx\)) and increasing trials (\(bn \geq n\))
makes the left tail event rarer. Let us assume that  \(p_2=\Betainv(1-\delta;\,\epsilon x+1,\,bn - \epsilon x)\). From the proof of Lemma~\ref{lemma: upper bound quantile}, it is equivalent to that $P\!\big(\mathrm{Bin}(bn,p_2)\le \epsilon x\big) = \delta$. It follows that $P\!\big(\mathrm{Bin}(bn,p_2)\le \epsilon x\big) = \delta \geq P\!\big(\mathrm{Bin}(bn,p_1)\le \epsilon x\big)$, which implies that $p_2 \leq p_1$. Hence the new upper bound \(p_2=\Betainv(1-\delta;\,x'+1,\,n'-x')\) satisfies \(p_2\le p_1\). This shows the statement (a).

(b) Write
\[
\mu_{\epsilon,b}:=\frac{\epsilon x+1}{bn+1},\qquad
\sigma_{\epsilon,b}:=\sqrt{\frac{(\epsilon x+1)(bn-\epsilon x)}{(bn+1)^2(bn+2)}}.
\]
In the large–sample, interior regime, e.g., $\min\{\epsilon x+1,\,n-\epsilon x\}\gg1$ and $x/n$ bounded away from $0$ and $1$,
\[
\Betainv\!\big(1-\delta;\,\epsilon x+1,\,bn-\epsilon x\big) \;=\;\mu_{\epsilon,b}\;+\;z\,\sigma_{\epsilon,b}\;+\;O\!\left(\frac{1}{n}\right).
\]
This is by the approximation to Beta distribution by normal distribution. 
Calculate $\alpha(1,1) -\alpha(\epsilon,b)$ demonstrate the result of theorem.
\end{proof}

\section{Proof: Unchanged Used-Queries}\label{proof: change of threshold}

Other than the case that the threshold remains unchanged, which is analyzed above, another case may be that when the user want the same number of queries to be covered by the Weak LLM during two rounds of queries (before and after adding strategies), one of which has a better Weak LLM. Such a case controls the cost. This case considers the influence of a better Weak LLM to our pipeline. In this case, we instead assume that $n(\lambda) = n(\lambda')$, and abbreviate them as $n$ for simplicity, which ensures the same coverage of Weak LLM. The number of wrongly answered queries before and after getting a better Weak LLM are denoted by $x$ and $\epsilon x$, and we still estimate the decrease of $\alpha$ under the same level of tolerance $\delta$. We give an approximation on the change rate of the risk bound with respect to the proportion of decrease of errors. We denote by $\alpha(\epsilon)$ the $\alpha(\epsilon,b=1)$ for simplicity, and present the analysis in Theorem~\ref{thm: decrease-c}.
\begin{theorem}\label{thm: decrease-c}
Suppose that $\widehat R^{+}(\lambda)$ is a monotonic decreasing function of $\lambda$. Fix $\delta\in(0,1)$ and an integer $n\ge 1$. For $x\in\{0,1,\dots,n\}$ and $\epsilon\in(0,1]$. Suppose that $\min\{\epsilon x+1,\,n-\epsilon x\}$ is moderately large and $1-\delta$ is not an extreme tail,
then:

\smallskip
\noindent\textbf{(a) Exact monotonicity.}
$\alpha(\epsilon)$ is strictly increasing in $\epsilon$. In particular, for any $\epsilon\in(0,1)$,
\[
\alpha(\epsilon)\;<\;\alpha(1).
\]

\smallskip
\noindent\textbf{(b) Normal approximation for the amount of decrease.}
Let $z:=\Phiinv(1-\delta)$, for $\epsilon$ near $1$,
\begin{align*}
\alpha(1)-\alpha(\epsilon)
&\approx (1-\epsilon)\,\Bigg[
\frac{x}{\,n+1\,}
+\frac{z}{2(n+1)\sqrt{n+2}}\,
\frac{x( n - 1 - 2 x)}{\sqrt{(x+1)(n-x)}}
\Bigg].
\end{align*}
Hence the decrease is approximately linear in $(1-\epsilon)$ with the coefficient in brackets; in particular,
when $x\le n/2$ the variance term is nonnegative and the decrease is at least
$(1-\epsilon)\,x/(n+1)$ to first order.
\end{theorem}

\begin{proof}
(a) Similar to the proof of the statement (a) of Theorem~\ref{thm: decrease}, increasing $x$
moves mass to the right in the Binomial, so the lower-tail CDF in $p$ decreases and its $(1-\delta)$ quantile increases; with $n$ fixed this is equivalent to $\alpha(\epsilon)$ being strictly increasing in $\epsilon$.

(b) Similar to the proof of the statement (a) of Theorem~\ref{thm: decrease}, notice that 
\[
\alpha(\epsilon,1)\;:=\;\Betainv\!\big(1-\delta;\,\epsilon x+1,\,n-\epsilon x\big).
\]
For $i=\epsilon x+1$, $j=n-\epsilon x$, the Beta$(i,j)$ mean and variance are
$\mu_\epsilon=i/(i+j)$ and $\sigma_\epsilon^2=ij/[(i+j)^2(i+j+1)]$.
Approximating the $(1-\delta)$ quantile by the Normal formula gives
$\alpha(\epsilon)=\mu_\epsilon+z\sigma_\epsilon+O(1/n)$.
Differentiate at $\epsilon=1$ to obtain the first-order change:

\[
\frac{d\mu_\epsilon}{d\epsilon}\Big|_{\epsilon=1}=\frac{x}{n+1},\qquad
\frac{d\sigma_\epsilon}{d\epsilon}\Big|_{\epsilon=1}
=\frac{1}{2(n+1)\sqrt{n+2}}\cdot
\frac{( n - 1 - 2 x)x}{\sqrt{(x+1)(n-x)}}.
\]
A first-order Taylor expansion around $\epsilon=1$ yields the displayed approximation.
\end{proof}

\begin{figure}[t]
  \centering
  \begin{subfigure}{0.32\linewidth}
    \centering
    \includegraphics[height=0.13\textheight,keepaspectratio]{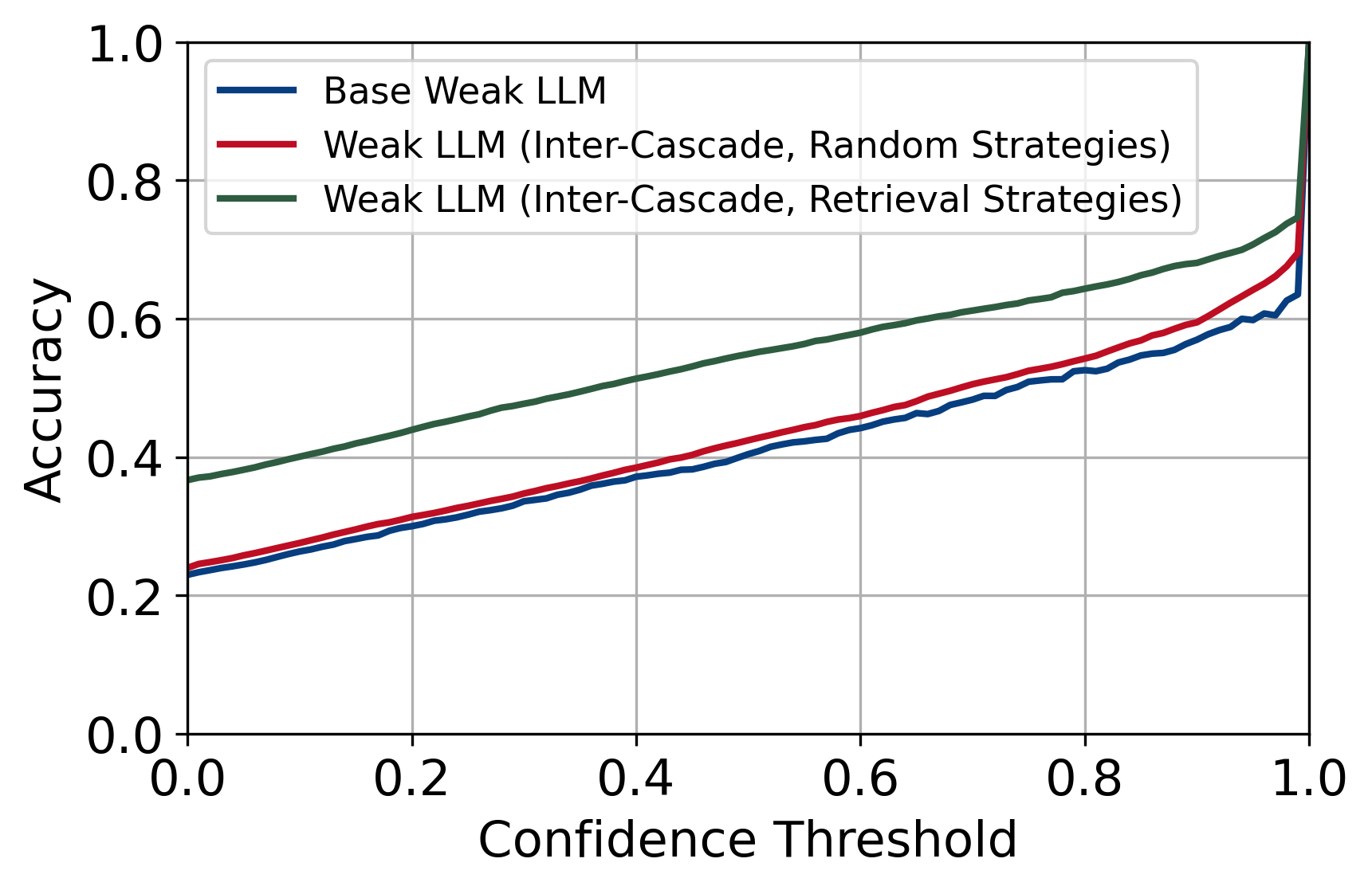}
    \caption{GSM-Plus}
  \end{subfigure}
  \begin{subfigure}{0.32\linewidth}
    \centering
    \includegraphics[height=0.13\textheight,keepaspectratio]{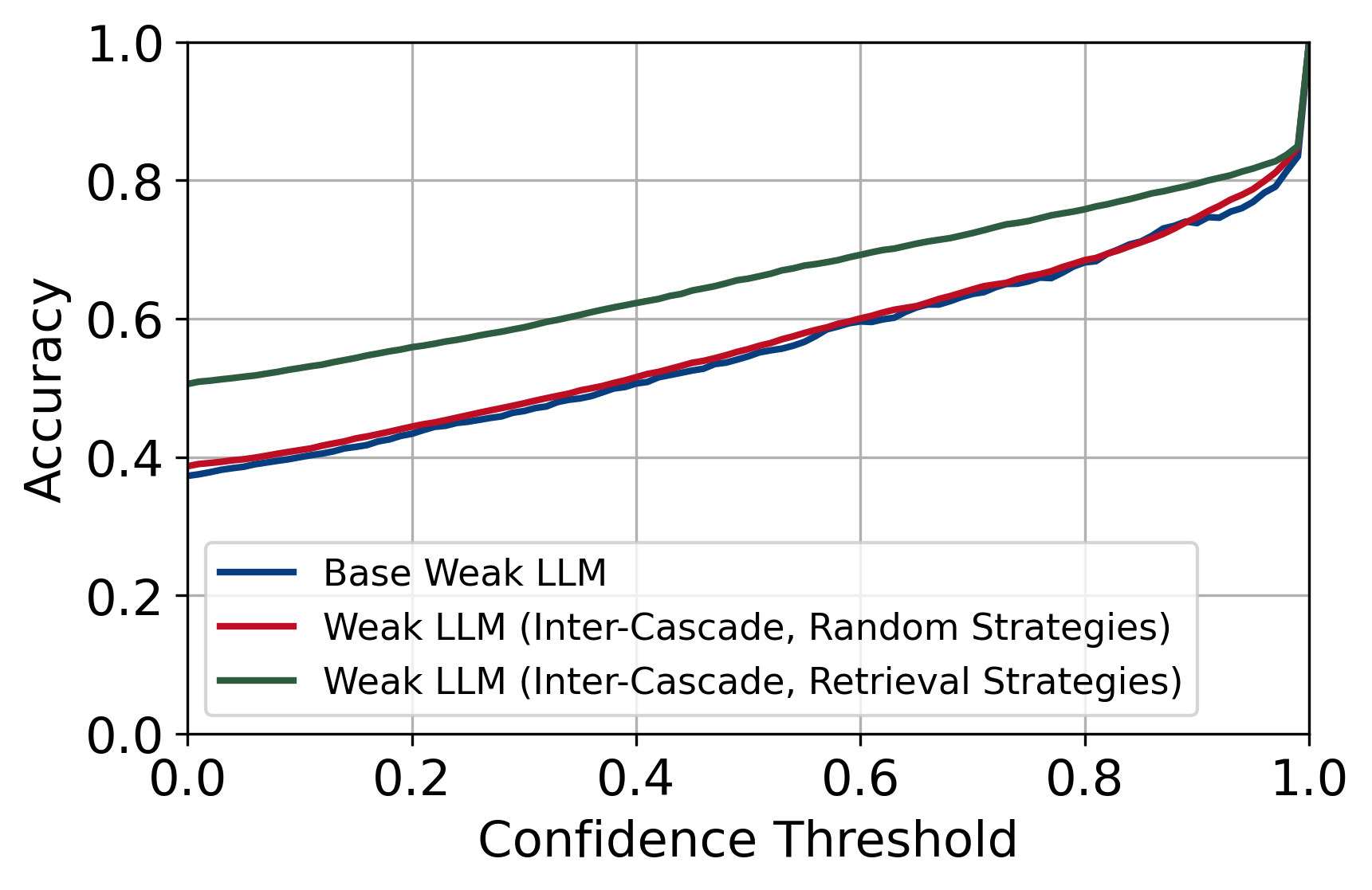}
    \caption{MetaMath}
  \end{subfigure}
  \begin{subfigure}{0.32\linewidth}
    \centering
    \includegraphics[height=0.13\textheight,keepaspectratio]{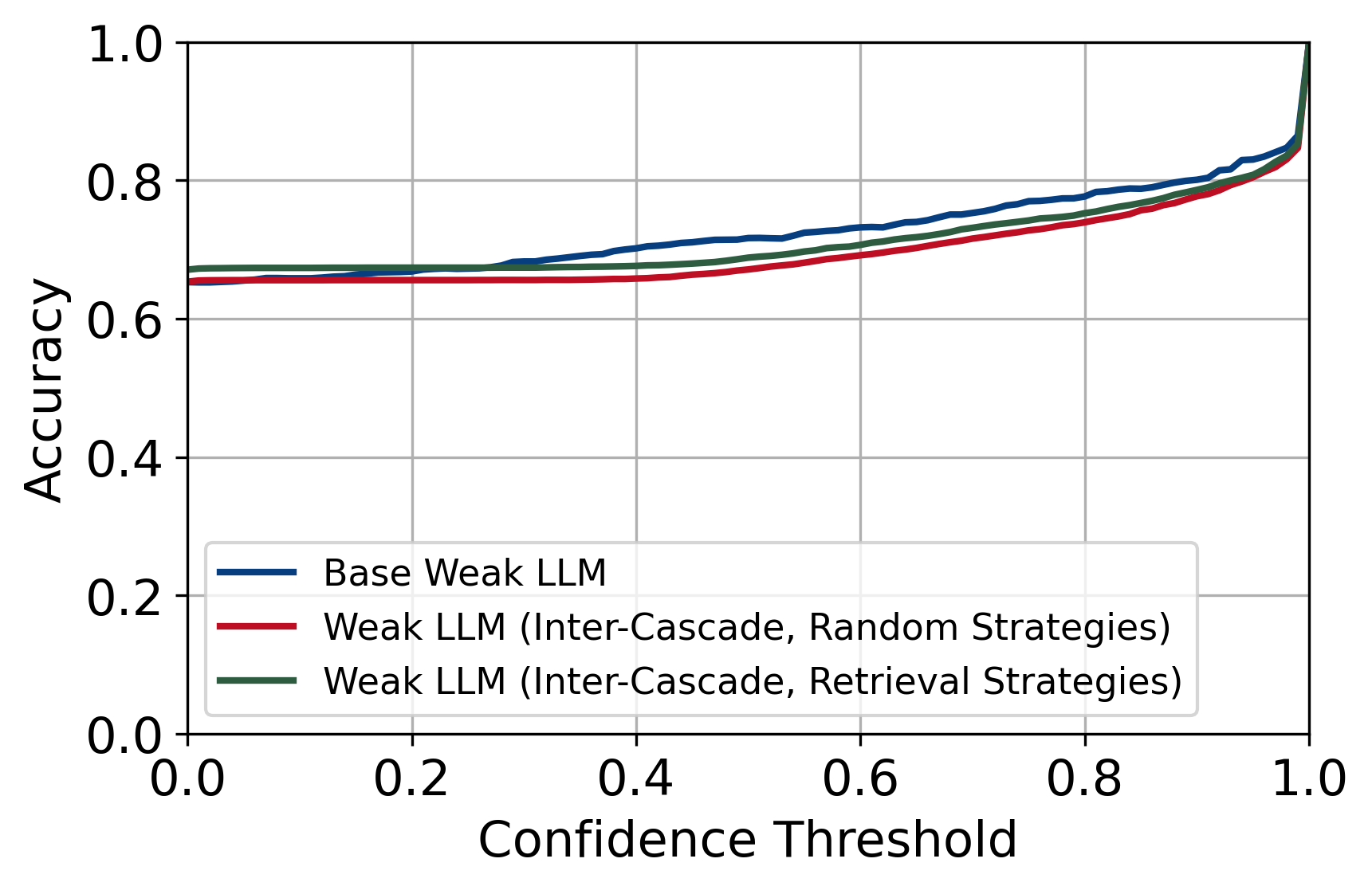}
    \caption{Nasa-History-MCQ}
  \end{subfigure}
  
  \caption{Accuracy as a function of the confidence threshold for the base Weak LLM and for the Weak LLM within the Inter-Cascade using random and retrieval strategies across three benchmarks.}
  \label{fig:dataset3_accuracy_vs_confidence}
\end{figure}

\begin{figure}[ht]
  \centering
  \begin{subfigure}{0.32\linewidth}
    \centering
    \includegraphics[height=0.13\textheight,keepaspectratio]{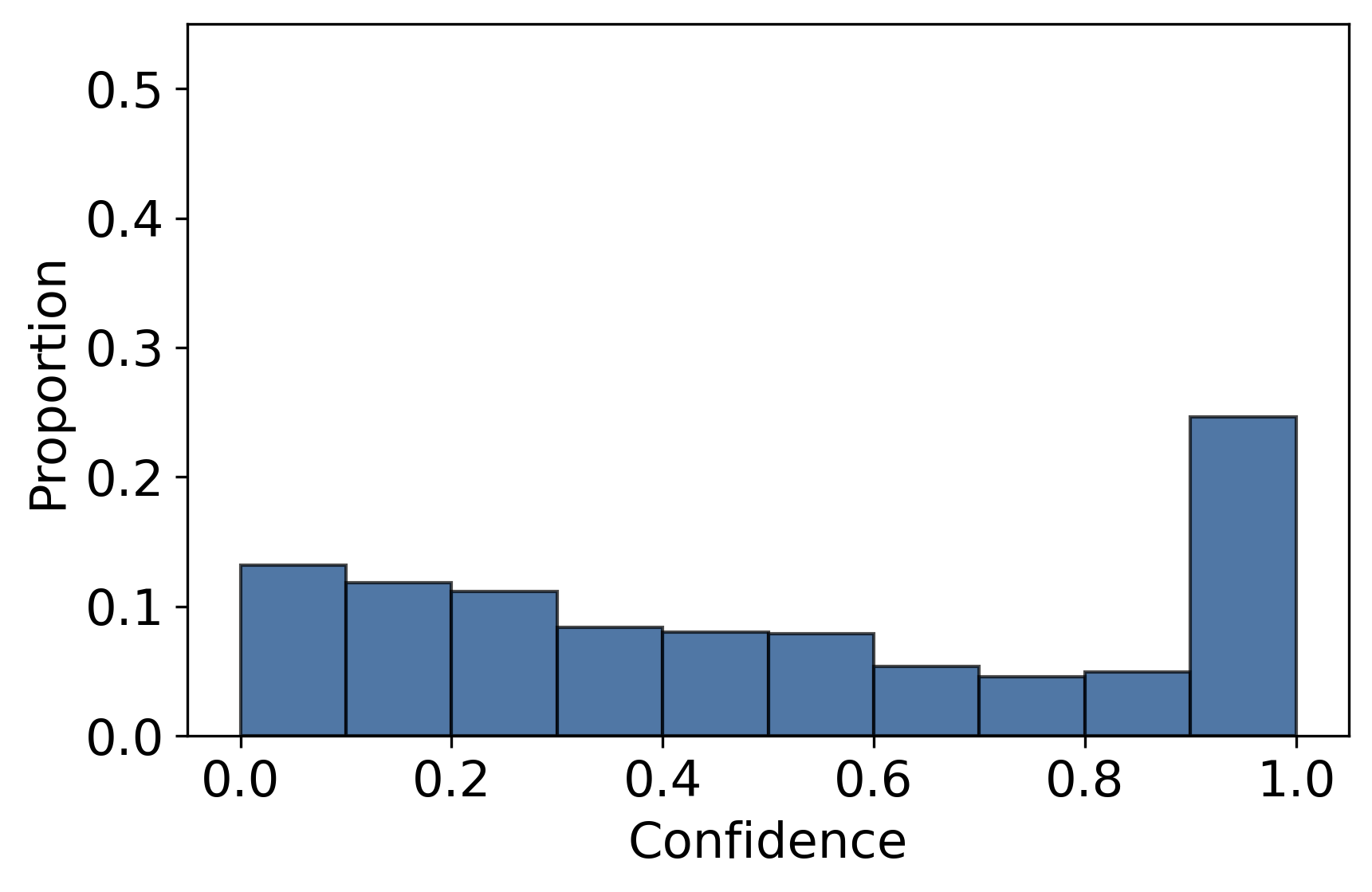}
    \caption{GSM-Plus}
    \label{fig:dataset1_conf_base}
  \end{subfigure}
  \hfill
  \begin{subfigure}{0.32\linewidth}
    \centering
    \includegraphics[height=0.13\textheight,keepaspectratio]{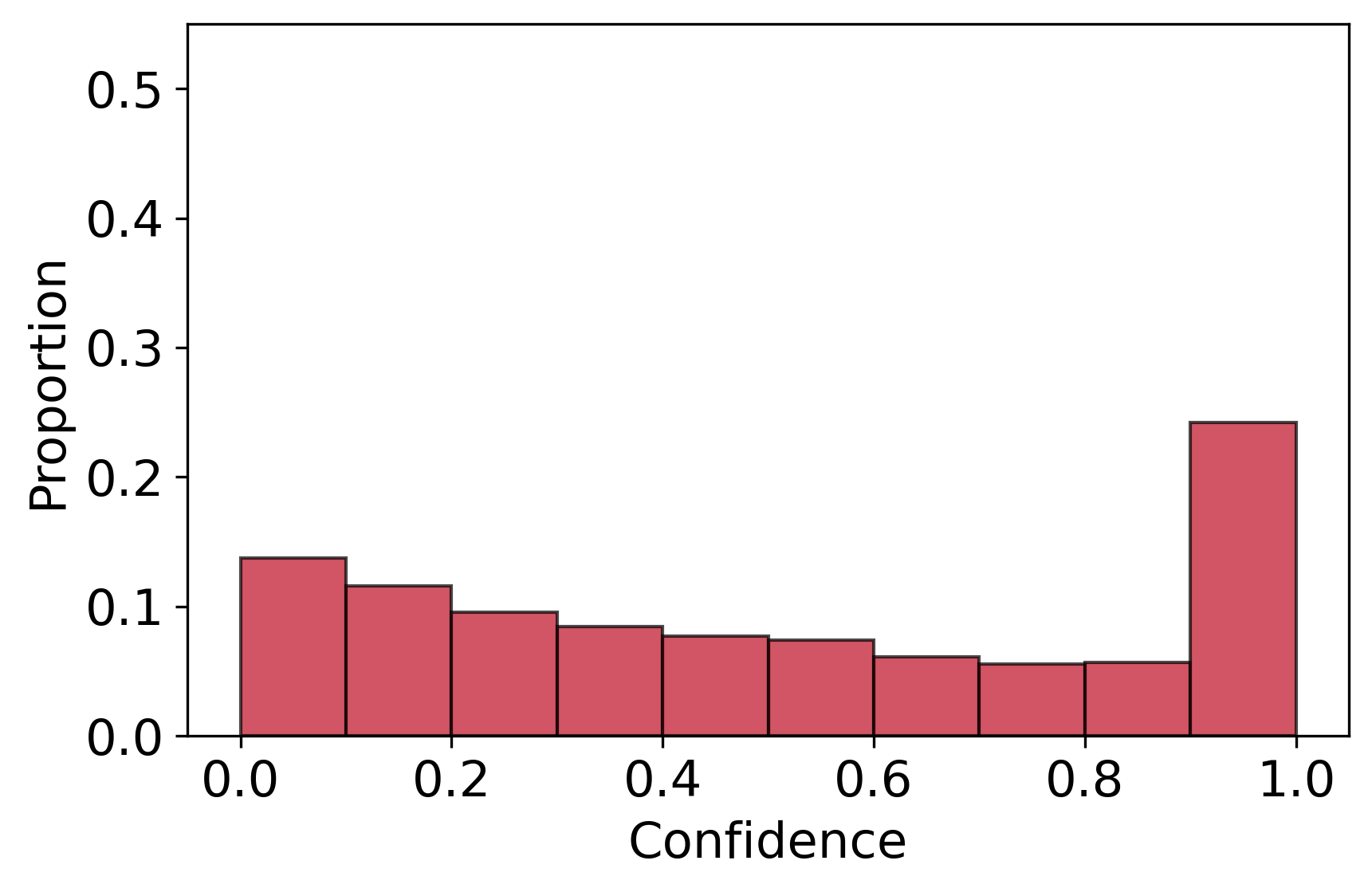}
    \caption{GSM-Plus}
    \label{fig:dataset1_conf_random}
  \end{subfigure}
  \hfill
  \begin{subfigure}{0.32\linewidth}
    \centering
    \includegraphics[height=0.13\textheight,keepaspectratio]{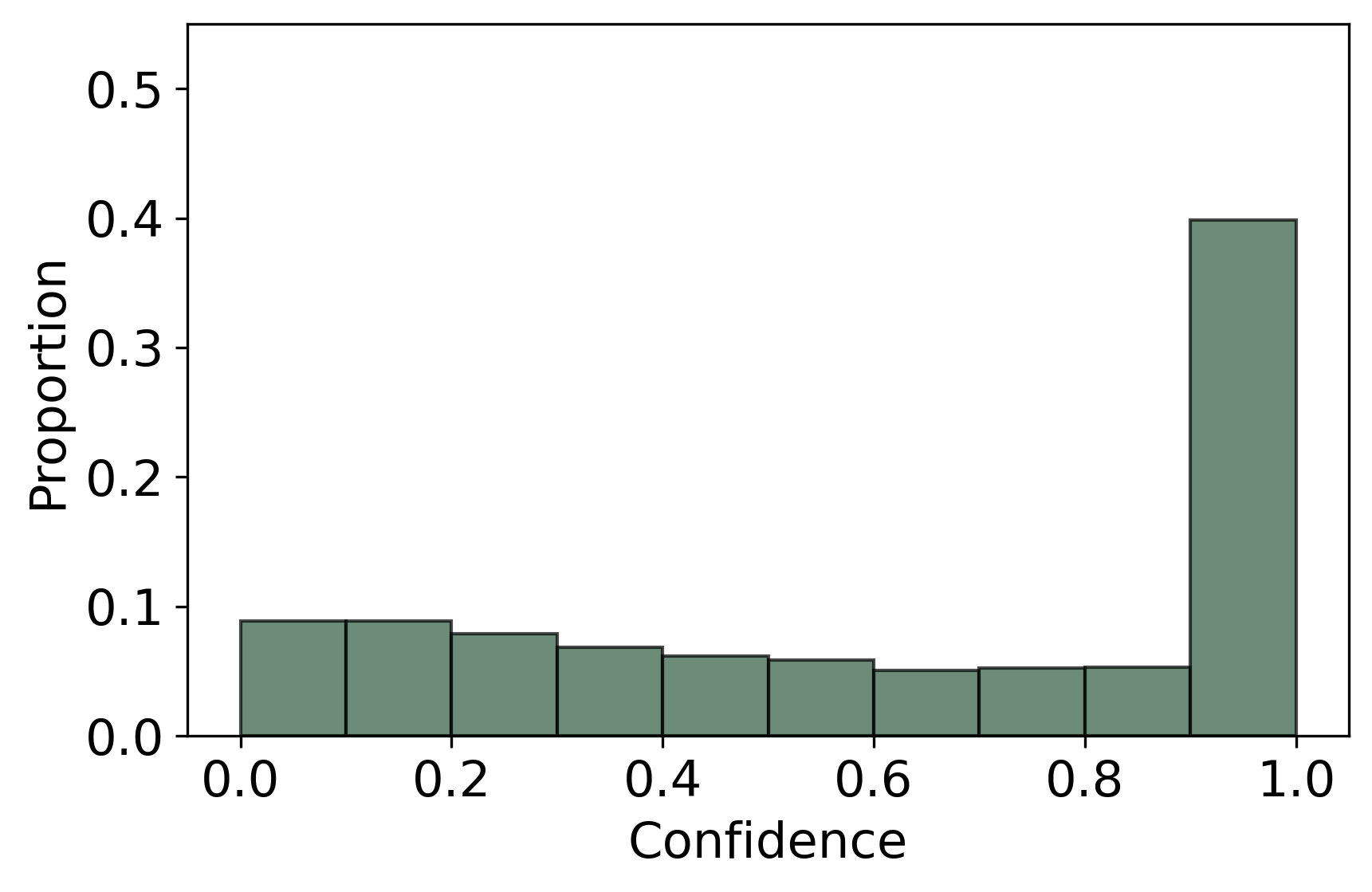}
    \caption{GSM-Plus}
    \label{fig:dataset1_conf_retrieval}
  \end{subfigure}

  \begin{subfigure}{0.32\linewidth}
    \centering
    \includegraphics[height=0.13\textheight,keepaspectratio]{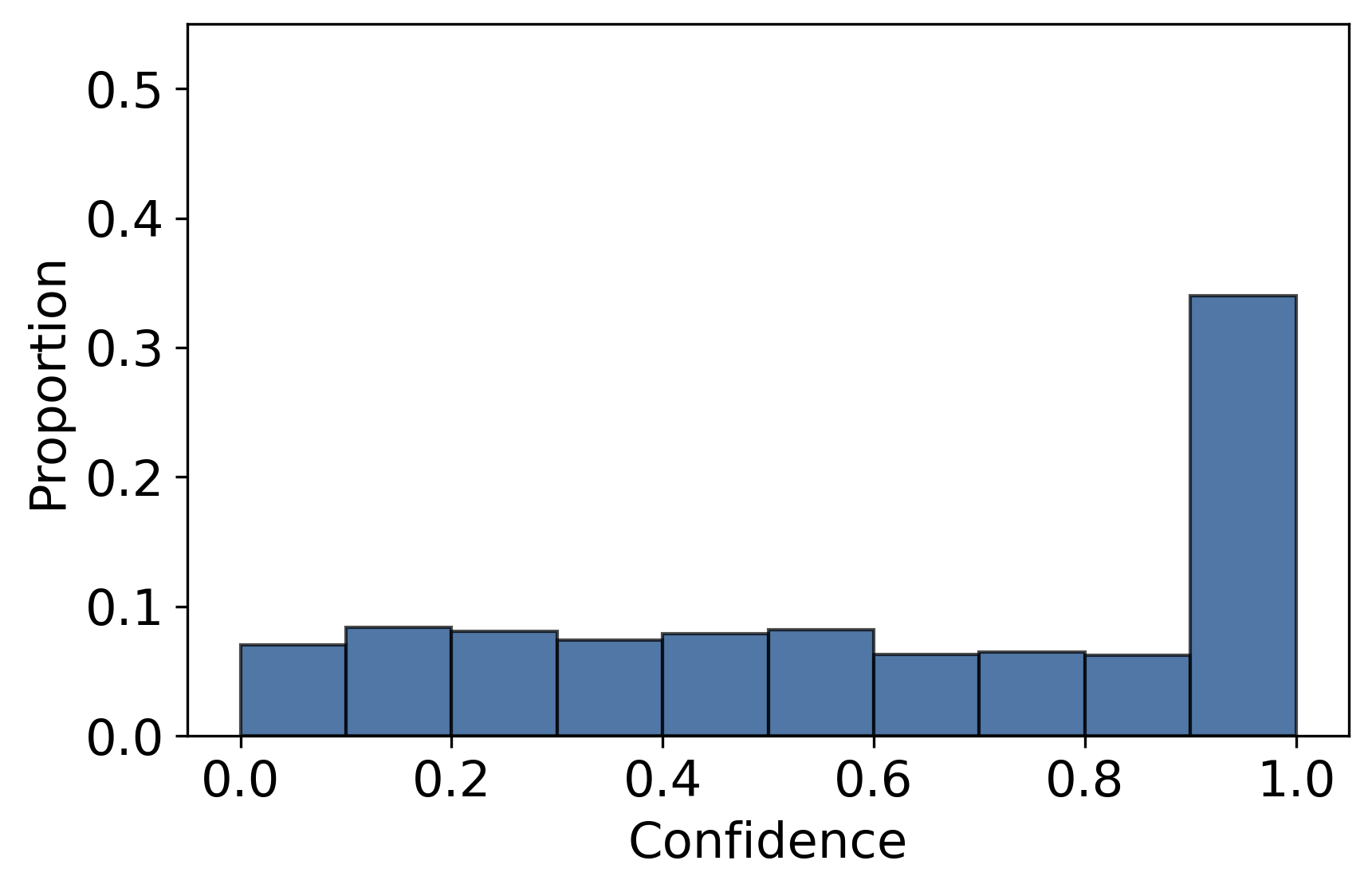}
    \caption{MetaMath}
    \label{fig:dataset2_conf_base}
  \end{subfigure}
  \hfill
  \begin{subfigure}{0.32\linewidth}
    \centering
    \includegraphics[height=0.13\textheight,keepaspectratio]{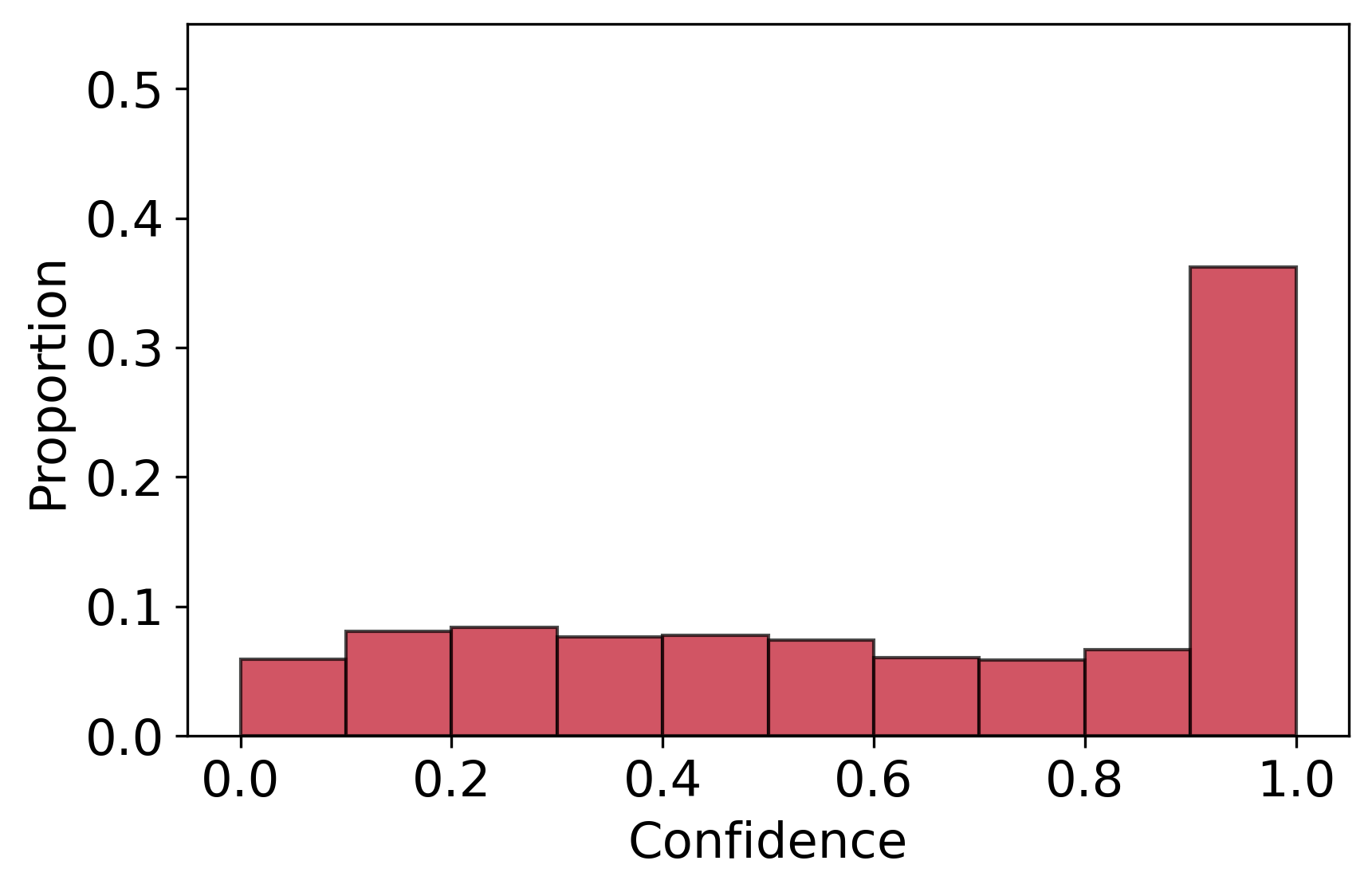}
    \caption{MetaMath}
    \label{fig:dataset2_conf_random}
  \end{subfigure}
  \hfill
  \begin{subfigure}{0.32\linewidth}
    \centering
    \includegraphics[height=0.13\textheight,keepaspectratio]{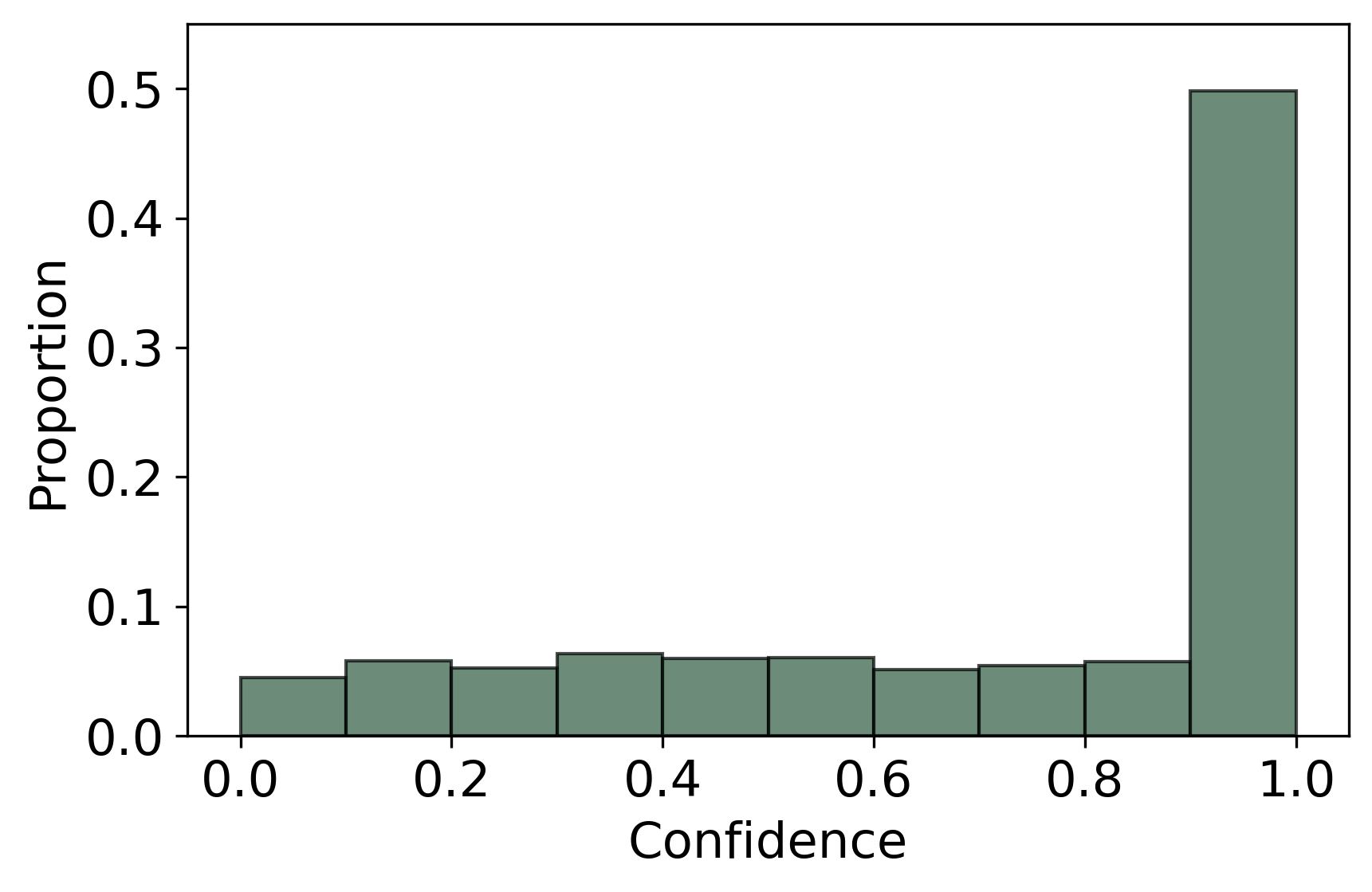}
    \caption{MetaMath}
    \label{fig:dataset2_conf_retrieval}
  \end{subfigure}

  \begin{subfigure}{0.32\linewidth}
    \centering
    \includegraphics[height=0.13\textheight,keepaspectratio]{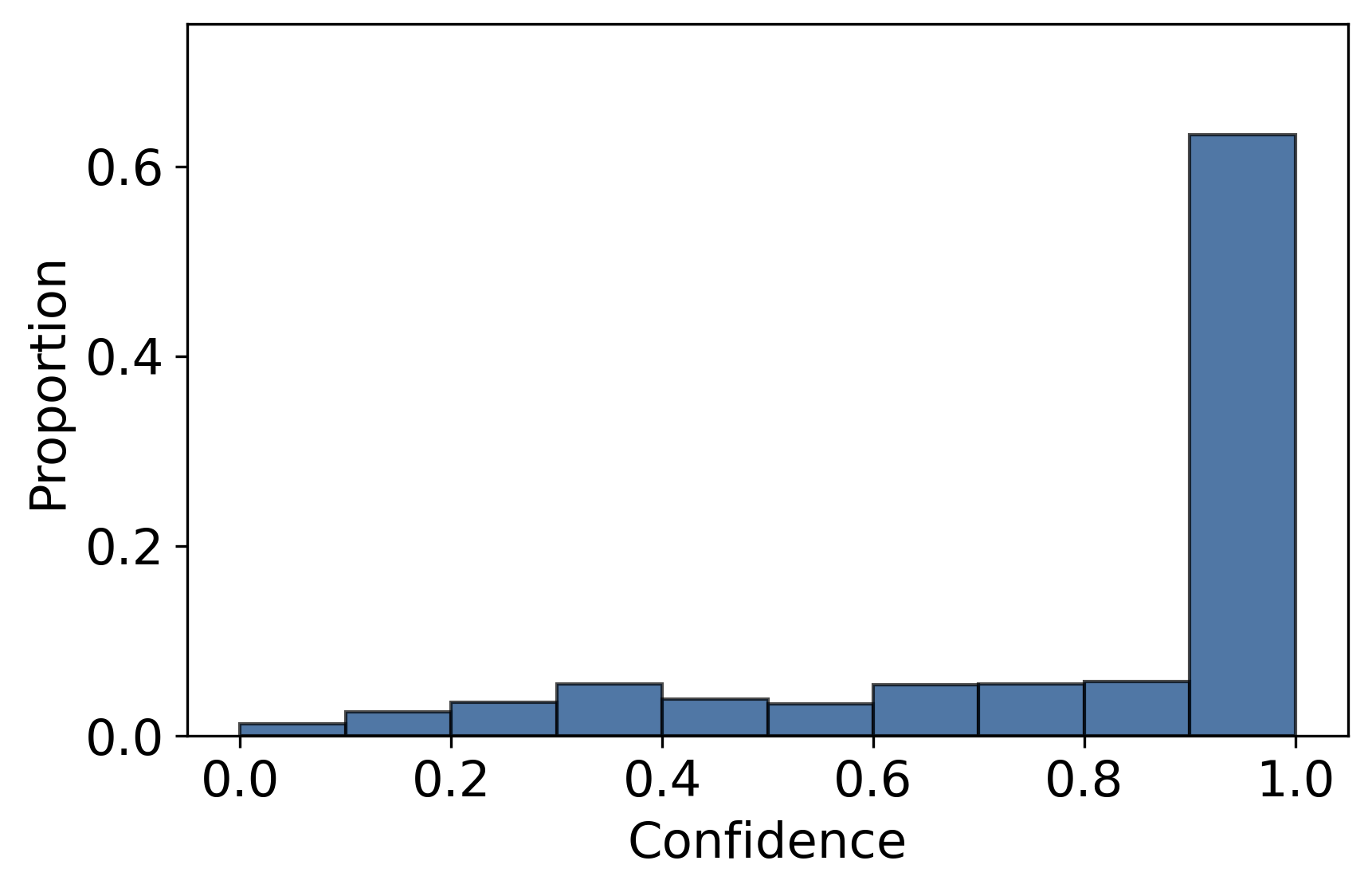}
    \caption{Nasa-History-MCQ}
    \label{fig:dataset3_conf_base}
  \end{subfigure}
  \hfill
  \begin{subfigure}{0.32\linewidth}
    \centering
    \includegraphics[height=0.13\textheight,keepaspectratio]{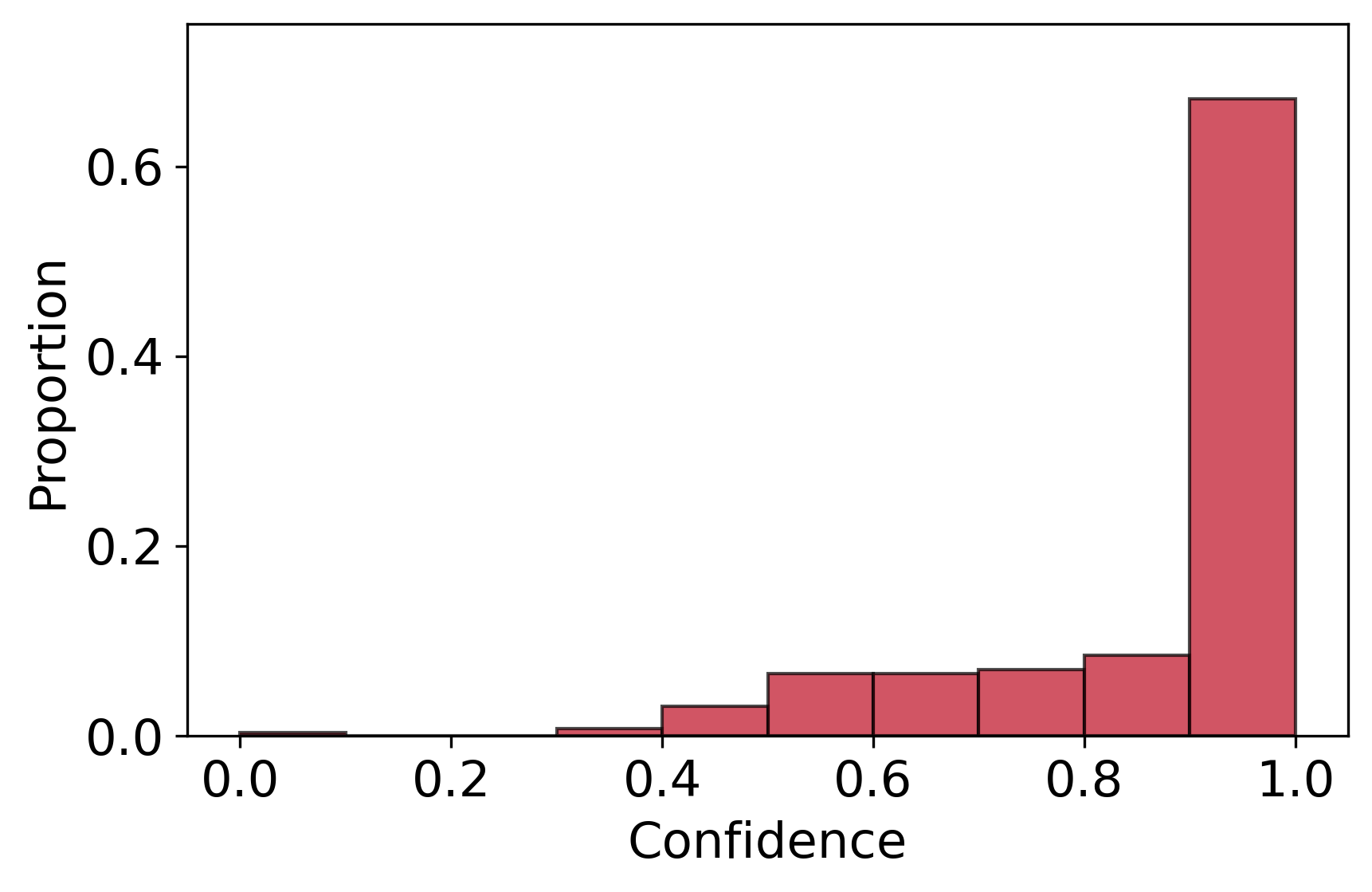}
    \caption{Nasa-History-MCQ}
    \label{fig:dataset3_conf_random}
  \end{subfigure}
  \hfill
  \begin{subfigure}{0.32\linewidth}
    \centering
    \includegraphics[height=0.13\textheight,keepaspectratio]{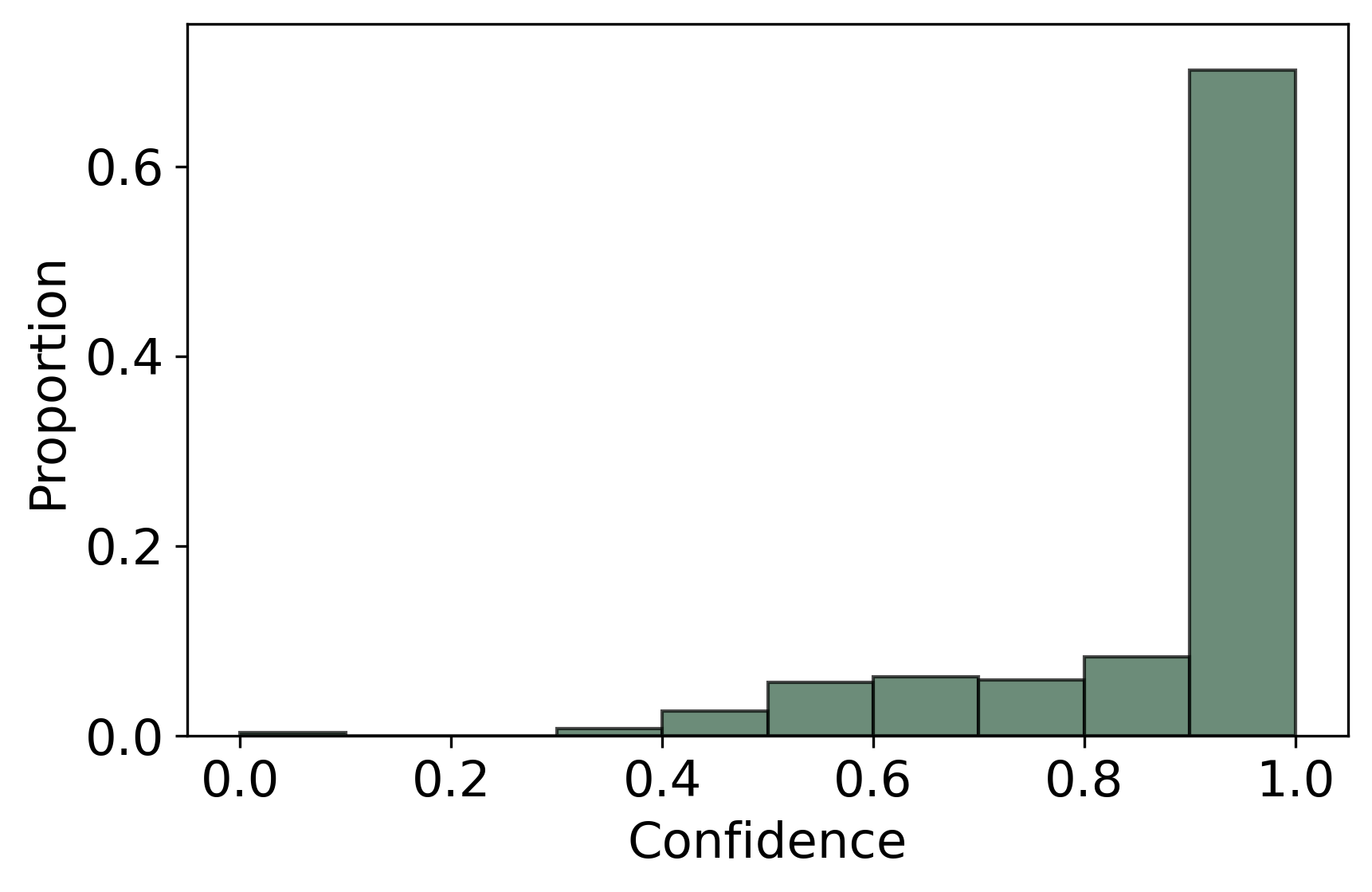}
    \caption{Nasa-History-MCQ}
    \label{fig:dataset3_conf_retrieval}
  \end{subfigure}

  \caption{Confidence histograms for three benchmarks. Columns correspond to (a)(d)(g) the base Weak LLM, (b)(e)(h) the Weak LLM within the Inter-Cascade using random strategies, and (c)(f)(i) the Weak LLM within the Inter-Cascade using retrieval strategies. Across all datasets, the Inter-Cascade with retrieval strategies concentrates probability mass near high confidence (0.9–1.0), while the base and random-strategy variants place more mass at lower confidence levels.}
  \label{fig:confidence_comparison_3x3}
\end{figure}

\section{Confidence Distribution}
\label{appendix:sec_confidence_dist}
Figures~\ref{fig:dataset3_accuracy_vs_confidence} and \ref{fig:confidence_comparison_3x3} present results for the GSM-Plus, MetaMath, and Nasa-History-MCQ datasets, complementing the GSM-Symbolic analyses in the main text. 

Figure~\ref{fig:dataset3_accuracy_vs_confidence} shows accuracy as a function of the confidence threshold for the base Weak LLM and for the Weak LLM within the Inter-Cascade using random and retrieval strategies. For each threshold, only queries with confidence equal to or above the threshold are considered, and accuracy is calculated as the proportion of correct predictions. Across the reasoning datasets (GSM-Plus and MetaMath), the Inter-Cascade with retrieval strategies consistently improves accuracy over the baseline and random-strategy variants. For the factual non-reasoning dataset (Nasa-History-MCQ), the Inter-Cascade achieves comparable performance.

Figure~\ref{fig:confidence_comparison_3x3} depicts the distribution of query confidence for the three benchmarks. Across all datasets, the Inter-Cascade with retrieval strategies concentrates probability mass near high confidence (0.9–1.0), whereas the base and random-strategy variants place more mass at lower confidence levels. These results further confirm that providing strategies helps the Weak LLM not only produce more accurate predictions but also better calibrate its confidence.

\section{Full Description of Benchmarks}\label{appendix:benchmarks}
\textbf{GSM-Symbolic.} 
The GSM-Symbolic benchmark, released by Apple's team~\citep{mirzadeh_gsm-symbolic_2025}, is a structured variant of GSM8K~\citep{cobbe_training_2021-1}. Unlike traditional benchmarks such as GSM8K, which present problems in a plain context, GSM-Symbolic reformulates problems into a more structured and abstract format following a symbolic template, providing a more reliable measure of models' reasoning capabilities. The dataset contains $12,500$ grade-school math problems. We randomly sample $1,250$ problems as the calibration set for threshold computation and use the remaining $11,250$ problems as the test set. The prompt template and an example problem are provided in Appendix~\ref{appendix:prompt_example}.

\textbf{GSM-Plus.} 
GSM-Plus~\citep{li_gsm-plus_2024} is derived from the $1,319$ test questions in GSM8K by introducing eight types of question variations: numerical substitution, digit expansion, integer-decimal-fraction conversion, adding operation, reversing operation, problem understanding, distractor insertion, and critical thinking. GSM-Plus thus comprises a total of $10,552$ question variations. We randomly sample $1,048$ problems as the calibration set for threshold computation and use the remaining $9,504$ problems as the test set. The prompt template and an example problem are provided in Appendix~\ref{appendix:prompt_example}.

\textbf{MetaMath.} 
MetaMath~\citep{yu_metamath_2024} is a dataset generated by bootstrapping the mathematical benchmarks GSM8K~\citep{cobbe_training_2021-1} and MATH~\citep{hendrycks2021measuringmathematicalproblemsolving}. The augmentation is performed in both forward and backward directions. In the forward direction, MetaMath contains the original and LLM-rephrased questions, while in the backward direction, it includes self-verification questions and FOBAR questions~\citep{jiang_forward-backward_2024}, resulting in a total of $395K$ diverse problems. For our experiments, we randomly select $1,000$ problems as the calibration set for threshold computation and use $20,000$ additional problems as the test set. The prompt template and an example problem are provided in Appendix~\ref{appendix:prompt_example}.


\textbf{NASA-History-MCQ.} 
NASA-History-MCQ~\citep{patrickfleith_NASA_History_MCQ_2025} is a multiple-choice question benchmark on the history of NASA. It contains $7.47K$ questions, and each question provides four answer choices. We randomly sample $1,000$ problems as the calibration set for threshold computation and use the remaining $6,469$ problems as the test set. The prompt template and an example problem are provided in Appendix~\ref{appendix:prompt_example}.

\textbf{BarExamQA.} 
BarExamQA~\citep{zhong_bar_exam_qa_2024} is a legal reasoning benchmark constructed from real U.S. bar examination questions. Each question is posed in a multiple-choice format and requires multi-step legal reasoning over complex legal fact patterns. BarexamQA contains a total of $954$ problems, we randomly sample $95$ problems as the calibration set for threshold computation and remaining $859$ as the test set.

\textbf{BigBench Hard.} 
BIG-Bench Hard~\citep{Suzgun2022ChallengingBT} is a subset of 23 particularly challenging BIG-Bench tasks for which no prior result from~\citep{srivastava_bigbench_2022} has outperformed the average human-rater score. It is a diverse benchmark designed to test capabilities of language models on a diverse set of crowd-sourced tasks. The benchmark aims to focus on the problems that beyond the capabilities of existing LLMs. We use $5412$ problems as test set and $599$ problems as calibration set for threshold computation. The calibration set are selected from each tasks with the same proportion. 

\textbf{GSM8K.} 
GSM8K~\citep{cobbe_training_2021-1} is a widely used grade-school math word problem benchmark designed to evaluate multi-step numerical reasoning. The dataset contains $7473$ training questions and $1719$ test questions, with each problem requiring several arithmetic operations and logical reasoning steps to reach the final answer. Following standard practice, we use problems in calibration set for threshold computation and use the remaining problems as the test set. 

\textbf{MedMCQA.} 
MedMCQA~\citep{pal_medmcqa_2022} is a large-scale multiple-choice question benchmark in the medical domain. It covers high-quality AIIMS and NEET PG entrance exam MCQs covering $2400$ healthcare topics and $21$ medical subjects. It contains over $194,000$ questions, each with four answer choices and a single correct answer. We randomly sample $2,000$ problems as the calibration set for threshold computation and use $8000$ additional problems as the test set. 

\section{Full Description of Token and API Cost Analysis}\label{app:full cost}
The full analysis on the token consumption, including input tokens and output token for the four benchmarks presented in main text is shown in Table~\ref{tab:full-cost}.

\begin{table}[htbp]
\centering
\caption{Token and API cost changes across datasets for Inter-Cascade compared with Jung's pipeline.}
\label{tab:full-cost}

\begingroup
\setlength{\tabcolsep}{3.5pt} 
\renewcommand{\arraystretch}{0.9} 
\setlength{\aboverulesep}{0.2ex} 
\setlength{\belowrulesep}{0.2ex}
\setlength{\cmidrulekern}{0.2em}

\begin{tabular}{lccccccc}
\toprule
\multirow{2}{*}{\textbf{Benchmark}} & 
\multicolumn{3}{c}{\textbf{Weak LLM Tokens}} & 
\multicolumn{3}{c}{\textbf{Strong LLM Tokens}} & 
\multirow{2}{*}{\shortstack{\textbf{Token}\\\textbf{Price}}} \\

\cmidrule(lr){2-4} \cmidrule(lr){5-7} 

 & Total & Input & Output & Total & Input & Output & \\
\midrule


GSM-Symb. & +147.66\% & +148.80\% & -17.10\% & -47.80\% & -45.80\% & -51.32\% & -49.63\% \\
\midrule

GSM-Plus & +145.96\% & +147.11\% & -3.56\% & -29.95\% & -29.51\% & -30.90\% & -30.41\% \\
\midrule

Meta.(20K) & +127.90\% & +128.66\% & -1.38\% & -52.18\% & -52.20\% & -52.12\% & -52.15\% \\
\midrule

NASA-Hist. & +132.58\% & +133.40\% & 0.99\% & -15.47\% & -15.22\% & -16.07\% & -15.75\% \\

\bottomrule
\end{tabular}
\endgroup
\end{table}

\section{Extensive Experiment on More Benchmarks}\label{app:extensive}

Although the Inter-Cascade diagram is motivated by the real-world scenarios that contain similar or repeated tasks, we also provide the result of our Inter-Cascade on extensive benchmarks that are more diverse and do not contain explicit sample variants: GSM8K~\citep{cobbe_training_2021}, BigBench Hard~\citep{Suzgun2022ChallengingBT}, BarExamQA~\citep{zhong_bar_exam_qa_2024} and MedMCQA~\citep{pal_medmcqa_2022}. The full description of those benchmarks are in Appendix~\ref{appendix:benchmarks}. We firstly test the accuracy of each single LLM on those benchmarks and the result is in Table~\ref{tab:single_llm_extra_benchmarks}.

\textbf{Inter-Cascade vs. Jung's LLM Cascade.} We evaluate our \textit{Inter-Cascade} pipeline and Jung's method, as shown in Table~\ref{tab:overall-per-dataset-tight-extensive}. Our method outperforms Jung's, with a $0.18\% - 3.96\%$ increase in Pipeline Accuracy. The Strong LLM Call Rate is reduced on all benchmarks, with reductions ranging from $1.52\%$ to $16.14\%$. Compared with the results on GSM-Symbolic, GSM-Plus and MetaMath benchmarks, the accuracy improvement is not that large, but the more important part is that our Inter-Cascade can still reach a better trade-off between accuracy and cost since our method still remarkablely reduce the usage of Strong LLM.  These results indicate that \textit{Inter-Cascade} pipeline is also beneficial across different categories of tasks on diverse benchmarks.

\textbf{Impact of Inter-Cascade on Weak LLM.} Having examined the overall pipeline improvements, including Pipeline Accuracy and Strong LLM Call Rate reduction, we now investigate how our proposed \textit{Inter-Cascade} affects the Weak LLM. As shown in Table~\ref{tab:weakllm-per-dataset-tight-extenvis}, our Weak LLM still outperforms the Weak LLM in the other pipeline across all benchmarks. The improvements on \emph{Weak Accuracy} are between $0.91\%$ and $9.56\%$ and the improvements on \emph{Weak Correct Accepted} are between $2.24\%$ and $15.56\%$. The results implies that even though we test our Inter-Cascade on diverse benchmarks, retrieving most similar problems and solution strategies can still help boosting the performance and confidence of Weak LLM.

According to experiment results for extensive benchmarks, it shows that Inter-Cascade not only work for tasks that contain constructive similarity, but also help in more general and diverse cases, since explicit or implicit similarity occurs everywhere and the pipeline in our Inter-Cascade take the advantage of the similarity nature of daily tasks.

\textbf{Token and API Cost Savings.} The results of analysis on cost and latency for extensive benchmarks are attached in Table~~\ref{tab:cost-extensive} and Table~\ref{tab:efficiency_analysis_moved-extensive}. The tendency is similar: integrating with strategies, the token usages on Weak LLM increase between $115.89\%$ and $216.37\%$, but since the \emph{Strong Call} decrease on all benchmark, the token usages on Strong LLM decrease between $1.28\%$ and $83.17\%$ and therefore, we can save $2.33\%$-$83.94\%$ money on API price. On the other hand, the average latency change on each query is between $0.005$s and $0.374$s on different benchmarks, which is acceptable to the user experience.

\begin{table}[ht]
\footnotesize
\caption{
Accuracies of the base LLMs on extensive benchmarks}
\label{tab:single_llm_extra_benchmarks}
\centering
\setlength{\tabcolsep}{4pt}
\begin{tabular}{llr|llr}

\hline
Dataset & LLM & Accuracy& Dataset & LLM & Accuracy\\
\hline
\multirow{2}{*}{GSM8K} & gpt-3.5-turbo & 31.46\% & \multirow{2}{*}{BigBench}   & gpt-3.5-turbo & 49.75\%\\
             & gemini-2.0-flash & 74.83\% &              & gemini-2.0-flash & 78.80\%\\
\cline{1-3}\cline{4-6}
\multirow{2}{*}{BarExamQA}     & gpt-3.5-turbo & 48.42\%  & \multirow{2}{*}{MedMCQA} & gpt-3.5-turbo & 62.80\%\\
             & gemini-2.0-flash & 78.95\% &              & gemini-2.0-flash & 83.05\%\\
\hline
\end{tabular}
\end{table}

\begin{table}[t]
\centering
\caption{Results across extensive datasets using different pipelines. ``Jung'' denotes Jung's LLM-Cascade and ``Our (Retrieval)'' denotes the Inter-Cascade with similarity-based retrieval. The number of strategies is fixed at $k=2$ for both Inter-Cascade settings. Metrics reported are Pipeline Accuracy (Pipeline Acc.), Strong LLM Call Rate (Strong Call), and Coverage Rate (Cov.). (a) GSM8K: For the Strong LLM, $\alpha_s = 0.2, \delta_s = 0.8, \lambda_s = 0.44$. For the Weak LLM, $\alpha_w = 0.5, \delta_w = 0.5, \lambda_w = 0.49$. (b) BigBench: No threshold is applied for the Strong LLM. For the Weak LLM, $\alpha_w = 0.4, \delta_w = 0.6, \lambda_w = 0.61$. (c) BarExamQA: No threshold is applied for the Strong LLM. For the Weak LLM, $\alpha_w = 0.5, \delta_w = 0.5, \lambda_w = 0.51$. (d) MedMCQA: No threshold is applied for the Strong LLM. For the Weak LLM, $\alpha_w = 0.3, \delta_w = 0.8, \lambda_w = 0.69$.}
\label{tab:overall-per-dataset-tight-extensive}


\begingroup
\footnotesize  
\setlength{\tabcolsep}{5pt} 
\renewcommand{\arraystretch}{1.0} 
\setlength{\aboverulesep}{0.3ex}
\setlength{\belowrulesep}{0.3ex}
\setlength{\cmidrulekern}{0.3em}

\begin{tabular}{@{} l l c c c @{}}
\toprule
\textbf{Data} & \textbf{Pipeline} & \textbf{Pipeline Acc. (\%)\,\boldmath$\uparrow$} &
\textbf{Strong Call (\%)\,\boldmath$\downarrow$} & \textbf{Cov. (\%)} \\
\midrule
\multirow{2}{*}{GSM8K}
& Jung            & 59.02 & 37.03 & 95.95 \\
& Our (Retrieval) & \textbf{60.62} & \textbf{35.46} & \textbf{96.05} \\
\midrule
\multirow{2}{*}{BigBench}
& Jung            & 64.14 & 33.04 & 100.00 \\
& Our (Retrieval) & \textbf{64.32} & \textbf{23.84} & 100.00 \\
\midrule
\multirow{2}{*}{BarExamQA}
& Jung            & 57.39 & 23.17 & 100.00 \\
& Our (Retrieval) & \textbf{58.67} & \textbf{21.65} & 100.00 \\
\midrule
\multirow{2}{*}{MedMCQA}
& Jung            & 71.69 & 18.74 & 100.00 \\
& Our (Retrieval) & \textbf{75.65} & \textbf{2.60} & 100.00 \\
\bottomrule
\end{tabular}
\endgroup
\end{table}

\begin{table}[t]
\centering
\caption{Results on Weak LLM across extensive datasets. Reported metrics are Weak LLM Accuracy (Weak Acc.) and Weak Correct Accepted (Weak Corr. Accpt.). Parameter settings are the same as in Table~\ref{tab:overall-per-dataset-tight-extensive}.}
\label{tab:weakllm-per-dataset-tight-extenvis}

\begingroup

\footnotesize  
\setlength{\tabcolsep}{5pt} 
\renewcommand{\arraystretch}{1.0} 
\setlength{\aboverulesep}{0.3ex}
\setlength{\belowrulesep}{0.3ex}
\setlength{\cmidrulekern}{0.3em}

\begin{tabular}{@{} l l c c @{}}
\toprule
\textbf{Data} & \textbf{Pipeline} & \textbf{Weak Acc. (\%)\,\boldmath$\uparrow$} & \textbf{Weak Corr. Accpt. (\%)\,\boldmath$\uparrow$} \\
\midrule
\multirow{2}{*}{GSM8K}
& Jung            & 37.06 & 33.38 \\
& Our (Retrieval) & \textbf{39.30} & \textbf{35.62} \\
\midrule
\multirow{2}{*}{BigBench}
& Jung            & 49.02 & 39.34 \\
& Our (Retrieval) & \textbf{49.93} & \textbf{46.60} \\
\midrule
\multirow{2}{*}{BarExamQA}
& Jung            & 47.50 & 39.81 \\
& Our (Retrieval) & \textbf{51.22} & \textbf{43.31} \\
\midrule
\multirow{2}{*}{MedMCQA}
& Jung            & 64.95 & 58.16 \\
& Our (Retrieval) & \textbf{74.51} & \textbf{73.72} \\
\bottomrule
\end{tabular}
\endgroup
\end{table}

\begin{table}[htbp]
\centering
\caption{Token and API cost changes across extensive datasets for Inter-Cascade compared with Jung's pipeline.}
\label{tab:cost-extensive}

\begingroup
\setlength{\tabcolsep}{3.5pt} 
\renewcommand{\arraystretch}{0.9} 
\setlength{\aboverulesep}{0.2ex} 
\setlength{\belowrulesep}{0.2ex}
\setlength{\cmidrulekern}{0.2em}

\begin{tabular}{lccccccc}
\toprule
\multirow{2}{*}{\textbf{Benchmark}} & 
\multicolumn{3}{c}{\textbf{Weak LLM Tokens}} & 
\multicolumn{3}{c}{\textbf{Strong LLM Tokens}} & 
\multirow{2}{*}{\shortstack{\textbf{Token}\\\textbf{Price}}} \\

\cmidrule(lr){2-4} \cmidrule(lr){5-7} 

 & Total & Input & Output & Total & Input & Output & \\
\midrule


GSM8K & +115.89\% & +116.56\% & -2.27\%& -3.25\% & -4.10\% & -1.28\%  & -2.33\% \\
\midrule

BigBench & +134.53\% & +135.32\% & -5.47\% & -26.37\% & -30.90\% & -19.67\% & -22.70\% \\
\midrule

BarExamQA & +216.37\% & +216.90\% & +0.12\% & -5.70\% & -5.39\% & -6.28\% & -5.98\% \\
\midrule

MedMCQA & +129.64\% & +130.70\% & -0.16\% & -84.74\% & -85.58\% & -83.17\% & -83.94\% \\

\bottomrule
\end{tabular}
\endgroup
\end{table}

\begin{table}[htbp]
\centering
\caption{Processing Latency and Strategy Repository Size across extensive datasets. Retrieval refers to the time spent on strategies matching and ranking. Generation refers to time spent on generating answer via API.}
\label{tab:efficiency_analysis_moved-extensive}

\begingroup
\setlength{\tabcolsep}{3.5pt} 
\renewcommand{\arraystretch}{0.9} 
\setlength{\aboverulesep}{0.2ex} 
\setlength{\belowrulesep}{0.2ex}
\setlength{\cmidrulekern}{0.2em}

\begin{tabular}{lcccccc}
\toprule
\multirow{2}{*}{\textbf{Benchmark}} & 
\multirow{2}{*}{\shortstack{\textbf{Tested}\\\textbf{Samples}}} & 
\multicolumn{3}{c}{\textbf{Our}} & 
\multicolumn{1}{c}{\textbf{Jung}} & 
\multirow{2}{*}{\shortstack{\textbf{Repository}\\\textbf{Size}}} \\

\cmidrule(lr){3-5} \cmidrule(lr){6-6}

 & & Total & Retrieval & Generation & Total & \\
\midrule


GSM8K & 7473 & 1.344s & 0.005s & 1.339s & 1.216s & 6.3MB \\
\midrule

BigBench & 5412 & 1.456s & 0.004s & 1.452s & 1.227s & 3.4MB \\
\midrule

BarExamQA & 859 & 1.686s & 0.254s & 1.432s & 1.312s & 1.1MB \\
\midrule

MedMCQA & 8000 & 0.975s & 0.004s & 0.971s & 0.970s & 6.3MB \\

\bottomrule
\end{tabular}
\endgroup
\end{table}

\section{Extra Ablation Study}\label{app:ablation}
To better evaluate the performance and generalization capacity of Inter-Cascade, we set up extra ablation studies in this section.
\subsection{Cold start}
To evaluate the effect of cold start of our strategy repository, we measure the dynamic pipeline accuracy for both Jung's method and our standard Inter-Cascade on GSM-Symbolic. The result in Figure~\ref{fig:cold_start} shows that at early stage, the pipeline accuracy for our Inter-Cascade is much close to baseline method: Jung~\citep{jung2024trustescalatellmjudges}. However, as the size of stored strategies increase, the performance of Inter-Cascade increase and gradually exceed Jung's method and eventually converges.
\begin{figure}[htbp]
    \centering
    \includegraphics[width=0.5\textwidth]{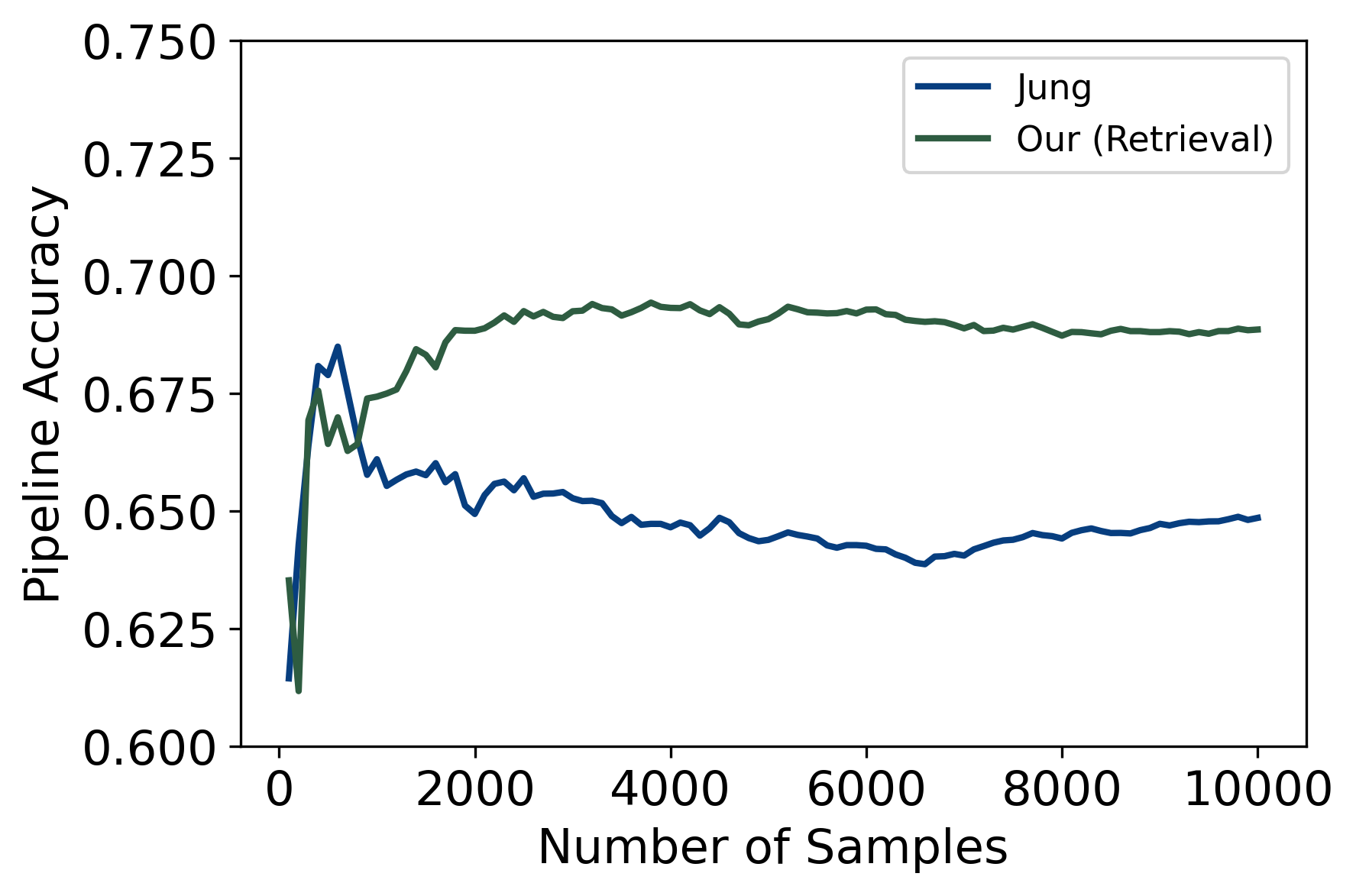}
    \caption{The dynamic of pipeline accuracy for both Jung's method and our standard Inter-Cascade on GSM-Symbolic.}
    \label{fig:cold_start}
\end{figure}

\subsection{Effect of Strategies Number}
To evaluate the effect the number of strategies we matched for each queries, we test the pipeline accuracy with different number of strategies that used for integrating with the input of Weak LLM. The result in Figure~\ref{fig:top-k} shows that the trend of pipeline accuracy is increasing first, reaching peak and then decreasing along with the number of strategies. The result makes sense because too few strategies might not retrieve the best strategy in repository, while too many strategies might distract the answer from certain query question, furthermore, there is a chance that the longer contexts may exceed the the maximum limit of the input context window. Both factors might undermine the performance of the pipeline accuracy. In our experiment on GSM-Symbolic benchmark, the empirical best number of strategies $k$ is 2.
\begin{figure}[htbp]
    \centering
    \includegraphics[width=0.5\textwidth]{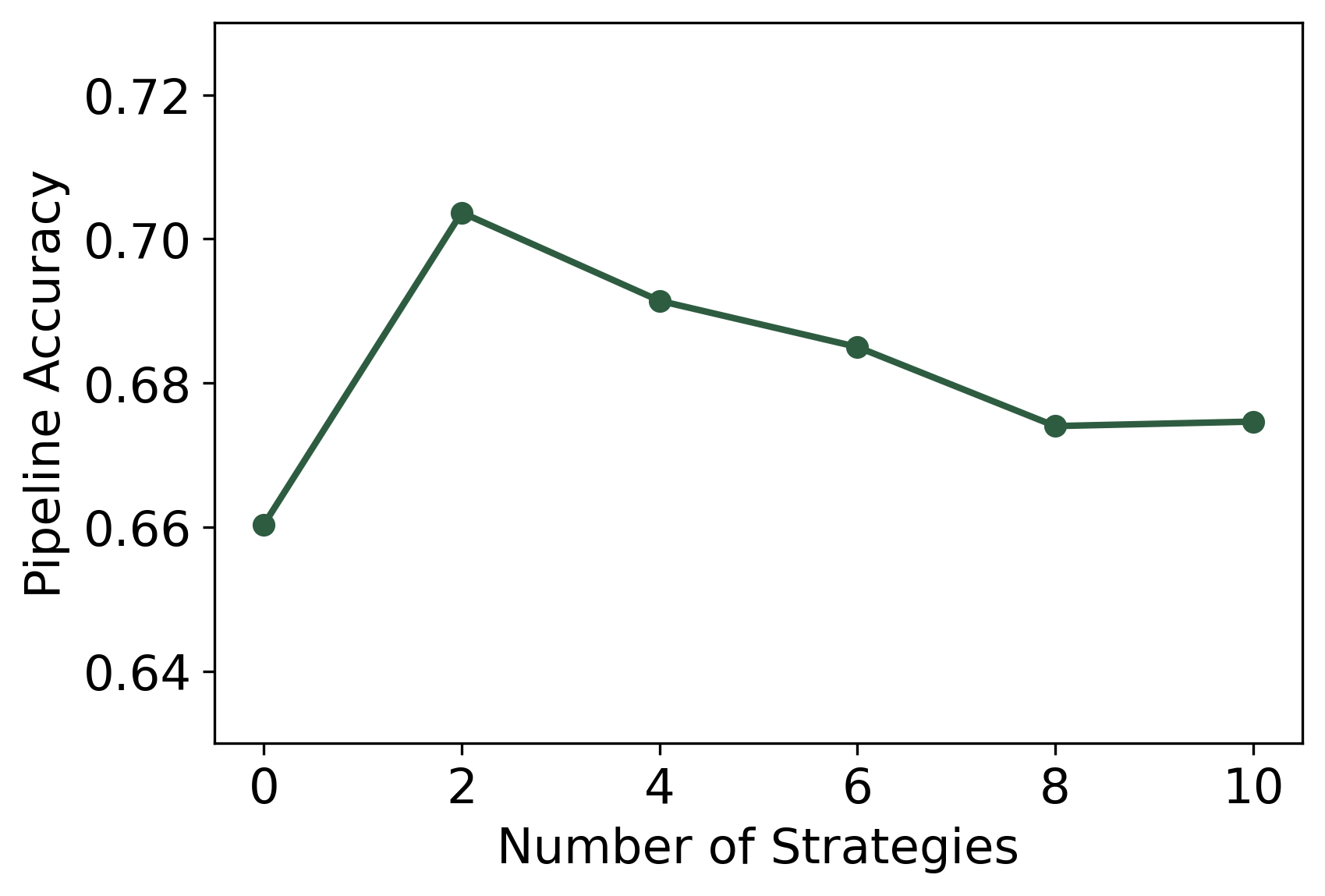}
    \caption{Effect of number of Strategies on pipeline accuracy for GSM-Symbolic Benchmark}
    \label{fig:top-k}
\end{figure}

\subsection{Results on New LLM Pairs}
To show that our Inter-Cascade is a framework that work general multiple LLM collaboration systems, we also test the result on different choice of Weak LLM and Strong LLM. We switch our Weak LLM to Gemini-2.0-flash and switch our Strong LLM to Gemini-2.5-flash. The results on single LLM are in Table~\ref{tab:single_llm_benchmarks_new_pair}. We also analyze the performance on those metrics: Pipeline Accuracy, Strong Call Rate, Weak Accuracy and Weak Correct Accept in Table~\ref{tab:overall-per-dataset-tight-new-pair} and Table~\ref{tab:weakllm-per-dataset-tight-new-pair}. The results shows that although we test on different pairs of Weak LLM and Strong LLM, the trend doesn't change: Inter-Cascade would help improve the accuracy of Weak LLM, pipeline accuracy, reduce the the usage of Strong LLM, reaching a better trade-off between the Accuracy and Cost in LLM Cascade systems.
\begin{table}[ht]
\footnotesize
\caption{
Accuracies of new pair of base LLMs on GSM-Symbolic Benchmark}
\label{tab:single_llm_benchmarks_new_pair}
\centering
\setlength{\tabcolsep}{4pt}
\begin{tabular}{llr}

\hline
Dataset & LLM & Accuracy\\
\hline
\multirow{2}{*}{GSM-Symbolic} & gemini-2.0-flash & 69.36\% \\
             & gemini-2.5-flash & 89.28\% \\

\hline
\end{tabular}
\end{table}

\begin{table}[t]
\centering
\caption{New LLM Pairs (Weak LLM: Gemini-2.0-flash; Strong LLM: Gemini-2.5-flash) Results on GSM-Symbolic dataset using different pipelines. ``Jung'' denotes Jung's LLM-Cascade and ``Our (Retrieval)'' denotes the Inter-Cascade with similarity-based retrieval. The number of strategies is fixed at $k=2$ for both Inter-Cascade settings. Metrics reported are Pipeline Accuracy (Pipeline Acc.), Strong LLM Call Rate (Strong Call), and Coverage Rate (Cov.). GSM-Symbolic: No threshold is applied for the Strong LLM. For the Weak LLM, $\alpha_w = 0.2, \delta_w = 0.8, \lambda_w = 0.47$. }
\label{tab:overall-per-dataset-tight-new-pair}


\begingroup
\footnotesize  
\setlength{\tabcolsep}{5pt} 
\renewcommand{\arraystretch}{1.0} 
\setlength{\aboverulesep}{0.3ex}
\setlength{\belowrulesep}{0.3ex}
\setlength{\cmidrulekern}{0.3em}

\begin{tabular}{@{} l l c c c @{}}
\toprule
\textbf{Data} & \textbf{Pipeline} & \textbf{Pipeline Acc. (\%)\,\boldmath$\uparrow$} &
\textbf{Strong Call (\%)\,\boldmath$\downarrow$} & \textbf{Cov. (\%)} \\
\midrule
\multirow{2}{*}{GSM-Symbolic}
& Jung            & 79.10 & 19.10 & 100.00 \\
& Our (Retrieval) & \textbf{85.50} & \textbf{9.90} & 100.00 \\

\bottomrule
\end{tabular}
\endgroup
\end{table}

\begin{table}[t]
\centering
\caption{New LLM Pairs (Weak LLM: Gemini-2.0-flash; Strong LLM: Gemini-2.5-flash) Results on Weak LLM across GSM-Symbolic dataset. Reported metrics are Weak LLM Accuracy (Weak Acc.) and Weak Correct Accepted (Weak Corr. Accpt.). Parameter settings are the same as in Table~\ref{tab:overall-per-dataset-tight-new-pair}.}
\label{tab:weakllm-per-dataset-tight-new-pair}

\begingroup

\footnotesize  
\setlength{\tabcolsep}{5pt} 
\renewcommand{\arraystretch}{1.0} 
\setlength{\aboverulesep}{0.3ex}
\setlength{\belowrulesep}{0.3ex}
\setlength{\cmidrulekern}{0.3em}

\begin{tabular}{@{} l l c c @{}}
\toprule
\textbf{Data} & \textbf{Pipeline} & \textbf{Weak Acc. (\%)\,\boldmath$\uparrow$} & \textbf{Weak Corr. Accpt. (\%)\,\boldmath$\uparrow$} \\
\midrule
\multirow{2}{*}{GSM-Symbolic}
& Jung            & 64.20 & 63.40 \\
& Our (Retrieval) & \textbf{77.00} & \textbf{76.80} \\
\bottomrule
\end{tabular}
\endgroup
\end{table}

\section{Prompt Templates and Examples}
\label{appendix:prompt_example}

Table~\ref{table:zero_shot_math} and Table~\ref{table:zero_shot_history} present the strategy-free prompt templates for the four datasets, along with one example question per dataset. Table~\ref{table:context_gsm_symbolic} - Table~\ref{table:context_nasa_history} show the strategy-based prompt templates and example inputs for each dataset. In our experiments, the number of strategies is set to $k=2$; these strategies and their corresponding answers are generated by the Strong LLM. Since the pipeline operates without human intervention, all strategies that exceed the Strong LLM confidence threshold $\lambda_s$ are accepted. Consequently, the $\repo$ may contain incorrect strategies or answers. Nonetheless, the results in Table~\ref{tab:overall-per-dataset-tight} and Table~\ref{tab:weakllm-per-dataset-tight} demonstrate the effectiveness of $\lambda_s$ and the robustness of our proposed Inter-Cascade pipeline.

\begin{table}[t]
\caption{Strategy-free prompt template with example questions from GSM-Symbolic, GSM-Plus, and MetaMath}
\vspace{-1em}
\centering
\small   
\begin{tcolorbox}[enhanced, width=\linewidth, colback=white]
\textbf{Prompt Template:} \\[0.3em]
Based on the question below, please strictly follow this format when answering: \\
1. Start with \text{[Strategy]} section explaining the general approach for solving similar problems; \\
2. End with \text{[Answer]} section containing ONLY the value. (1) Do NOT include units such as minutes, feet, etc.; (2) If the question asks for a percentage, ONLY provide the number (e.g., answer $20$ instead of $20\%$); (3) Do NOT include any explanations; (4) If there is no answer, RETURN None as the value. \\ \\
Example: \\
\text{[Question]}: $x + y = 10, y = 4$, what is $x$? \\
\text{[Strategy]}: To solve for $x$, isolate $x$ by subtracting $y$ from both sides of the equation. $x = 10 - y = 10 - 4 = 6.$ \\
\text{[Answer]}: $6$ \\ \\
Now answer this question: \\
\text{[Question]}: \text{\{question\}} \\
\text{[Strategy]}: \\
\text{[Answer]}:
\bigskip\hrule\bigskip
\textbf{GSM-Symbolic Example Question:} \\[0.3em]
\text{[Question]}: A fog bank rolls in from the ocean to cover a city. It takes $495$ minutes to cover every $95$ miles of the city. If the city is $95$ miles across from the oceanfront to the opposite inland edge, how many minutes will it take for the fog bank to cover the whole city? \\ \\
\textbf{GSM-Plus Example Question:} \\[0.3em]
\text{[Question]}: A clothing store has $60$ white shirts and $70$ floral shirts. Two-thirds of the white shirts have collars, and $30$ of the floral shirts have buttons. How many more floral shirts with no buttons are there than white shirts with no collars? \\ \\
\textbf{MetaMath Example Question:} \\[0.3em]
\text{[Question]}: Lara bought $52$ stems of flowers at the flower shop. She gave $15$ flowers to her mom and gave $x$ more flowers than she gave to her mom to her grandma. She put the rest in a vase. How many stems of flowers did Lara put in the vase? If we know the answer to the above question is $16$, what is the value of unknown variable $x$?
\end{tcolorbox}
\label{table:zero_shot_math}
\end{table}

\begin{table}[t]
\caption{Strategy-free prompt template with example question from NASA-History-MCQ}
\vspace{-1em}
\centering
\small   
\begin{tcolorbox}[enhanced, width=\linewidth, colback=white]
\textbf{Prompt Template:} \\[0.3em]
Based on the question below, please strictly follow this format when answering: \\
1. Start with \text{[Strategy]} section explaining the general approach for solving similar problems; \\
2. End with \text{[Answer]} section containing ONLY the single capital letter of the correct option (exactly one of: A, B, C, D). Do NOT include the option text or any explanation. \\ \\
Example: \\
\text{[Question]}: Which of the following was a direct result of the Treaty of Versailles (1919)? \\
A. The outbreak of World War II \\
B. The United States declaring war on Germany and entering World War I \\
C. The establishment of the Fascist regime in Italy \\
D. The creation of the League of Nations \\
\text{[Strategy]}: Recall that the treaty ended World War I and included peacekeeping measures. Compare options with the treaty: only the League of Nations was directly established by it. \\
\text{[Answer]}: D \\ \\
Now answer this question: \\
\text{[Question]}: \text{\{question\}} \\
\text{[Strategy]}: \\
\text{[Answer]}:
\bigskip\hrule\bigskip
\textbf{Nasa-History-MCQ Example Question:} \\[0.3em]
\text{[Question]}: What three distinct stages comprise current psychological support protocols for US astronauts? \\
A. Individual, group, family \\ 
B. Training, mission, debriefing \\
C. Cognitive, emotional, behavioral \\
D. Preflight, in-flight, postflight
\end{tcolorbox}
\label{table:zero_shot_history}
\end{table}

\begin{table}[t]
\caption{Strategy-based prompt template with example input from GSM-Symbolic}
\vspace{-1em}
\centering
\small   
\begin{tcolorbox}[enhanced, width=\linewidth, colback=white]
\textbf{Prompt Template:} \\[0.3em]
Using the following problem examples: \\
\text{\{strategy\}} \\ \\
Based on the question below, please strictly follow this format when answering: \\
1. Start with \text{[Strategy]} section explaining the general approach for solving similar problems; \\
2. End with \text{[Answer]} section containing ONLY the value. (1) Do NOT include units such as minutes, feet, etc.; (2) If the question asks for a percentage, ONLY provide the number (e.g., answer $20$ instead of $20\%$); (3) Do NOT include any explanations; (4) If there is no answer, RETURN None as the value. \\ \\
Now answer this question: \\
\text{[Question]}: \text{\{question\}} \\
\text{[Strategy]}: \\
\text{[Answer]}:
\bigskip\hrule\bigskip
\textbf{GSM-Symbolic Example Input:} \\[0.3em]
Using the following problem examples:\\
1. \text{[Question]}: A fog bank rolls in from the ocean to cover a city. It takes $93$ minutes to cover every $2$ miles of the city. If the city is $24$ miles across from the oceanfront to the opposite inland edge, how many minutes will it take for the fog bank to cover the whole city? \\
\text{[Strategy]}: First, determine how many $2$-mile segments are in the $24$-mile city. Divide the total distance of the city by the length of each segment. Then, multiply the number of segments by the time it takes to cover each segment to find the total time. \\
\text{[Answer]}: $1116$ \\ \\
2. \text{[Question]}: A fog bank rolls in from the ocean to cover a city. It takes $114$ minutes to cover every $3$ miles of the city. If the city is $99$ miles across from the oceanfront to the opposite inland edge, how many minutes will it take for the fog bank to cover the whole city? \\
\text{[Strategy]}: First, determine the rate at which the fog bank covers the city in miles per minute. Then, multiply this rate by the total distance of the city to find the total time it takes to cover the city. The rate is $3$ miles / $114$ minutes = $1/38$ miles per minute. The total time is ($1/38$ miles/minute) $*$ $99$ miles = $99/38$ minutes. Simplify the fraction $99/38 = 2.60526315789$. Multiply $114$ by $99/3$ to get the answer $114 * (99/3) = 114 * 33 = 3762$. \\
\text{[Answer]}: $3762$ \\ \\
Based on the question below, please strictly follow this format when answering: \\
1. Start with \text{[Strategy]} section explaining the general approach for solving similar problems; \\
2. End with \text{[Answer]} section containing ONLY the value. (1) Do NOT include units such as minutes, feet, etc.; (2) If the question asks for a percentage, ONLY provide the number (e.g., answer $20$ instead of $20\%$); (3) Do NOT include any explanations; (4) If there is no answer, RETURN None as the value. \\ \\
Now answer this question:\\
\text{[Question]}: A fog bank rolls in from the ocean to cover a city. It takes $495$ minutes to cover every $95$ miles of the city. If the city is $95$ miles across from the oceanfront to the opposite inland edge, how many minutes will it take for the fog bank to cover the whole city? \\
\text{[Strategy]}: \\
\text{[Answer]}:
\end{tcolorbox}
\label{table:context_gsm_symbolic}
\end{table}

\begin{table}[t]
\caption{Strategy-based prompt template with example input from GSM-Plus}
\vspace{-1em}
\centering
\small   
\begin{tcolorbox}[enhanced, width=\linewidth, colback=white]
\textbf{Prompt Template:} \\[0.3em]
Using the following problem examples: \\
\text{\{strategy\}} \\ \\
Based on the question below, please strictly follow this format when answering: \\
1. Start with \text{[Strategy]} section explaining the general approach for solving similar problems; \\
2. End with \text{[Answer]} section containing ONLY the value. (1) Do NOT include units such as minutes, feet, etc.; (2) If the question asks for a percentage, ONLY provide the number (e.g., answer $20$ instead of $20\%$); (3) Do NOT include any explanations; (4) If there is no answer, RETURN None as the value. \\ \\
Now answer this question: \\
\text{[Question]}: \text{\{question\}} \\
\text{[Strategy]}: \\
\text{[Answer]}:
\bigskip\hrule\bigskip
\textbf{GSM-Plus Example Input:} \\[0.3em]
Using the following problem examples:\\
1. \text{[Question]}: A clothing store has some white shirts and $50$ floral shirts. Half of the white shirts have collars, and $20$ of the floral shirts have buttons. How many more floral shirts with no buttons are there than white shirts with no collars? \\
\text{[Strategy]}: Let $W$ be the number of white shirts. The number of white shirts with collars is $W/2$, so the number of white shirts without collars is $W - W/2 = W/2$. The number of floral shirts is $50$. The number of floral shirts with buttons is $20$, so the number of floral shirts without buttons is $50 - 20 = 30$. We want to find the difference between the number of floral shirts without buttons and the number of white shirts without collars, which is $30 - W/2$. However, we are not given the number of white shirts, $W$. Therefore, we cannot determine the exact numerical value of the difference. \\
\text{[Answer]}: None \\ \\
2. \text{[Question]}: A clothing store has $40$ white shirts and $50$ floral shirts. Half of the white shirts have collars, and some of the floral shirts have buttons. There are $10$ more floral shirts with no buttons than white shirts with no collars. How many floral shirts have buttons? \\
\text{[Strategy]}: 1. Calculate the number of white shirts with collars: $40 / 2 = 20$. 2. Calculate the number of white shirts without collars: $40 - 20 = 20$. 3. Calculate the number of floral shirts without buttons: $20 + 10 = 30$. 4. Calculate the number of floral shirts with buttons: $50 - 30 = 20$. \\
\text{[Answer]}: $20$ \\ \\
Based on the question below, please strictly follow this format when answering: \\
1. Start with \text{[Strategy]} section explaining the general approach for solving similar problems; \\
2. End with \text{[Answer]} section containing ONLY the value. (1) Do NOT include units such as minutes, feet, etc.; (2) If the question asks for a percentage, ONLY provide the number (e.g., answer $20$ instead of $20\%$); (3) Do NOT include any explanations; (4) If there is no answer, RETURN None as the value. \\ \\
Now answer this question:\\
\text{[Question]}: A clothing store has $60$ white shirts and $70$ floral shirts. Two-thirds of the white shirts have collars, and $30$ of the floral shirts have buttons. How many more floral shirts with no buttons are there than white shirts with no collars? \\
\text{[Strategy]}: \\
\text{[Answer]}:
\end{tcolorbox}
\label{table:context_gsm_plus}
\end{table}

\begin{table}[t]
\caption{Strategy-based prompt template with example input from MetaMath}
\vspace{-1em}
\centering
\small   
\begin{tcolorbox}[enhanced, width=\linewidth, colback=white]
\textbf{Prompt Template:} \\[0.3em]
Using the following problem examples: \\
\text{\{strategy\}} \\ \\
Based on the question below, please strictly follow this format when answering: \\
1. Start with \text{[Strategy]} section explaining the general approach for solving similar problems; \\
2. End with \text{[Answer]} section containing ONLY the value. (1) Do NOT include units such as minutes, feet, etc.; (2) If the question asks for a percentage, ONLY provide the number (e.g., answer $20$ instead of $20\%$); (3) Do NOT include any explanations; (4) If there is no answer, RETURN None as the value. \\ \\
Now answer this question: \\
\text{[Question]}: \text{\{question\}} \\
\text{[Strategy]}: \\
\text{[Answer]}:
\bigskip\hrule\bigskip
\textbf{MetaMath Example Input:} \\[0.3em]
Using the following problem examples:\\
1. \text{[Question]}: Martha has $x$ crayons. She lost half of them, so she bought a new set of $20$ crayons. Martha has $29$ crayons in total after the purchase. What is the value of unknown variable $x$? \\
\text{[Strategy]}: Let $x$ be the number of crayons Martha initially had. She lost half of them, so she had $x/2$ crayons left. Then she bought 20 new crayons, so she had $x/2 + 20$ crayons. We are given that she has $29$ crayons in total, so $x/2 + 20 = 29$. Subtracting $20$ from both sides gives $x/2 = 9$. Multiplying both sides by $2$ gives $x = 18$. \\
\text{[Answer]}: $18$ \\ \\
2. \text{[Question]}: Justin and Sabrina at Pine Grove Elementary have $50$ pencils combined. Justin has $x$ more than twice as many pencils as Sabrina. Sabrina has $14$ pencils. What is the value of unknown variable $x$? \\
\text{[Strategy]}: Let $J$ be the number of pencils Justin has and $S$ be the number of pencils Sabrina has. We are given that $J + S = 50$ and $S = 14$. We are also given that Justin has $x$ more than twice as many pencils as Sabrina, which can be written as $J = 2S + x$. We can substitute $S = 14$ into the first equation to find $J$: $J + 14 = 50$, so $J = 50 - 14 = 36$. Now we can substitute $J = 36$ and $S = 14$ into the second equation: $36 = 2(14) + x$, so $36 = 28 + x$. Solving for $x$, we get $x = 36 - 28 = 8$. \\
\text{[Answer]}: $8$ \\ \\
Based on the question below, please strictly follow this format when answering: \\
1. Start with \text{[Strategy]} section explaining the general approach for solving similar problems; \\
2. End with \text{[Answer]} section containing ONLY the value. (1) Do NOT include units such as minutes, feet, etc.; (2) If the question asks for a percentage, ONLY provide the number (e.g., answer $20$ instead of $20\%$); (3) Do NOT include any explanations; (4) If there is no answer, RETURN None as the value. \\ \\
Now answer this question:\\
\text{[Question]}: Lara bought $52$ stems of flowers at the flower shop. She gave $15$ flowers to her mom and gave $x$ more flowers than she gave to her mom to her grandma. She put the rest in a vase. How many stems of flowers did Lara put in the vase? If we know the answer to the above question is $16$, what is the value of unknown variable $x$? \\
\text{[Strategy]}: \\
\text{[Answer]}:
\end{tcolorbox}
\label{table:context_metamath}
\end{table}

\begin{table}[t]
\caption{Strategy-based prompt template with example input from NASA-History-MCQ}
\vspace{-1em}
\centering
\small   
\begin{tcolorbox}[enhanced, width=\linewidth, colback=white]
\textbf{Prompt Template:} \\[0.3em]
Using the following problem examples: \\
\text{\{strategy\}} \\ \\
Based on the question below, please strictly follow this format when answering: \\
1. Start with \text{[Strategy]} section explaining the general approach for solving similar problems; \\
2. End with \text{[Answer]} section containing ONLY the single capital letter of the correct option (exactly one of: A, B, C, D). Do NOT include the option text or any explanation. \\ \\
Now answer this question: \\
\text{[Question]}: \text{\{question\}} \\
\text{[Strategy]}: \\
\text{[Answer]}:
\bigskip\hrule\bigskip
\textbf{Nasa-History-MCQ Example Input:} \\[0.3em]
Using the following problem examples:\\
1. \text{[Question]}: Beyond communication and care packages, what specific types of hardware or software aid psychological well-being during long-duration spaceflights? \\
A. Specialized dietary supplements to combat isolation \\
B. Automated exercise routines tailored to reduce stress \\
C. Psychological support hardware and software \\
D. Advanced life support systems with mood stabilizers \\
\text{[Strategy]}: The question asks about specific hardware or software that aids psychological well-being during long-duration spaceflights, beyond communication and care packages. We need to evaluate each option to see if it fits this description. Option A focuses on dietary supplements, which are not hardware or software. Option B describes automated exercise routines, which could involve software and hardware. Option C is too general, simply restating the question. Option D focuses on life support systems with mood stabilizers, which are not necessarily hardware or software designed specifically for psychological well-being. Therefore, option B is the most specific and relevant answer. \\
\text{[Answer]}: B \\ \\
2. \text{[Question]}: What is the anticipated effect of constraints inherent in lunar and Martian missions on psychological support approaches? \\
A. Greater emphasis on real-time communication with Earth-based support teams \\
B. Increased reliance on virtual reality and AI companionship to mitigate isolation \\
C. A shift towards highly individualized psychological profiles and tailored interventions \\
D. A return to the mindset and strategies of earlier explorers and their families \\
\text{[Strategy]}: The question asks about the impact of constraints in lunar and Martian missions on psychological support. These constraints include isolation, limited resources, communication delays, and the need for self-sufficiency. Considering these limitations, the most likely effect would be a greater reliance on technologies that can provide support in the absence of immediate Earth-based assistance and a need for personalized approaches due to the unique challenges faced by each astronaut. Options A and D are less likely because of communication delays and the differences between modern space missions and earlier explorations. Option B is plausible, but option C is more comprehensive as it addresses the need for personalized support, which is crucial given the constraints. \\
\text{[Answer]}: C \\ \\
Based on the question below, please strictly follow this format when answering: \\
1. Start with \text{[Strategy]} section explaining the general approach for solving similar problems; \\
2. End with \text{[Answer]} section containing ONLY the single capital letter of the correct option (exactly one of: A, B, C, D). Do NOT include the option text or any explanation. \\ \\
Now answer this question:\\
\text{[Question]}: What three distinct stages comprise current psychological support protocols for US astronauts? \\
A. Individual, group, family \\
B. Training, mission, debriefing \\
C. Cognitive, emotional, behavioral \\
D. Preflight, in-flight, postflight \\
\text{[Strategy]}: \\
\text{[Answer]}:
\end{tcolorbox}
\label{table:context_nasa_history}
\end{table}


\end{document}